\documentclass{article}
\usepackage[letterpaper, left=1in, right=1in, top=1in, bottom=1in]{geometry}

\usepackage{natbib}
\bibpunct{(}{)}{;}{a}{,}{,}
\usepackage[colorlinks,citecolor=blue!70!black,linkcolor=blue!70!black,
urlcolor=blue!70!black,breaklinks=true]{hyperref}

\usepackage{parskip}

\usepackage[utf8]{inputenc} %
\usepackage[T1]{fontenc}    %
\usepackage{booktabs}       %
\usepackage{microtype}      %

\usepackage{amsmath, amsthm, amssymb}
\usepackage{dsfont} 
\usepackage{url}            %
\usepackage{algorithm}
\usepackage[noend]{algorithmic}
\usepackage{listings}
\usepackage{bm}
\usepackage{bbm}
\usepackage{enumitem}
\usepackage{nicefrac}       %
\usepackage{mathrsfs} 
\usepackage{inconsolata}
\usepackage{comment}
\usepackage{xspace}
\usepackage{transparent}
\usepackage[nameinlink, capitalize]{cleveref}
\makeatletter
\AddToHook{cmd/appendix/before}{
\def\cref@section@alias{appendix}
\def\cref@subsection@alias{appendix}
\def\cref@subsubsection@alias{appendix}
}
\makeatother
\usepackage{thmtools}
\usepackage{thm-restate}
\usepackage{nicematrix}
\usepackage{arydshln}
\usepackage{mathtools}
\usepackage{lineno}
\usepackage{xspace}

\usepackage{enumitem}
\setlist[enumerate]{leftmargin=.2in}
\setlist[itemize]{leftmargin=.2in}

\def\E{{\mathbb E}}
\def\V{{\mathbb V}}
\def\P{{\mathbb P}}
\def\R{{\mathbb R}}
\def\N{{\mathbb N}}
\def\B{{\mathbb B}}
\def\S{{\mathbb S}}

\newcommand{\mfe}{\mathfrak{e}}

\newcommand{\mfF}{\mathfrak{F}}

\DeclareMathOperator*{\argmin}{arg\,min}
\DeclareMathOperator{\esup}{ess\,sup}

\DeclareMathOperator{\err}{err}
\DeclareMathOperator{\poly}{poly}
\DeclareMathOperator{\supp}{supp}

\DeclareMathOperator{\unif}{{unif}}

\DeclareMathOperator{\1}{{\mathds{1}}}

\DeclareMathOperator{\val}{{val}}

\newcommand{\trn}{\top}

\newcommand{\polylog}{\mathrm{polylog}}
\renewcommand{\epsilon}{\varepsilon}

\newtheorem{theorem}{Theorem}
\newtheorem{lemma}{Lemma}

\newtheorem{definition}{Definition}

\newtheorem{proposition}{Proposition}
\newtheorem{corollary}{Corollary}
\newtheorem{remark}{Remark}

\DeclarePairedDelimiter{\abs}{\lvert}{\rvert} %
\DeclarePairedDelimiter{\brk}{[}{]}
\DeclarePairedDelimiter{\crl}{\{}{\}}
\DeclarePairedDelimiter{\prn}{(}{)}
\DeclarePairedDelimiter{\nrm}{\|}{\|}
\DeclarePairedDelimiter{\ang}{\langle}{\rangle}

\DeclarePairedDelimiter{\ceil}{\lceil}{\rceil}
\DeclarePairedDelimiter{\floor}{\lfloor}{\rfloor}

\DeclarePairedDelimiterX{\infdiv}[2]{(}{)}{%
  #1\;\delimsize\|\;#2%
}

\newcommand{\wt}[1]{\widetilde{#1}}
\newcommand{\wh}[1]{\widehat{#1}}
\newcommand{\wb}[1]{\widebar{#1}}

\def\ddefloop#1{\ifx\ddefloop#1\else\ddef{#1}\expandafter\ddefloop\fi}
\def\ddef#1{\expandafter\def\csname bb#1\endcsname{\ensuremath{\mathbb{#1}}}}
\ddefloop ABCDEFGHIJKLMNOPQRSTUVWXYZ\ddefloop
\def\ddefloop#1{\ifx\ddefloop#1\else\ddef{#1}\expandafter\ddefloop\fi}
\def\ddef#1{\expandafter\def\csname b#1\endcsname{\ensuremath{\mathbf{#1}}}}
\ddefloop ABCDEFGHIJKLMNOPQRSTUVWXYZ\ddefloop
\def\ddef#1{\expandafter\def\csname sf#1\endcsname{\ensuremath{\mathsf{#1}}}}
\ddefloop ABCDEFGHIJKLMNOPQRSTUVWXYZ\ddefloop
\def\ddef#1{\expandafter\def\csname c#1\endcsname{\ensuremath{\mathcal{#1}}}}
\ddefloop ABCDEFGHIJKLMNOPQRSTUVWXYZ\ddefloop
\def\ddef#1{\expandafter\def\csname h#1\endcsname{\ensuremath{\widehat{#1}}}}
\ddefloop ABCDEFGHIJKLMNOPQRSTUVWXYZ\ddefloop
\def\ddef#1{\expandafter\def\csname hc#1\endcsname{\ensuremath{\widehat{\mathcal{#1}}}}}
\ddefloop ABCDEFGHIJKLMNOPQRSTUVWXYZ\ddefloop
\def\ddef#1{\expandafter\def\csname t#1\endcsname{\ensuremath{\widetilde{#1}}}}
\ddefloop ABCDEFGHIJKLMNOPQRSTUVWXYZ\ddefloop
\def\ddef#1{\expandafter\def\csname tc#1\endcsname{\ensuremath{\widetilde{\mathcal{#1}}}}}
\ddefloop ABCDEFGHIJKLMNOPQRSTUVWXYZ\ddefloop
\def\ddefloop#1{\ifx\ddefloop#1\else\ddef{#1}\expandafter\ddefloop\fi}
\def\ddef#1{\expandafter\def\csname scr#1\endcsname{\ensuremath{\mathscr{#1}}}}
\ddefloop ABCDEFGHIJKLMNOPQRSTUVWXYZ\ddefloop

\let\oldparagraph\paragraph
\renewcommand{\paragraph}[1]{\oldparagraph{#1.}}

\renewcommand{\epsilon}{\varepsilon}

\newcommand{\ind}{\mathbbm{1}}    %

\newcommand{\eps}{\epsilon}

\newcommand{\ldef}{\vcentcolon=}
\newcommand{\rdef}{=\vcentcolon}

\renewcommand{\bigm}[1]{%
  \ifcsname fenced@\string#1\endcsname
    \expandafter\@firstoftwo
  \else
    \expandafter\@secondoftwo
  \fi
  {\expandafter\amsmath@bigm\csname fenced@\string#1\endcsname}%
  {\amsmath@bigm#1}%
}

\newcommand{\DeclareFence}[2]{\@namedef{fenced@\string#1}{#2}}
\makeatother

\makeatletter
\let\save@mathaccent\mathaccent
\newcommand*\if@single[3]{%
  \setbox0\hbox{${\mathaccent"0362{#1}}^H$}%
  \setbox2\hbox{${\mathaccent"0362{\kern0pt#1}}^H$}%
  \ifdim\ht0=\ht2 #3\else #2\fi
  }
\newcommand*\rel@kern[1]{\kern#1\dimexpr\macc@kerna}
\newcommand*\widebar[1]{\@ifnextchar^{{\wide@bar{#1}{0}}}{\wide@bar{#1}{1}}}
\newcommand*\wide@bar[2]{\if@single{#1}{\wide@bar@{#1}{#2}{1}}{\wide@bar@{#1}{#2}{2}}}
\newcommand*\wide@bar@[3]{%
  \begingroup
  \def\mathaccent##1##2{%
    \let\mathaccent\save@mathaccent
    \if#32 \let\macc@nucleus\first@char \fi
    \setbox\z@\hbox{$\macc@style{\macc@nucleus}_{}$}%
    \setbox\tw@\hbox{$\macc@style{\macc@nucleus}{}_{}$}%
    \dimen@\wd\tw@
    \advance\dimen@-\wd\z@
    \divide\dimen@ 3
    \@tempdima\wd\tw@
    \advance\@tempdima-\scriptspace
    \divide\@tempdima 10
    \advance\dimen@-\@tempdima
    \ifdim\dimen@>\z@ \dimen@0pt\fi
    \rel@kern{0.6}\kern-\dimen@
    \if#31
      \overline{\rel@kern{-0.6}\kern\dimen@\macc@nucleus\rel@kern{0.4}\kern\dimen@}%
      \advance\dimen@0.4\dimexpr\macc@kerna
      \let\final@kern#2%
      \ifdim\dimen@<\z@ \let\final@kern1\fi
      \if\final@kern1 \kern-\dimen@\fi
    \else
      \overline{\rel@kern{-0.6}\kern\dimen@#1}%
    \fi
  }%
  \macc@depth\@ne
  \let\math@bgroup\@empty \let\math@egroup\macc@set@skewchar
  \mathsurround\z@ \frozen@everymath{\mathgroup\macc@group\relax}%
  \macc@set@skewchar\relax
  \let\mathaccentV\macc@nested@a
  \if#31
    \macc@nested@a\relax111{#1}%
  \else
    \def\gobble@till@marker##1\endmarker{}%
    \futurelet\first@char\gobble@till@marker#1\endmarker
    \ifcat\noexpand\first@char A\else
      \def\first@char{}%
    \fi
    \macc@nested@a\relax111{\first@char}%
  \fi
  \endgroup
}
\makeatother

\newcommand{\DIS}{\mathsf{{DIS}}}

\newcommand{\AlgLcb}{\mathrm{\mathbf{Alg}}_{\mathsf{lcb}}}
\newcommand{\AlgUcb}{\mathrm{\mathbf{Alg}}_{\mathsf{ucb}}}
\newcommand{\ucb}{\mathsf{ucb}}
\newcommand{\lcb}{\mathsf{lcb}}
\newcommand{\dnn}{\mathsf{dnn}}
\newcommand{\pseud}{\mathrm{Pdim}}
\newcommand{\vcd}{\mathrm{VCdim}}

\newcommand{\exc}{\mathsf{excess}}
\newcommand{\BV}{\mathsf{BV}}
\newcommand{\TV}{\mathsf{TV}}
\newcommand{\wderi}{\mathsf{D}}

\newcommand{\textCAL}{\textsf{CAL}\xspace}

\newcommand{\textRCAL}{\textsf{RobustCAL}\xspace}

\newcommand{\textNCAL}{\textsf{NeuralCAL}\xspace}

\newcommand{\textNCALP}{\textsf{NeuralCAL}\texttt{++}\xspace}

\newcommand{\relu}{\mathsf{ReLU}}
\newcommand{\RBV}{\mathscr{R}\, \BV}
\newcommand{\RTV}{\mathscr{R}\, \TV}

\newcommand{\curly}{\crl}
\newcommand{\paren}{\prn}

\newcommand{\sq}{\brk}

\title{Active Learning with Neural Networks: Insights from Nonparametric Statistics}
\date{}
\author{
Yinglun Zhu\\
{\normalsize University of Wisconsin--Madison}\\
{\normalsize\texttt{yinglun@cs.wisc.edu}}
\and
\and
Robert Nowak\\
{\normalsize University of Wisconsin--Madison}\\
{\normalsize\texttt{rdnowak@wisc.edu}}
}

\begin{document}

\maketitle

\begin{abstract}
	Deep neural networks have great representation power, but typically require large numbers of training examples. This motivates deep active learning methods that can significantly reduce the amount of labeled training data. Empirical successes of deep active learning have been recently reported in the literature, however, rigorous label complexity guarantees of deep active learning have remained elusive. This constitutes a significant gap between theory and practice. This paper tackles this gap by providing the first near-optimal label complexity guarantees for deep active learning. The key insight is to study deep active learning from the nonparametric classification perspective. Under standard low noise conditions, we show that active learning with neural networks can provably achieve the minimax label complexity, up to disagreement coefficient and other logarithmic terms. When equipped with an abstention option, we further develop an efficient deep active learning algorithm that achieves $\mathsf{polylog}(\frac{1}{\epsilon})$ label complexity, without any low noise assumptions.  We also provide extensions of our results beyond the commonly studied Sobolev/H\"older spaces and develop label complexity guarantees for learning in Radon $\mathsf{BV}^2$ spaces, which have recently been proposed as natural function spaces associated with neural networks.
\end{abstract}

\section{Introduction}
\label{sec:intro}

We study active learning with neural network hypothesis classes, sometimes known as \emph{deep active learning}.
Active learning agent proceeds by selecting the most informative data points to label: The goal of active learning is to achieve the same accuracy achievable by passive learning, but with much fewer label queries \citep{settles2009active, hanneke2014theory}.
When the hypothesis class is a set of neural networks, the learner further benefits from the representation power of deep neural networks, which has driven the successes of passive learning in the past decade \citep{krizhevsky2012imagenet, lecun2015deep}.
With these added benefits,
deep active learning has become a popular research area, with empirical successes observed in many recent papers \citep{sener2018active, ash2019deep, citovsky2021batch, ash2021gone, kothawade2021similar, emam2021active, ren2021survey}.
However, due to the difficulty of analyzing a set of neural networks, rigorous label complexity guarantees for deep active learning have remained largely elusive. 

To the best of our knowledge, there are only two papers \citep{karzand2020maximin, wang2021neural} that have made the attempts at theoretically quantifying active learning gains with neural networks.
While insightful views are provided, these two works have their own limitations. The guarantees provided in \citet{karzand2020maximin} only work in the $1d$ case where data points are uniformly sampled from $[0,1]$ and labeled by a well-seperated piece-wise constant function in a noise-free way (i.e., without any labeling noise).
\citet{wang2021neural} study deep active learning by linearizing the neural network at its random initialization and then analyzing it as a linear function; moreover, as the authors agree, their error bounds and label complexity guarantees can in fact be \emph{vacuous} in certain cases.
Thus, it's fair to say that up to now researchers have not identified cases where deep active learning are provably near minimax optimal (or even with provably non-vacuous guarantees), which constitutes a significant gap between theory and practice.

In this paper, we bridge this gap by providing the first near-optimal label complexity guarantees for deep active learning.
We obtain insights from the nonparametric setting where the conditional probability (of taking a label of $1$) is assumed to be a smooth function \citep{tsybakov2004optimal, audibert2007fast}.
Previous nonparametric active learning algorithms proceed by partitioning the action space into exponentially many sub-regions (e.g., partitioning the unit cube $[0,1]^{d}$ into $\eps^{-d}$ sub-cubes each with volume $\eps^{d}$), and then conducting local mean (or some higher-order statistics) estimation  within each sub-region \citep{castro2008minimax, minsker2012plug, locatelli2017adaptivity, locatelli2018adaptive, shekhar2021active, kpotufe2021nuances}.
We show that, with an appropriately chosen set of neural networks that \emph{globally} approximates the smooth regression function, one can in fact recover the minimax label complexity for active learning, up to disagreement coefficient \citep{hanneke2007bound, hanneke2014theory} and other logarithmic factors.
Our results are established by (i) identifying the ``right tools'' to study neural networks (ranging from approximation results \citep{yarotsky2017error,yarotsky2018optimal} to complexity measure of neural networks \citep{bartlett2019nearly}), and (ii) developing novel extensions of agnostic active learning algorithms \citep{balcan2006agnostic, hanneke2007bound, hanneke2014theory} to work with a set of neural networks.

While matching the minimax label complexity in nonparametric active learning is existing, such minimax results scale as $\Theta(\poly(\frac{1}{\eps}))$ \citep{castro2008minimax, locatelli2017adaptivity} and do not resemble what is practically observed in deep active learning: A fairly accurate neural network classifier can be obtained by training with only a few labeled data points.
Inspired by recent results in \emph{parametric} active learning with abstention \citep{puchkin2021exponential, zhu2022efficient}, we develop an oracle-efficient algorithm showing that deep active learning provably achieves $\polylog(\frac{1}{\eps})$ label complexity when equipped with an abstention option \citep{chow1970optimum}.
Our algorithm not only achieves an exponential saving in label complexity (\emph{without any low noise assumptions}), but is also highly practical: 
In real-world scenarios such as medical imaging, it makes more sense for the classifier to abstain from making prediction on hard examples (e.g., those that are close to the boundary), and ask medical experts to make the judgments.

\subsection{Problem setting}
\label{sec:setting}

Let $\cX$ denote the instance space and $\cY$ denote the label space. 
We focus on the binary classification problem where $\cY \ldef \curly*{0, 1}$. The joint distribution over $\cX \times \cY$ is denoted as $\cD_{\cX \cY}$. 
We use $\cD_{\cX}$ to denote the marginal distribution over the instance space $\cX$, and use $\cD_{\cY \vert x}$ to denote the conditional distribution of $\cY$ with respect to any $x \in \cX$.   
We consider the standard active learning setup where $x \sim \cD_\cX$ but its label  $y \sim \cD_{\cY \vert x}$ is only observed after issuing a label query.
We define $\eta(x) \ldef \P_{y \sim \cD_{\cY \vert x}} (y = 1)$ as the conditional probability of taking a label of $1$.
The Bayes optimal classifier $h^{\star}$ can thus be expressed as $h^{\star}(x) \ldef \ind(\eta(x) \geq 1/2)$.
For any classifier $h: \cX \rightarrow \cY$, 
its (standard) error is calculated as $\err(h) \ldef \P_{(x,y) \sim \cD_{\cX \cY}} (h(x) \neq y)$;  
and its (standard) excess error is defined as 
$\exc(h) \ldef \err(h) - \err(h^{\star})$.
Our goal is to learn an accurate classifier with a small number of label querying.

\paragraph{The nonparametric setting}
We consider the nonparametric setting where the conditional probability $\eta$ is characterized by a smooth function.
Fix any $\alpha \in \N_+$, the \emph{Sobolev norm} of a function  $f: \cX \rightarrow \R$ is defined as 
$\nrm{f}_{\cW^{\alpha, \infty}} \ldef \max_{\wb \alpha, \abs{\wb \alpha} \leq \alpha} \esup_{x \in \cX} \abs{\wderi^{\alpha}f (x)}$,
where $\alpha = \prn{ \alpha_1, \ldots, \alpha_d}$, $\abs{\alpha} = \sum_{i=1}^{d} \alpha_i$ and $\wderi^{\alpha}f$ denotes the standard $\alpha$-th weak derivative of $f$.
The unit ball in the Sobolev space is defined as 
$
	\cW^{\alpha, \infty}_1 (\cX) \ldef \crl{ f : \nrm{f}_{\cW^{\alpha, \infty}} \leq 1}.
$
Following the convention of nonparametric active learning \citep{castro2008minimax, minsker2012plug, locatelli2017adaptivity, locatelli2018adaptive, shekhar2021active, kpotufe2021nuances}, we assume $\cX = [0,1]^{d}$ and $\eta \in \cW^{\alpha, \infty}_1(\cX)$ (except in \cref{sec:extension}).

\paragraph{Neural networks}
We consider \emph{feedforward neural networks} with Rectified Linear Unit (ReLU) activation function, which is defined as $\relu(x) \ldef \max \crl{x, 0}$. Each neural network $f_\dnn: \cX \rightarrow \R$ consists of several input units (which corresponds to the covariates of $x \in \cX$), one output unit (which corresponds to the prediction in $\R$), and multiple hidden computational units.
Each hidden computational unit takes inputs $\crl{\wb x_i}_{i=1}^{N}$ (which are outputs from previous layers) and perform the computation $\relu(\sum_{i=1}^{N} w_i \wb x_i + b)$ with \emph{adjustable} parameters $\crl{w_i}_{i=1}^{N}$ and $b$; 
the output unit performs the same operation, but without the ReLU nonlinearity.  
We use $W$ to denote the total number of parameters of a neural network, and  $L$ to denote the depth of the neural network.

\subsection{Contributions and paper organization}
\label{sec:contributions}

Neural networks are known to be universal approximators \citep{cybenko1989approximation, hornik1991approximation}. 
In this paper, we argue that, in both passive and active regimes, 
the universal approximatability makes neural networks ``universal classifiers'' for classification problems: With an appropriately chosen set of neural networks, one can recover known minimax rates (up to disagreement coefficients in the active setting) in the rich nonparametric regimes.\footnote{
As a byproduct, our results also provide a new perspective on nonparametric active learning through the lens of neural network approximations. Nonparametric active learning was previously tackled through space partitioning and local estimations over exponentially many sub-regions \citep{castro2008minimax, minsker2012plug, locatelli2017adaptivity, locatelli2018adaptive, shekhar2021active, kpotufe2021nuances}.}
We provide informal statements of our main results in the sequel, 
with detailed statements and associated definitions/algorithms deferred to later sections.

In \cref{sec:noise}, we analyze the label complexity of deep active learning under the standard Tsybakov noise condition with smoothness parameter $\beta \geq 0$ \citep{tsybakov2004optimal}. Let $\cH_\dnn$ be an appropriately chosen set of neural network classifiers and denote $\theta_{\cH_\dnn}(\eps)$ as the disagreement coefficient \citep{hanneke2007bound, hanneke2014theory} at level $\eps$. 
We develop the following label complexity guarantees for deep active learning.

\begin{theorem}
[Informal]
\label{thm:noise_informal}
There exists an algorithm that returns a neural network classifier $\wh h \in \cH_\dnn$ with excess error $\wt O(\eps)$ after querying $\wt O \prn{ \theta_{\cH_\dnn} \prn{ \eps^{\frac{\beta}{1 +  \beta}}} \cdot \eps^{- \frac{d + 2 \alpha}{\alpha + \alpha\beta}} }$ labels.
\end{theorem}

The label complexity presented in \cref{thm:noise_informal} matches the active learning lower bound $\Omega(\eps^{- \frac{d+2\alpha}{\alpha + \alpha \beta}})$ \citep{locatelli2017adaptivity} up to the dependence on the disagreement coefficient (and other logarithmic factors).
Since $\theta_{\cH_\dnn}(\eps) \leq \eps^{-1}$ by definition, the label complexity presented in \cref{thm:noise_informal} is never worse than the passive learning rates $\wt \Theta(\eps^{-\frac{d+2\alpha + \alpha \beta}{\alpha + \alpha \beta}})$ \citep{audibert2007fast}. 
We also discover conditions under which the disagreement coefficient with respect to a set of neural network classifiers can be properly bounded, i.e., $\theta_{\cH_\dnn} (\eps) = o(\eps^{-1})$ (implying strict improvement over passive learning) and $\theta_{\cH_\dnn} (\eps) = o(1)$ (implying matching active learning lower bound).

In \cref{sec:abstention}, we develop label complexity guarantees for deep active learning when an additional abstention option is allowed \citep{chow1970optimum, puchkin2021exponential, zhu2022efficient}.
Suppose a cost (e.g. $0.49$) that is marginally smaller than random guessing (which has expected cost $0.5$) is incurred whenever the classifier abstains from making a predication, we develop the following label complexity guarantees for deep active learning. 
\looseness=-1

\begin{theorem}
	[Informal]
	\label{thm:abs_informal}
	There exists an efficient algorithm that constructs a neural network classifier $\wh h_\dnn$ with Chow's excess error $\wt O(\eps)$ after querying $\polylog(\frac{1}{\eps})$ labels.
\end{theorem}
The above $\polylog(\frac{1}{\eps})$ label complexity bound is achieved \emph{without any low noise assumptions}.
Such exponential label savings theoretically justify the great empirical performances of deep active learning observed in practice (e.g., in \citet{sener2018active}): 
It suffices to label a few data points to achieve a high accuracy level.
Moreover, apart from an initialization step, our algorithm (\cref{alg:abs}) developed for \cref{thm:abs_informal}  can be \emph{efficiently} implemented in $\wt O(\eps^{-1})$ time, given a convex loss regression oracle over an appropriately chosen set of neural networks; in practice, the regression oracle can be approximated by running stochastic gradient descent.
\looseness=-1

\paragraph{Technical contributions}
Besides identifying the ``right tools'' (ranging from approximation results \citep{yarotsky2017error,yarotsky2018optimal} to complexity analyses \citep{bartlett2019nearly}) to analyze deep active learning, our theoretical guarantees are empowered by novel extensions of active learning algorithms \emph{under neural network approximations}. In particular, we deal with approximation error in active learning under Tsybakov noise, and identify conditions that greatly relax the approximation requirement in the learning with abstention setup; we also analyze the disagreement coefficient, both classifier-based and value function-based, with a set of neural networks.These analyses together lead to our main results for deep active learning (e.g., \cref{thm:noise_informal} and \cref{thm:abs_informal}).
More generally, we establish a bridge between approximation theory and active learning; we provide these general guarantees in \cref{app:RCAL_gen} (under Tsybakov noise) and \cref{app:abs_gen} (with the abstention option), which can be of independent interests.
Benefited from these generic algorithms and guarantees, in \cref{sec:extension}, we extend our results into learning smooth functions in the Radon $\BV^2$ space \citep{ongie2020function, parhi2021banach, parhi2022kinds, parhi2022near, unser2022ridges}, which is recently proposed as a natural space to analyze neural networks.
\looseness=-1

\subsection{Additional related work}
\label{sec:related}

Active learning concerns about learning accurate classifiers without extensive human labeling.
One of the earliest work of active learning dates back to the \textCAL algorithm proposed by \citet{cohn1994improving}, which set the cornerstone for \emph{disagreement-based} active learning.
Since then, a long line of work have been developed, either directly working with a set classifier \citep{balcan2006agnostic, hanneke2007bound, dasgupta2007general, beygelzimer2009importance, beygelzimer2010agnostic, huang2015efficient, cortes2019active} or work with a set of regression functions \citep{krishnamurthy2017active,krishnamurthy2019active}.
These work mainly focus on the parametric regime (e.g., learning with a set of linear classifiers), and their label complexities rely on the boundedness of the so-called disagreement coefficient \citep{hanneke2007bound, hanneke2014theory, friedman2009active}.
Active learning in the nonparametric regime has been analyzed in \citet{castro2008minimax, minsker2012plug, locatelli2017adaptivity, locatelli2018adaptive, kpotufe2021nuances}. These algorithms rely on partitioning of the input space $\cX \subseteq [0,1]^{d}$ into exponentially (in dimension) many small cubes, and then conduct local mean (or some higher-order statistics) estimation within each small cube.
\looseness=-1

It is well known that, in the worst case, active learning exhibits no label complexity gains over the passive counterpart \citep{kaariainen2006active}.
To bypass these worst-case scenarios, active learning has been popularly analyzed under the so-called Tsybakov low noise conditions \citep{tsybakov2004optimal}.
Under Tsybakov noise conditions, active learning has been shown to be strictly superior than passive learning in terms of label complexity \citep{castro2008minimax, locatelli2017adaptivity}.
Besides analyzing active learning under favorable low noise assumptions, more recently, researchers consider active learning with an abstention option and analyze its label complexity under Chow's error \citep{chow1970optimum}.
In particular, \citet{puchkin2021exponential, zhu2022efficient} develop active learning algorithms with $\polylog(\frac{1}{\eps})$ label complexity when analyzed under Chow's excess error.
\citet{shekhar2021active} study nonparametric active learning under a different notion of the Chow's excess error, and propose algorithms with $\poly(\frac{1}{\eps})$ label complexity; their algorithms follow similar procedures of those partition-based nonparametric active learning algorithms (e.g., \citet{minsker2012plug, locatelli2017adaptivity}).\looseness=-1

Inspired by the success of deep learning in the passive regime, active learning with neural networks has been extensively explored in recent years 
\citep{sener2018active, ash2019deep, citovsky2021batch, ash2021gone, kothawade2021similar, emam2021active, ren2021survey}.
Great empirical performances are observed in these papers, however, rigorous label complexity guarantees have largely remains elusive (except in \citet{karzand2020maximin, wang2021neural}, with limitations discussed before).
We bridge the gap between practice and theory by providing the first near-optimal label complexity guarantees for deep active learning.
Our results are built upon approximation results of deep neural networks \citep{yarotsky2017error, yarotsky2018optimal, parhi2022near} and VC/pseudo dimension analyses of neural networks with given structures \citep{bartlett2019nearly}.\looseness=-1

\section{Label complexity of deep active learning}
\label{sec:noise}

We analyze the label complexity of deep active learning in this section.
We first introduce the Tsybakov noise condition in \cref{sec:tsybakov}, and then identify the ``right tools'' to analyze classification problems with neural network classifiers in \cref{sec:noise_passive} (where we also provide passive learning guarantees).
We establish our main active learning guarantees in \cref{sec:noise_active}.

\subsection{Tsybakov noise condition}
\label{sec:tsybakov}
It is well known that active learning exhibits no label complexity gains over the passive counterpart without additional low noise assumptions \citep{kaariainen2006active}.
We next introduce the Tsybokov low noise condition \citep{tsybakov2004optimal}, which has been extensively analyzed in active learning literature.

\begin{definition}[Tsybakov noise]
	\label{def:Tsybakov}
	A distribution $\cD_{\cX\cY}$ satisfies the Tsybakov noise condition with parameter $\beta \geq 0$ and a universal constant $c \geq 1$ if, $\forall \tau > 0$,
	\begin{align*}
	\P_{x \sim \cD_\cX} \prn{ \abs{ \eta(x) - 1 / 2} \leq \tau} \leq  c \, \tau^\beta.
	\end{align*}
\end{definition}

The case with $\beta = 0$ corresponds to the general case \emph{without} any low noise conditions, 
where no active learning algorithm can outperform the passive counterpart \citep{audibert2007fast, locatelli2017adaptivity}.
We use $\cP(\alpha, \beta)$ to denote the set of distributions satisfying:
(i) the smoothness conditions introduced in \cref{sec:setting} with parameter $\alpha > 0$; 
and (ii) the Tsybakov low noise condition (i.e., \cref{def:Tsybakov}) with parameter $\beta \geq 0$. 
We assume 
$\cD_{\cX \cY} \in \cP(\alpha, \beta)$ in the rest of \cref{sec:noise}.
As in \citet{castro2008minimax, hanneke2014theory}, we assume the knowledge of noise/smoothness parameters.

\subsection{Approximation and expressiveness of neural networks}
\label{sec:noise_passive}

Neural networks are known to be universal approximators \citep{cybenko1989approximation, hornik1991approximation}: For any continuous function $g : \cX \rightarrow \R$ and any error tolerance $\kappa > 0$, there exists a large enough neural network $f_{\dnn}$ such that $\nrm{f_\dnn - g}_{\infty} \ldef \sup_{x \in \cX} \abs{f_{\dnn}(x)-g(x)} \leq \kappa$.
Recently, \emph{non-asympototic} approximation rates by ReLU neural networks have been developed for smooth functions in the Sobolev space, which we restate in the following.\footnote{As in \citet{yarotsky2017error}, we hide constants that are potentially $\alpha$-dependent and $d$-dependent into the Big-Oh notation.} 
\begin{theorem}[\citet{yarotsky2017error}]
	\label{thm:approx_sobolev}
	Fix any $\kappa>0$. For any $f^{\star} = \eta \in \cW^{\alpha, \infty}_1([0,1]^{d})$, there exists a neural network $f_\dnn$ with $W = O \prn{\kappa^{- \frac{d}{\alpha}} \log \frac{1}{\kappa} }$ total number of parameters arranged in $L = O ( \log \frac{1}{\kappa})$ layers such that $\nrm{f_\dnn - f^{\star} }_\infty \leq \kappa$. 
\end{theorem}

The architecture of the neural network $f_\dnn$ appearing in the above theorem only depends on the smooth function space $\cW^{\alpha, \infty}_1 \prn{[0,1]^{d}}$, but otherwise is independent of the true regression function $f^{\star}$; also see \citet{yarotsky2017error} for details.
Let $\cF_{\dnn}$ denote the set of neural network \emph{regression functions} with the same architecture. We construct a set of neural network \emph{classifiers} by thresholding the regression function at $\frac{1}{2}$, i.e.,  
$\cH_{\dnn} \ldef \crl{ h_f \ldef \ind(f(x) \geq 1/2): f \in \cF_{\dnn}} $.
The next result concerns about the expressiveness of the neural network classifiers, in terms of a well-known complexity measure: the VC dimension \citep{vapnik1971uniform}.

\begin{theorem}[\citet{bartlett2019nearly}]
	\label{thm:vcd_nn}
	Let $\cH_{\dnn}$ be a set of neural network classifiers of the same architecture and with $W$ parameters arranged in  $L$ layers.
	We then have 
	\begin{align*}
	\Omega(WL \log \prn*{{W}/{L}}) \leq \vcd(\cH_\dnn)  \leq O(WL \log \prn*{ W}). 
	\end{align*}
\end{theorem}

With these tools, we can construct a set of neural network classifiers $\cH_{\dnn}$ such that (i) the best in-class classifier $\check h \in \cH_{\dnn}$ has small excess error, and (ii) $\cH_{\dnn}$  has a well-controlled VC dimension that is proportional to smooth/noise parameters. 
More specifically, we have the following proposition.
\begin{restatable}{proposition}{propVCApprox}
	\label{prop:vc_approx}
Suppose $\cD_{\cX \cY} \in \cP(\alpha, \beta)$.
One can construct a set of neural network classifier $\cH_{\dnn}$ such that the following two properties hold simultaneously: 
\begin{align*}
	\inf_{h \in \cH_{\dnn} }\err(h) - \err(h^{\star}) = O \prn{\eps} \quad \text{ and }
\quad 	\vcd( \cH_{\dnn}) = \wt O \prn{ \eps^{- \frac{d}{\alpha(1+\beta)}}}.
\end{align*}
\end{restatable}

With the approximation results obtained above, to learn a classifier with $O(\eps)$ excess error, one only needs to focus on a set of neural networks $\cH_{\dnn}$ with a well-controlled VC dimension.
As a warm-up, we first analyze the label complexity of such procedure in the passive regime (with fast rates).

\begin{restatable}{theorem}{thmPassiveNoise}
	\label{thm:passive_noise}
	Suppose $\cD_{\cX \cY} \in \cP(\alpha, \beta)$.
	Fix any $\eps, \delta > 0$.
	Let $\cH_{\dnn}$ be the set of neural network classifiers constructed in \cref{prop:vc_approx}.
	With $n = \wt O ( \eps^{- \frac{d+2\alpha + \alpha \beta}{\alpha(1+\beta)} })$ i.i.d. sampled points, with probability at least $1-\delta$,
	the empirical risk minimizer $\wh h \in \cH_{\dnn}$ achieves excess error $ O(\eps)$.
\end{restatable}

The label complexity results obtained in \cref{thm:passive_noise} matches, up to logarithmic factors, the passive learning lower bound $\Omega ( \eps^{- \frac{d+2\alpha + \alpha \beta}{\alpha(1+\beta)} })$ established in \citet{audibert2007fast}, indicating that our proposed learning procedure \emph{with a set of neural networks} is near minimax optimal.\footnote{Similar passive learning guarantees have been developed with different tools and analyses, e.g., see results in \citet{kim2021fast}.}

\subsection{Deep active learning and guarantees}
\label{sec:noise_active}

The passive learning procedure presented in the previous section treats every data point equally, i.e., it requests the label of every data point. 
Active learning reduces the label complexity by only querying labels of data points that are ``more important''.
We present deep active learning results in this section. 
Our algorithm (\cref{alg:NCAL}) is inspired by \textRCAL \citep{balcan2006agnostic, hanneke2007bound, hanneke2014theory} and the seminal \textCAL algorithm \citep{cohn1994improving}; we call our algorithm \textNCAL to emphasize that it works with a set of neural networks.
\looseness=-1

For any accuracy level $\eps > 0$, \textNCAL first initialize a set of neural network classifiers $\cH_0 \ldef \cH_{\dnn}$ such that (i) the best in-class classifier $\check h \ldef \argmin_{h \in \cH_{\dnn}} \err(h)$ has excess error at most  $O(\eps)$, and (ii) the VC dimension of  $\cH_\dnn$ is upper bounded by  $\wt O(\eps^{ - \frac{d}{\alpha(1+\beta)}})$ (see \cref{sec:noise_passive} for more details).
\textNCAL then runs in epochs of geometrically increasing lengths.
At the beginning of epoch $m$, based on previously \emph{labeled} data points, \textNCAL updates a set of active classifier $\cH_m$ such that, with high probability, the best classifier $\check h$ remains \emph{uneliminated}.
Within each epoch $m$, \textNCAL only queries the label $y$ of a data point $x$ if it lies in the \emph{region of disagreement} with respect to the current active set of classifier  $\cH_m$, i.e., 
	$\DIS(\cH_m) \ldef \crl{x \in \cX: \exists h_1, h_2 \in \cH_m \text{ s.t. } h_1(x) \neq h_2(x)}$.
\textNCAL returns any classifier $\wh h \in \cH_m$ that remains uneliminated after  $M-1$ epoch.

\begin{algorithm}[H]
	\caption{\textNCAL}
	\label{alg:NCAL} 
	\renewcommand{\algorithmicrequire}{\textbf{Input:}}
	\renewcommand{\algorithmicensure}{\textbf{Output:}}
	\newcommand{\algorithmicbreak}{\textbf{break}}
    \newcommand{\BREAK}{\STATE \algorithmicbreak}
	\begin{algorithmic}[1]
		\REQUIRE Accuracy level $\epsilon \in (0, 1)$, confidence level $\delta \in (0, 1)$.
		\STATE Let $\cH_\dnn$ be a set of neural networks classifiers constructed in \cref{prop:vc_approx}.
		\STATE Define $T \ldef \eps^{- \frac{2+\beta}{1+\beta}} \cdot \vcd(\cH_{\dnn}) $, $M \ldef \ceil{\log_2 T}$, $\tau_m \ldef 2^m$ for $m\geq1$ and $\tau_0 \ldef 0$. 
		\STATE Define $\rho_m \ldef O \prn*{ \prn*{\frac{\vcd(\cH_{\dnn}) \cdot \log (\tau_{m-1}) \cdot  \log (M /\delta) }{\tau_{m-1}}}^{\frac{1+\beta}{2+\beta}} }$ for $m \geq 2$ and  $\rho_1 \ldef 1$.
		\STATE Define $\wh R_m(h) \ldef \sum_{t = 1}^{\tau_{m-1}} Q_t \ind\prn*{h(x_t) \neq y_t}$ with the convention that $\sum_{t=1}^{0} \ldots = 0$.
		\STATE Initialize $\cH_0 \ldef \cH_{\dnn}$.
		\FOR{epoch $m = 1, 2, \dots, M$}
		\STATE Update active set 
		    $\cH_m \ldef \crl*{ h \in \cH_{m-1}:  \wh R_m(h) \leq \inf_{h \in \cH_{m-1}} \wh R_m(h) + \tau_{m-1} \cdot \rho_m}.$
		\IF{epoch $m=M$}
		\STATE \textbf{Return} any classifier $\wh h \in \cH_M$.
		\ENDIF
		\FOR{time $t = \tau_{m-1} + 1 ,\ldots , \tau_{m} $} 
		\STATE Observe $x_t \sim \cD_{\cX}$. Set $Q_t \ldef \ind(x_t \in \DIS(\cH_m))$.
		\IF{$Q_t = 1$}
		\STATE Query the label $y_t$ of $x_t$.
		\ENDIF
		\ENDFOR
		\ENDFOR

	\end{algorithmic}
\end{algorithm}

Since \textNCAL only queries labels of data points lying in the region of disagreement, its label complexity should intuitively be related to how fast the region of disagreement shrinks. More formally, the rate of collapse of the (probability measure of) region of disagreement is captured by the \emph{(classifier-based) disagreement coefficient} \citep{hanneke2007bound, hanneke2014theory}, which 
we introduce next.

\begin{definition}[Classifier-based disagreement coefficient]
	\label{def:dis_coeff_classifier}
	For any $\eps_0$ and classifier  $h \in \cH$, the classifier-based disagreement coefficient of  $h$ is defined as 
	 \begin{align*}
		\theta_{\cH, h} (\eps_0)
        \ldef \sup_{\eps > \eps_0} \frac{\P_{ x \sim \cD_{\cX}} 
	\prn{ \DIS \prn{ \cB_{\cH}( h, \eps)}}}{\eps} \vee 1,
	\end{align*}
  where $\cB_{\cH}(h,\eps) \ldef \crl{ g \in \cH: \P_{x \sim \cD_\cX} \prn{  g(x) \neq h(x)} \leq \eps}$.
	We also define $\theta_{\cH}(\eps_0) \ldef \sup_{h \in \cH} \theta_{\cH, h}(\eps_0)$.
\end{definition}

The guarantees of \textNCAL follows from a more general analysis of \textRCAL under function approximation.
In particular, to achieve fast rates under Tsybakov noise, previous analysis of \textRCAL requires that the Bayes optimal classifier lies within the hypothesis class \citep{hanneke2014theory}.
This requirement is typically not satisfied in our setting with neural network approximations.
Our analysis broadens the understanding of \textRCAL under function approximation; 
we defer the general analysis to \cref{app:RCAL_gen} and present the guarantees below.
\looseness=-1

\begin{restatable}{theorem}{thmActiveNoise}
	\label{thm:active_noise}
	Suppose $\cD_{\cX \cY} \in \cP(\alpha, \beta)$.
	Fix any $\eps,\delta > 0$.
	With probability at least $1-\delta$, \cref{alg:NCAL} 
	returns a classifier $\wh h \in \cH_{\dnn}$ with excess error $\wt O(\eps)$ after querying 
	$\wt O \prn{ \theta_{\cH_\dnn} \prn{ \eps^{\frac{\beta}{1 +  \beta}}} \cdot \eps^{- \frac{d + 2 \alpha}{\alpha + \alpha\beta}} }$ labels.
\end{restatable}

We next discuss in detail the label complexity of deep active learning proved in \cref{thm:active_noise}.

\begin{itemize}
\item 
Ignoring the dependence on disagreement coefficient, the label complexity appearing in \cref{thm:active_noise} matches, up to logarithmic factors, the lower bound $\Omega(\eps^{- \frac{d+ 2 \alpha}{\alpha + \alpha \beta}})$ for active learning \citep{locatelli2017adaptivity}.
	At the same time, the label complexity appearing in \cref{thm:active_noise} is \emph{never worse} than the passive counterpart (i.e., $\wt \Theta( \eps^{- \frac{d+2\alpha + \alpha \beta}{\alpha(1+\beta)} })$ since $\theta_{\cH_\dnn}(\eps^{\frac{\beta}{1+\beta}}) \leq \eps^{- \frac{\beta}{1 + \beta}}$.
\item 
We also identify cases when $\theta_{\cH_\dnn}(\eps^{\frac{\beta}{1+\beta}}) = o ( \eps^{- \frac{\beta}{1+\beta}})$, indicating \emph{strict} improvement over passive learning (e.g., when $\cD_\cX$ is supported on countably many data points), 
and when $\theta_{\cH_\dnn}(\eps^{\frac{\beta}{1+\beta}}) = O(1)$, indicating matching the minimax active lower bound (e.g., when $\cD_{\cX \cY}$ satisfies conditions such as \emph{decomposibility} defined in \cref{def:decomposable}. See \cref{app:dis_coeff} for detailed discussion).\footnote{We remark that disagreement coefficient is usually bounded/analyzed under additional assumptions on $\cD_{\cX \cY}$, even for simple cases with a set of linear classifiers \citep{friedman2009active, hanneke2014theory}. 
The label complexity guarantees of partition-based nonparametric active algorithms (e.g., \citet{castro2008minimax}) do not depend on the disagreement coefficient, but they are analyzed under stronger assumptions, e.g., they require the strictly stronger membership querying oracle. See \citet{wang2011smoothness} for a discussion.
We left a comprehensive analysis of the disagreement coefficient with a set of neural network classifiers for future work.}
\end{itemize}

Our algorithm and theorems lead to the following results, 
which could benefit both deep active learning and nonparametric learning communities.
\begin{itemize}
	\item \textbf{Near minimax optimal label complexity for deep active learning.}
	While empirical successes of deep active learning have been observed, rigorous label complexity analysis remains elusive except for two attempts made in \citet{karzand2020maximin, wang2021neural}.
	The guarantees provided in \citet{karzand2020maximin} only work in very special cases (i.e., data uniformly sampled from $[0,1]$ and labeled by well-separated piece-constant functions in a noise-free way). \citet{wang2021neural} study deep active learning in the NTK regime by linearizing the neural network at its random initialization and analyzing it as a linear function; moreover, as the authors agree, their error bounds and label complexity guarantees are \emph{vacuous} in certain cases. 
	On the other hand, our guarantees are minimax optimal, up to disagreement coefficient and other logarithmic factors, which bridge the gap between theory and practice in deep active learning.
  \looseness=-1
	\item \textbf{New perspective on nonparametric learning.}
	Nonparametric learning of smooth functions have been mainly approached by partitioning-based methods \citep{tsybakov2004optimal, audibert2007fast, castro2008minimax, minsker2012plug, locatelli2017adaptivity, locatelli2018adaptive, kpotufe2021nuances} : Partition the unit cube  $[0,1]^{d}$ into exponentially (in dimension) many sub-cubes and conduct local mean estimation within each sub-cube (which additionally requires a strictly stronger membership querying oracle).
Our results show that, in both passive and active settings, one can learn \emph{globally} with a set of neural networks and achieve near minimax optimal label complexities.
\looseness=-1
\end{itemize}

\section{Deep active learning with abstention: Exponential speedups}
\label{sec:abstention}

While the theoretical guarantees provided in \cref{sec:noise} are near minimax optimal, the label complexity scales as $\poly(\frac{1}{\eps})$, which doesn't match the great empirical performance observed in deep active learning.
In this section, we fill in this gap by leveraging the idea of abstention and provide a deep active learning algorithm that achieves exponential label savings.
We introduce the concepts of abstention and Chow's excess error in \cref{sec:abs_chow}, and provide our label complexity guarantees in \cref{sec:exponential}.\looseness=-1

\subsection{Active learning without low noise conditions}
\label{sec:abs_chow}

The previous section analyzes active learning under Tsybakov noise, which has been extensively studied in the literature since \citet{castro2008minimax}.
More recently, promising results are observed in active learning under Chow's excess error, but otherwise \emph{without any low noise assumption} \citep{puchkin2021exponential, zhu2022efficient}.
We introduce this setting in the following.

\paragraph{Abstention and Chow's error \citep{chow1970optimum}} 
We consider classifier of the form $\widehat h: \cX \rightarrow \cY \cup \curly*{\bot}$ where $\bot$ denotes the action of abstention. For any fixed $0 < \gamma < \frac{1}{2}$, the Chow's error is defined as 
\begin{align*}
    \err_{\gamma}(\widehat h)  \ldef \P_{(x,y) \sim \cD_{\cX \cY}} \prn{\widehat h (x) \neq y,  \widehat h(x) \neq \bot} + \prn*{{1}/{2} - \gamma} \cdot \P_{(x,y) \sim \cD_{\cX \cY}} \prn{\widehat h(x) = \bot}.
\end{align*}
The parameter $\gamma$ can be chosen as a small constant, e.g., $\gamma = 0.01$, to avoid excessive abstention: The price of abstention is only marginally smaller than random guess (which incurs cost $0.5$).
The \emph{Chow's excess error} is then defined as $\exc_\gamma(\wh h) \ldef \err_{\gamma}(\wh h) - \err(h^\star)$ \citep{puchkin2021exponential}.

At a high level, analyzing with Chow's excess error allows slackness in predications of hard examples (e.g., data points whose $\eta(x)$ is close to $\frac{1}{2}$) by leveraging the power of abstention.
\citet{puchkin2021exponential,zhu2022efficient} show that $\polylog(\frac{1}{\eps})$ is always achievable in the \emph{parametric} settings.
We generalize their results to the \emph{nonparametric} setting and analyze active learning with a set of neural networks.

\subsection{Exponential speedups with abstention}
\label{sec:exponential}

In this section, we work with a set of neural network \emph{regression functions}  $\cF_\dnn: \cX \rightarrow [0,1]$ (that approximates  $\eta$) and then \emph{construct} classifiers  $h: \cX \rightarrow \cY \cup \crl{\bot}$ \emph{with an additional abstention action}. 
To work with a set of regression functions $\cF_{\dnn}$, we analyze its ``complexity'' from the lenses of \emph{pseudo dimension} $\pseud(\cF_{\dnn})$ \citep{pollard1984convergence, haussler1989decision, haussler1995sphere} and \emph{value function disagreement coefficient} $\theta^{\val}_{\cF_\dnn} (\iota)$ (for some $\iota >0$) \citep{foster2020instance}.
We defer detailed definitions of these complexity measures to \cref{app:value_func_complexity}.

\begin{algorithm}[H]
	\caption{\textNCALP}
	\label{alg:NCALP} 
	\renewcommand{\algorithmicrequire}{\textbf{Input:}}
	\renewcommand{\algorithmicensure}{\textbf{Output:}}
	\newcommand{\algorithmicbreak}{\textbf{break}}
    \newcommand{\BREAK}{\STATE \algorithmicbreak}
	\begin{algorithmic}[1]
		\REQUIRE Accuracy level $\epsilon \in (0, 1)$, confidence level $\delta \in (0, 1)$, abstention parameter $\gamma \in (0, 1/2)$.
		\STATE Let $\cF_\dnn$ be a set of neural network regression functions obtained by (i) applying \cref{thm:approx_sobolev} with an appropriate approximation level $\kappa$ (which satisfies $\frac{1}{\kappa} = \poly(\frac{1}{\gamma}) \, \polylog(\frac{1}{\eps \, \gamma})$), and (ii) applying a preprocessing step on the set of neural networks obtained from step (i). See \cref{app:abs} for details.
		\STATE Define $T \ldef \frac{\theta^{\val}_{\cF_\dnn}(\gamma /4) \cdot \pseud(\cF_\dnn) }{\eps \, \gamma }$, $M \ldef \ceil{\log_2 T}$, and $C_\delta \ldef O\prn{\pseud(\cF_\dnn) \cdot \log(T /\delta)}$.
		\STATE Define $\tau_m \ldef 2^m$ for $m\geq1$, $\tau_0 \ldef 0$, and $\beta_m \ldef 3(M-m+1) C_\delta$. 
		\STATE Define  $\wh R_m(f) \ldef \sum_{t=1}^{\tau_{m-1}} Q_t \prn{\wh f(x_t) - y_t}^2 $ with the convention that $\sum_{t=1}^{0} \ldots = 0$.
		\FOR{epoch $m = 1, 2, \dots, M$}
		\STATE Get $\widehat f_m \ldef \argmin_{f \in \cF_\dnn} \sum_{t=1}^{\tau_{m-1}} Q_t \paren{f(x_t) - y_t}^2 $.
		\STATE 
		(Implicitely) Construct active set 
		   $ \cF_m \ldef \crl*{ f \in \cF_\dnn:  \wh R_m(f) \leq  \wh R_m(\wh f_m) + \beta_m }$.
		\STATE Construct classifier $\wh h_m: \cX \rightarrow \crl{0, 1, \bot}$ as 
		\begin{align*}
			\wh h_m (x) \ldef 
			\begin{cases}
				\bot, & \text{ if } \brk { \lcb(x;\cF_m) - \frac{\gamma}{4}, \ucb(x;\cF_m) + \frac{\gamma}{4}} \subseteq 
				\brk*{ \frac{1}{2} - \gamma, \frac{1}{2} + \gamma}; \\
        \ind(\wh f_m(x) \geq \frac{1}{2} ) , & \text{o.w.}
			\end{cases}
		\end{align*}
		and query function 
			$g_m(x)\ldef \ind \prn*{ \frac{1}{2} \in \prn*{\lcb(x;\cF_m) - \frac{\gamma}{4}, \ucb(x;\cF_m) + \frac{\gamma}{4}} } \cdot
		\ind \prn{\wh h_m(x) \neq \bot}$.
		\IF{epoch $m=M$}
		\STATE \textbf{Return} classifier $\wh h_M$.
		\ENDIF
		\FOR{time $t = \tau_{m-1} + 1 ,\ldots , \tau_{m} $} 
		\STATE Observe $x_t \sim \cD_{\cX}$. Set $Q_t \ldef g_m(x_t)$.
		\IF{$Q_t = 1$}
		\STATE Query the label $y_t$ of $x_t$.
		\ENDIF
		\ENDFOR
		\ENDFOR

	\end{algorithmic}
\end{algorithm}

We now present \textNCALP (\cref{alg:NCALP}), a deep active learning algorithm that leverages the power of abstention.
\textNCALP first initialize a set of set of neural network regression functions $\cF_\dnn$ by applying a preprocessing step on top of the set of regression functions obtained from \cref{thm:approx_sobolev} with a carefully chosen approximation level $\kappa$.
The preprocessing step mainly contains two actions: 
(1) clipping  $f_\dnn: \cX \rightarrow \R$ into  $\check f_\dnn: \cX \rightarrow [0,1]$ (since we obviously have $\eta(x) \in [0,1]$); and
(2) filtering out  $f_\dnn \in \cF_\dnn$ that are clearly not a good approximation of  $\eta$.
After initialization, \textNCALP runs in epochs of geometrically increasing lengths.
At the beginning of epoch $m \in [M]$, \textNCALP (implicitly) constructs an active set of regression functions $\cF_m$ that are ``close'' to the true conditional probability $\eta$.
For any $x \sim \cD_\cX$, \textNCALP constructs a lower bound $\lcb(x;\cF_m) \ldef \inf_{f \in \cF_m} f(x)$ and an upper bound $\ucb(x;\cF_m) \ldef \sup_{f \in \cF_m} f(x)$ as a confidence range of $\eta(x)$ (based on $\cF_m$).  
An empirical classifier with an abstention option $\wh h_m: \cX \rightarrow \crl{0, 1,\bot}$ and a query function $g_m:\cX \rightarrow \crl{0, 1}$ are then constructed based on the confidence range (and the abstention parameter $\gamma$).
For any time step $t$ within epoch $m$, \textNCALP queries the label of the observed data point $x_t$ if and only if $Q_t \ldef g_m(x_t) = 1$.
\textNCALP returns $\wh h_M$ as the learned classifier.

\textNCALP is adapted from the algorithm developed in \citet{zhu2022efficient}, but with novel extensions. In particular, the algorithm presented in \citet{zhu2022efficient} requires the existence of a $\wb f \in \cF$ such that $\nrm{\wb f - \eta}_\infty \leq \eps$ (to achieve  $\eps$ Chow's excess error),
Such an approximation requirement directly leads to $\poly(\frac{1}{\eps})$ label complexity
\emph{in the nonparametric setting}, which is unacceptable.
The initialization step of \textNCALP (line 1) is carefully chosen to ensure that $\pseud(\cF_\dnn), \theta^{\val}_{\cF_\dnn}(\frac{\gamma}{4}) = \poly(\frac{1}{\gamma}) \cdot \polylog(\frac{1}{\eps})$; together with a sharper analysis of concentration results, these conditions help us derive the following deep active learning guarantees (also see \cref{app:abs_gen} for a more general guarantee).

\begin{restatable}{theorem}{thmAbs}
	\label{thm:abs}
	Fix any $\eps, \delta, \gamma >0$.
	With probability at least $1-\delta$, \cref{alg:NCALP} (with an appropriate initialization at line 1) returns a classifier  $\wh h$ with Chow's excess error $\wt O(\eps)$ after querying $\poly(\frac{1}{\gamma}) \cdot \polylog(\frac{1}{\eps \, \delta})$ labels.
\end{restatable}

We discuss two important aspects of \cref{alg:NCALP}/\cref{thm:abs} in the following, i.e., exponential savings and computational efficiency. We defer more detailed discussions to \cref{app:proper} and \cref{app:computational}.

\begin{itemize}
	\item 
	\textbf{Exponential speedups.}
	\cref{thm:abs} shows that, equipped with an abstention option, deep active learning enjoys $\polylog(\frac{1}{\eps})$ label complexity. This provides theoretical justifications for great empirical results of deep active learning observed in practice.
	Moreover, \cref{alg:NCALP} outputs a classifier that abstains \emph{properly}, i.e., it abstains only if abstention is the optimal choice; such a property further implies $\polylog(\frac{1}{\eps})$ label complexity under \emph{standard} excess error and Massart noise \citep{massart2006risk}.\looseness=-1
	 \item  
	\textbf{Computational efficiency.}
	Suppose one can efficiently implement a (weighted) square loss regression oracle over the \emph{initialized} set of neural networks $\cF_\dnn$:  
		Given any set $\cS$ of weighted examples $(w, x, y) \in \R_+ \times \cX \times \cY$, the regression oracle outputs
$\widehat f_\dnn \ldef \argmin_{f \in \cF_\dnn} \sum_{(w, x, y) \in \cS} w \paren*{f(x) - y}^2$
	.\footnote{In practice, this oracle can be approximated using gradient descent or its variants.} \cref{alg:NCALP} can then be \emph{efficiently} implemented with $\poly(\frac{1}{ \gamma}) \cdot \frac{1}{\eps}$ oracle calls.
\end{itemize}

While the label complexity obtained in \cref{thm:abs} has desired dependence on $\polylog(\frac{1}{\eps})$, its dependence on $\gamma$ can be of order  $\gamma^{-\poly(d)}$.
Our next result shows that, however, such dependence is unavoidable even in the case of learning a single ReLU function. 
\begin{restatable}{theorem}{thmSingleReLULB}
	\label{thm:single_relu_lb}
Fix any $\gamma \in (0,1 /8)$.
For any accuracy level $\eps$ sufficiently small, there exists a problem instance such that 
(1)  $\eta \in \cW^{1,\infty}_1 (\cX)$ and is of the form $\eta(x) \ldef \relu(\ang{w,x}+a)+b$;
and (2) for any active learning algorithm, it takes at least $\gamma^{-\Omega(d)}$ labels to identify an $\eps$-optimal classifier, for either standard excess error or Chow's excess error (with parameter  $\gamma$).
\end{restatable}

\section{Extensions}
\label{sec:extension}

Previous results are developed in the commonly studied Sobolev/H\"older spaces.
Our techniques, however, are generic and can be adapted to other function spaces, given neural network approximation results.
In this section, we provide extensions of our results to the Radon $\BV^{2}$ space, which was recently proposed as the natural function space associated with ReLU neural networks \citep{ongie2020function, parhi2021banach, parhi2022kinds, parhi2022near, unser2022ridges}.\footnote{Other extensions are also possible given neural network approximation results, e.g., recent results established in \citet{lu2021deep}.}

\paragraph{The Radon $\BV^2$ space}
The Radon $\BV^2$ unit ball over domain $\cX$ is defined as 
	$\RBV^2_1(\cX) \ldef \crl{f: \nrm{f}_{\RBV^2(\cX)} \leq 1}$,
where $\nrm{f}_{\RBV^2(\cX)}$ denotes the Radon $\BV^2$ norm of $f$ over domain $\cX$.\footnote{We provide more mathematical backgrounds and associated definitions in \cref{app:extension}.}
Following \citet{parhi2022near}, we assume $\cX = \crl{x \in \R^{d}: \nrm{x}_2 \leq 1}$ and $\eta \in \RBV^2_1(\cX)$.\looseness=-1

The Radon $\BV^2$ space naturally contains neural networks of the form $f_{\dnn}(x) = \sum_{k=1}^{K} v_i \cdot \relu(w_i^{\trn}x + b_i) $.
On the contrary, such $f_\dnn$ doesn't lie in any Sobolev space of order $\alpha \geq 2$ (since  $f_\dnn$ doesn't have second order \emph{weak} derivative).
Thus, if  $\eta$ takes the form of the aforementioned neural network (e.g., $\eta = f_\dnn$), approximating  $\eta$ up to $\kappa $ from a Sobolev perspective requires  $\wt O(\kappa^{-{d}})$ total parameters, which suffers from the curse of dimensionality.
On the other side, however, such bad dependence on dimensionality goes away when approximating from a Radon $\BV^2$ perspective, as shown in the following theorem.

\begin{theorem}[\citet{parhi2022near}]
\label{thm:approx_RBV2}	
	Fix any $\kappa>0$. For any $f^{\star} \in \RBV^2_1(\cX)$, there exists a one-hidden layer neural network $f_\dnn$ of width $K = O \prn{\kappa^{- \frac{2d}{d+3}}  }$ such that $\nrm{f^{\star} - f_\dnn}_\infty \leq \kappa$. 
\end{theorem}

Equipped with this approximation result, we provide the active learning guarantees for learning a smooth function within the Radon $\BV^2$ unit ball as follows.

\begin{restatable}{theorem}{thmActiveNoiseRBV}
	\label{thm:active_noise_RBV}
	Suppose $\eta \in \RBV^2_1(\cX)$ and the Tsybakov noise condition is satisfied with parameter $\beta \geq 0$.
	Fix any $\eps, \delta > 0$.
	There exists an algorithm such that, with probability at least $1-\delta$, it learns 
	a classifier $\wh h \in \cH_\dnn$ with excess error $\wt O(\eps)$ after querying  
	$\wt O \prn{ \theta_{\cH_\dnn} \prn{ \eps^{\frac{\beta}{1 +  \beta}}} \cdot \eps^{- \frac{4d + 6}{(1+\beta)(d+3)}} }$ labels.
\end{restatable}

Compared to the label complexity obtained in \cref{thm:active_noise}, the label complexity obtained in the above theorem doesn't suffer from the curse of dimensionality: For $d$ large enough, the above label complexity scales as  $\eps^{-O(1)}$ yet label complexity in \cref{thm:active_noise} scales as $\eps^{-O(d)}$.
Active learning guarantees under Chow's excess error in the Radon $\BV^2$ space are similar to results presented in \cref{thm:abs}, and are thus deferred to \cref{app:extension}.

\section{Discussion}
\label{sec:discussion}
We provide the first near-optimal deep active learning guarantees, under both standard excess error and Chow's excess error. 
Our results are powered by generic algorithms and analyses developed for active learning that bridge approximation guarantees into label complexity guarantees. 
We outline some natural directions for future research below.
\begin{itemize}
	\item \textbf{Disagreement coefficients for neural networks.}
		While we have provided some results regarding the disagreement coefficients for neural networks, we believe a comprehensive investigation on this topic is needed.
	For instance, can we discover more general settings where the classifier-based disagreement coefficient can be upper bounded by $O(1)$?
	It is also interesting to explore sharper analyses on the value function disagreement coefficient.
\item \textbf{Adaptivity in deep active learning.}
	Our current results are established with the knowledge of some problem-dependent parameters, e.g., the smoothness parameters regarding the function spaces and the noise levels.
	It will be interesting to see if one can develop algorithms that can automatically adapt to unknown parameters, e.g., by leveraging techniques developed in \citet{locatelli2017adaptivity, locatelli2018adaptive}.
\end{itemize}

\section*{Acknowledgements}
The authors would like to thank Rahul Parhi for many helpful discussions regarding his papers. We also would like to thank anonymous reviewers for their constructive comments. 
This work is partially supported by NSF grant 1934612 and AFOSR grant FA9550-18-1-0166.

\bibliography{refs.bib}

\begin{thebibliography}{66}
\providecommand{\natexlab}[1]{#1}
\providecommand{\url}[1]{\texttt{#1}}
\expandafter\ifx\csname urlstyle\endcsname\relax
  \providecommand{\doi}[1]{doi: #1}\else
  \providecommand{\doi}{doi: \begingroup \urlstyle{rm}\Url}\fi

\bibitem[Agarwal et~al.(2014)Agarwal, Hsu, Kale, Langford, Li, and Schapire]{agarwal2014taming}
Alekh Agarwal, Daniel Hsu, Satyen Kale, John Langford, Lihong Li, and Robert Schapire.
\newblock Taming the monster: A fast and simple algorithm for contextual bandits.
\newblock In \emph{International Conference on Machine Learning}, pages 1638--1646. PMLR, 2014.

\bibitem[Anthony(2002)]{anthony2002uniform}
Martin Anthony.
\newblock Uniform glivenko-cantelli theorems and concentration of measure in the mathematical modelling of learning.
\newblock \emph{Research Report LSE-CDAM-2002--07}, 2002.

\bibitem[Ash et~al.(2021)Ash, Goel, Krishnamurthy, and Kakade]{ash2021gone}
Jordan Ash, Surbhi Goel, Akshay Krishnamurthy, and Sham Kakade.
\newblock Gone fishing: Neural active learning with fisher embeddings.
\newblock \emph{Advances in Neural Information Processing Systems}, 34, 2021.

\bibitem[Ash et~al.(2019)Ash, Zhang, Krishnamurthy, Langford, and Agarwal]{ash2019deep}
Jordan~T Ash, Chicheng Zhang, Akshay Krishnamurthy, John Langford, and Alekh Agarwal.
\newblock Deep batch active learning by diverse, uncertain gradient lower bounds.
\newblock \emph{arXiv preprint arXiv:1906.03671}, 2019.

\bibitem[Audibert and Tsybakov(2007)]{audibert2007fast}
Jean-Yves Audibert and Alexandre~B Tsybakov.
\newblock Fast learning rates for plug-in classifiers.
\newblock \emph{The Annals of statistics}, 35\penalty0 (2):\penalty0 608--633, 2007.

\bibitem[Balcan et~al.(2006)Balcan, Beygelzimer, and Langford]{balcan2006agnostic}
Maria-Florina Balcan, Alina Beygelzimer, and John Langford.
\newblock Agnostic active learning.
\newblock In \emph{Proceedings of the 23rd international conference on Machine learning}, pages 65--72, 2006.

\bibitem[Bartlett et~al.(2019)Bartlett, Harvey, Liaw, and Mehrabian]{bartlett2019nearly}
Peter~L Bartlett, Nick Harvey, Christopher Liaw, and Abbas Mehrabian.
\newblock Nearly-tight vc-dimension and pseudodimension bounds for piecewise linear neural networks.
\newblock \emph{The Journal of Machine Learning Research}, 20\penalty0 (1):\penalty0 2285--2301, 2019.

\bibitem[Beygelzimer et~al.(2009)Beygelzimer, Dasgupta, and Langford]{beygelzimer2009importance}
Alina Beygelzimer, Sanjoy Dasgupta, and John Langford.
\newblock Importance weighted active learning.
\newblock In \emph{Proceedings of the 26th annual international conference on machine learning}, pages 49--56, 2009.

\bibitem[Beygelzimer et~al.(2010)Beygelzimer, Hsu, Langford, and Zhang]{beygelzimer2010agnostic}
Alina Beygelzimer, Daniel~J Hsu, John Langford, and Tong Zhang.
\newblock Agnostic active learning without constraints.
\newblock \emph{Advances in neural information processing systems}, 23, 2010.

\bibitem[Boucheron et~al.(2005)Boucheron, Bousquet, and Lugosi]{boucheron2005theory}
St{\'e}phane Boucheron, Olivier Bousquet, and G{\'a}bor Lugosi.
\newblock Theory of classification: A survey of some recent advances.
\newblock \emph{ESAIM: probability and statistics}, 9:\penalty0 323--375, 2005.

\bibitem[Castro and Nowak(2008)]{castro2008minimax}
Rui~M Castro and Robert~D Nowak.
\newblock Minimax bounds for active learning.
\newblock \emph{IEEE Transactions on Information Theory}, 54\penalty0 (5):\penalty0 2339--2353, 2008.

\bibitem[Chow(1970)]{chow1970optimum}
CK~Chow.
\newblock On optimum recognition error and reject tradeoff.
\newblock \emph{IEEE Transactions on information theory}, 16\penalty0 (1):\penalty0 41--46, 1970.

\bibitem[Citovsky et~al.(2021)Citovsky, DeSalvo, Gentile, Karydas, Rajagopalan, Rostamizadeh, and Kumar]{citovsky2021batch}
Gui Citovsky, Giulia DeSalvo, Claudio Gentile, Lazaros Karydas, Anand Rajagopalan, Afshin Rostamizadeh, and Sanjiv Kumar.
\newblock Batch active learning at scale.
\newblock \emph{Advances in Neural Information Processing Systems}, 34, 2021.

\bibitem[Cohn et~al.(1994)Cohn, Atlas, and Ladner]{cohn1994improving}
David Cohn, Les Atlas, and Richard Ladner.
\newblock Improving generalization with active learning.
\newblock \emph{Machine learning}, 15\penalty0 (2):\penalty0 201--221, 1994.

\bibitem[Cortes et~al.(2019)Cortes, DeSalvo, Mohri, Zhang, and Gentile]{cortes2019active}
Corinna Cortes, Giulia DeSalvo, Mehryar Mohri, Ningshan Zhang, and Claudio Gentile.
\newblock Active learning with disagreement graphs.
\newblock In \emph{International Conference on Machine Learning}, pages 1379--1387. PMLR, 2019.

\bibitem[Cybenko(1989)]{cybenko1989approximation}
George Cybenko.
\newblock Approximation by superpositions of a sigmoidal function.
\newblock \emph{Mathematics of control, signals and systems}, 2\penalty0 (4):\penalty0 303--314, 1989.

\bibitem[Dasgupta et~al.(2007)Dasgupta, Hsu, and Monteleoni]{dasgupta2007general}
Sanjoy Dasgupta, Daniel~J Hsu, and Claire Monteleoni.
\newblock A general agnostic active learning algorithm.
\newblock \emph{Advances in neural information processing systems}, 20, 2007.

\bibitem[Emam et~al.(2021)Emam, Chu, Chiang, Czaja, Leapman, Goldblum, and Goldstein]{emam2021active}
Zeyad Ali~Sami Emam, Hong-Min Chu, Ping-Yeh Chiang, Wojciech Czaja, Richard Leapman, Micah Goldblum, and Tom Goldstein.
\newblock Active learning at the imagenet scale.
\newblock \emph{arXiv preprint arXiv:2111.12880}, 2021.

\bibitem[Foster et~al.(2018)Foster, Agarwal, Dud{\'\i}k, Luo, and Schapire]{foster2018practical}
Dylan Foster, Alekh Agarwal, Miroslav Dud{\'\i}k, Haipeng Luo, and Robert Schapire.
\newblock Practical contextual bandits with regression oracles.
\newblock In \emph{International Conference on Machine Learning}, pages 1539--1548. PMLR, 2018.

\bibitem[Foster et~al.(2020)Foster, Rakhlin, Simchi-Levi, and Xu]{foster2020instance}
Dylan~J Foster, Alexander Rakhlin, David Simchi-Levi, and Yunzong Xu.
\newblock Instance-dependent complexity of contextual bandits and reinforcement learning: A disagreement-based perspective.
\newblock \emph{arXiv preprint arXiv:2010.03104}, 2020.

\bibitem[Freedman(1975)]{freedman1975tail}
David~A Freedman.
\newblock On tail probabilities for martingales.
\newblock \emph{the Annals of Probability}, pages 100--118, 1975.

\bibitem[Friedman(2009)]{friedman2009active}
Eric Friedman.
\newblock Active learning for smooth problems.
\newblock In \emph{COLT}. Citeseer, 2009.

\bibitem[Hanneke(2007)]{hanneke2007bound}
Steve Hanneke.
\newblock A bound on the label complexity of agnostic active learning.
\newblock In \emph{Proceedings of the 24th international conference on Machine learning}, pages 353--360, 2007.

\bibitem[Hanneke(2014)]{hanneke2014theory}
Steve Hanneke.
\newblock Theory of active learning.
\newblock \emph{Foundations and Trends in Machine Learning}, 7\penalty0 (2-3), 2014.

\bibitem[Haussler(1989)]{haussler1989decision}
David Haussler.
\newblock Decision theoretic generalizations of the pac model for neural net and other learning applications.
\newblock 1989.

\bibitem[Haussler(1995)]{haussler1995sphere}
David Haussler.
\newblock Sphere packing numbers for subsets of the boolean n-cube with bounded vapnik-chervonenkis dimension.
\newblock \emph{Journal of Combinatorial Theory, Series A}, 69\penalty0 (2):\penalty0 217--232, 1995.

\bibitem[Heinonen(2005)]{heinonen2005lectures}
Juha Heinonen.
\newblock \emph{Lectures on Lipschitz analysis}.
\newblock Number 100. University of Jyv{\"a}skyl{\"a}, 2005.

\bibitem[Hornik(1991)]{hornik1991approximation}
Kurt Hornik.
\newblock Approximation capabilities of multilayer feedforward networks.
\newblock \emph{Neural networks}, 4\penalty0 (2):\penalty0 251--257, 1991.

\bibitem[Huang et~al.(2015)Huang, Agarwal, Hsu, Langford, and Schapire]{huang2015efficient}
Tzu-Kuo Huang, Alekh Agarwal, Daniel~J Hsu, John Langford, and Robert~E Schapire.
\newblock Efficient and parsimonious agnostic active learning.
\newblock \emph{Advances in Neural Information Processing Systems}, 28, 2015.

\bibitem[K{\"a}{\"a}ri{\"a}inen(2006)]{kaariainen2006active}
Matti K{\"a}{\"a}ri{\"a}inen.
\newblock Active learning in the non-realizable case.
\newblock In \emph{International Conference on Algorithmic Learning Theory}, pages 63--77. Springer, 2006.

\bibitem[Karzand and Nowak(2020)]{karzand2020maximin}
Mina Karzand and Robert~D Nowak.
\newblock Maximin active learning in overparameterized model classes.
\newblock \emph{IEEE Journal on Selected Areas in Information Theory}, 1\penalty0 (1):\penalty0 167--177, 2020.

\bibitem[Kim et~al.(2021)Kim, Ohn, and Kim]{kim2021fast}
Yongdai Kim, Ilsang Ohn, and Dongha Kim.
\newblock Fast convergence rates of deep neural networks for classification.
\newblock \emph{Neural Networks}, 138:\penalty0 179--197, 2021.

\bibitem[Kothawade et~al.(2021)Kothawade, Beck, Killamsetty, and Iyer]{kothawade2021similar}
Suraj Kothawade, Nathan Beck, Krishnateja Killamsetty, and Rishabh Iyer.
\newblock Similar: Submodular information measures based active learning in realistic scenarios.
\newblock \emph{Advances in Neural Information Processing Systems}, 34, 2021.

\bibitem[Kpotufe et~al.(2021)Kpotufe, Yuan, and Zhao]{kpotufe2021nuances}
Samory Kpotufe, Gan Yuan, and Yunfan Zhao.
\newblock Nuances in margin conditions determine gains in active learning.
\newblock \emph{arXiv preprint arXiv:2110.08418}, 2021.

\bibitem[Krishnamurthy et~al.(2017)Krishnamurthy, Agarwal, Huang, Daum{\'e}~III, and Langford]{krishnamurthy2017active}
Akshay Krishnamurthy, Alekh Agarwal, Tzu-Kuo Huang, Hal Daum{\'e}~III, and John Langford.
\newblock Active learning for cost-sensitive classification.
\newblock In \emph{International Conference on Machine Learning}, pages 1915--1924. PMLR, 2017.

\bibitem[Krishnamurthy et~al.(2019)Krishnamurthy, Agarwal, Huang, Daum{\'e}~III, and Langford]{krishnamurthy2019active}
Akshay Krishnamurthy, Alekh Agarwal, Tzu-Kuo Huang, Hal Daum{\'e}~III, and John Langford.
\newblock Active learning for cost-sensitive classification.
\newblock \emph{Journal of Machine Learning Research}, 20:\penalty0 1--50, 2019.

\bibitem[Krizhevsky et~al.(2012)Krizhevsky, Sutskever, and Hinton]{krizhevsky2012imagenet}
Alex Krizhevsky, Ilya Sutskever, and Geoffrey~E Hinton.
\newblock Imagenet classification with deep convolutional neural networks.
\newblock \emph{Advances in neural information processing systems}, 25, 2012.

\bibitem[LeCun et~al.(2015)LeCun, Bengio, and Hinton]{lecun2015deep}
Yann LeCun, Yoshua Bengio, and Geoffrey Hinton.
\newblock Deep learning.
\newblock \emph{nature}, 521\penalty0 (7553):\penalty0 436--444, 2015.

\bibitem[Li et~al.(2021)Li, Kamath, Foster, and Srebro]{li2021eluder}
Gene Li, Pritish Kamath, Dylan~J Foster, and Nathan Srebro.
\newblock Eluder dimension and generalized rank.
\newblock \emph{arXiv preprint arXiv:2104.06970}, 2021.

\bibitem[Locatelli et~al.(2017)Locatelli, Carpentier, and Kpotufe]{locatelli2017adaptivity}
Andrea Locatelli, Alexandra Carpentier, and Samory Kpotufe.
\newblock Adaptivity to noise parameters in nonparametric active learning.
\newblock In \emph{Proceedings of the 2017 Conference on Learning Theory, PMLR}, 2017.

\bibitem[Locatelli et~al.(2018)Locatelli, Carpentier, and Kpotufe]{locatelli2018adaptive}
Andrea Locatelli, Alexandra Carpentier, and Samory Kpotufe.
\newblock An adaptive strategy for active learning with smooth decision boundary.
\newblock In \emph{Algorithmic Learning Theory}, pages 547--571. PMLR, 2018.

\bibitem[Lu et~al.(2021)Lu, Shen, Yang, and Zhang]{lu2021deep}
Jianfeng Lu, Zuowei Shen, Haizhao Yang, and Shijun Zhang.
\newblock Deep network approximation for smooth functions.
\newblock \emph{SIAM Journal on Mathematical Analysis}, 53\penalty0 (5):\penalty0 5465--5506, 2021.

\bibitem[Massart and N{\'e}d{\'e}lec(2006)]{massart2006risk}
Pascal Massart and {\'E}lodie N{\'e}d{\'e}lec.
\newblock Risk bounds for statistical learning.
\newblock \emph{The Annals of Statistics}, 34\penalty0 (5):\penalty0 2326--2366, 2006.

\bibitem[Minsker(2012)]{minsker2012plug}
Stanislav Minsker.
\newblock Plug-in approach to active learning.
\newblock \emph{Journal of Machine Learning Research}, 13\penalty0 (1), 2012.

\bibitem[Ongie et~al.(2020)Ongie, Willett, Soudry, and Srebro]{ongie2020function}
Greg Ongie, Rebecca Willett, Daniel Soudry, and Nathan Srebro.
\newblock A function space view of bounded norm infinite width relu nets: The multivariate case.
\newblock In \emph{International Conference on Learning Representations}, 2020.

\bibitem[Parhi and Nowak(2021)]{parhi2021banach}
Rahul Parhi and Robert~D Nowak.
\newblock Banach space representer theorems for neural networks and ridge splines.
\newblock \emph{J. Mach. Learn. Res.}, 22\penalty0 (43):\penalty0 1--40, 2021.

\bibitem[Parhi and Nowak(2022{\natexlab{a}})]{parhi2022kinds}
Rahul Parhi and Robert~D Nowak.
\newblock What kinds of functions do deep neural networks learn? insights from variational spline theory.
\newblock \emph{SIAM Journal on Mathematics of Data Science}, 4\penalty0 (2):\penalty0 464--489, 2022{\natexlab{a}}.

\bibitem[Parhi and Nowak(2022{\natexlab{b}})]{parhi2022near}
Rahul Parhi and Robert~D Nowak.
\newblock Near-minimax optimal estimation with shallow relu neural networks.
\newblock \emph{IEEE Transactions on Information Theory}, 2022{\natexlab{b}}.

\bibitem[Pollard(1984)]{pollard1984convergence}
D~Pollard.
\newblock \emph{Convergence of Stochastic Processes}.
\newblock David Pollard, 1984.

\bibitem[Puchkin and Zhivotovskiy(2021)]{puchkin2021exponential}
Nikita Puchkin and Nikita Zhivotovskiy.
\newblock Exponential savings in agnostic active learning through abstention.
\newblock \emph{arXiv preprint arXiv:2102.00451}, 2021.

\bibitem[Ren et~al.(2021)Ren, Xiao, Chang, Huang, Li, Gupta, Chen, and Wang]{ren2021survey}
Pengzhen Ren, Yun Xiao, Xiaojun Chang, Po-Yao Huang, Zhihui Li, Brij~B Gupta, Xiaojiang Chen, and Xin Wang.
\newblock A survey of deep active learning.
\newblock \emph{ACM Computing Surveys (CSUR)}, 54\penalty0 (9):\penalty0 1--40, 2021.

\bibitem[Russo and Van~Roy(2013)]{russo2013eluder}
Daniel Russo and Benjamin Van~Roy.
\newblock Eluder dimension and the sample complexity of optimistic exploration.
\newblock In \emph{NIPS}, pages 2256--2264. Citeseer, 2013.

\bibitem[Sener and Savarese(2018)]{sener2018active}
Ozan Sener and Silvio Savarese.
\newblock Active learning for convolutional neural networks: A core-set approach.
\newblock In \emph{International Conference on Learning Representations}, 2018.

\bibitem[Settles(2009)]{settles2009active}
Burr Settles.
\newblock Active learning literature survey.
\newblock 2009.

\bibitem[Shalev-Shwartz and Ben-David(2014)]{shalev2014understanding}
Shai Shalev-Shwartz and Shai Ben-David.
\newblock \emph{Understanding machine learning: From theory to algorithms}.
\newblock Cambridge university press, 2014.

\bibitem[Shekhar et~al.(2021)Shekhar, Ghavamzadeh, and Javidi]{shekhar2021active}
Shubhanshu Shekhar, Mohammad Ghavamzadeh, and Tara Javidi.
\newblock Active learning for classification with abstention.
\newblock \emph{IEEE Journal on Selected Areas in Information Theory}, 2\penalty0 (2):\penalty0 705--719, 2021.

\bibitem[Tsybakov(2004)]{tsybakov2004optimal}
Alexander~B Tsybakov.
\newblock Optimal aggregation of classifiers in statistical learning.
\newblock \emph{The Annals of Statistics}, 32\penalty0 (1):\penalty0 135--166, 2004.

\bibitem[Unser(2022)]{unser2022ridges}
Michael Unser.
\newblock Ridges, neural networks, and the radon transform.
\newblock \emph{arXiv preprint arXiv:2203.02543}, 2022.

\bibitem[Vapnik and Chervonenkis(1971)]{vapnik1971uniform}
VN~Vapnik and A~Ya Chervonenkis.
\newblock On the uniform convergence of relative frequencies of events to their probabilities.
\newblock \emph{Theory of Probability and its Applications}, 16\penalty0 (2):\penalty0 264, 1971.

\bibitem[Wainwright(2019)]{wainwright2019high}
Martin~J Wainwright.
\newblock \emph{High-dimensional statistics: A non-asymptotic viewpoint}, volume~48.
\newblock Cambridge University Press, 2019.

\bibitem[Wang(2011)]{wang2011smoothness}
Liwei Wang.
\newblock Smoothness, disagreement coefficient, and the label complexity of agnostic active learning.
\newblock \emph{Journal of Machine Learning Research}, 12\penalty0 (7), 2011.

\bibitem[Wang et~al.(2021)Wang, Awasthi, Dann, Sekhari, and Gentile]{wang2021neural}
Zhilei Wang, Pranjal Awasthi, Christoph Dann, Ayush Sekhari, and Claudio Gentile.
\newblock Neural active learning with performance guarantees.
\newblock \emph{Advances in Neural Information Processing Systems}, 34, 2021.

\bibitem[Yao(1977)]{yao1977probabilistic}
Andrew Chi-Chin Yao.
\newblock Probabilistic computations: Toward a unified measure of complexity.
\newblock In \emph{18th Annual Symposium on Foundations of Computer Science (sfcs 1977)}, pages 222--227. IEEE Computer Society, 1977.

\bibitem[Yarotsky(2017)]{yarotsky2017error}
Dmitry Yarotsky.
\newblock Error bounds for approximations with deep relu networks.
\newblock \emph{Neural Networks}, 94:\penalty0 103--114, 2017.

\bibitem[Yarotsky(2018)]{yarotsky2018optimal}
Dmitry Yarotsky.
\newblock Optimal approximation of continuous functions by very deep relu networks.
\newblock In \emph{Conference on learning theory}, pages 639--649. PMLR, 2018.

\bibitem[Zhu and Nowak(2022)]{zhu2022efficient}
Yinglun Zhu and Robert Nowak.
\newblock Efficient active learning with abstention.
\newblock \emph{Advances in Neural Information Processing Systems}, 35:\penalty0 35379--35391, 2022.

\end{thebibliography}

\clearpage

\appendix

\section{Omitted details for \cref{sec:noise_passive}}

\propVCApprox*
\begin{proof}
We take $\kappa = \eps^{\frac{1}{1+\beta}}$ in \cref{thm:approx_sobolev} to construct a set of neural network classifiers $\cH_\dnn$ with  $W = O \prn{ \eps^{- \frac{d}{\alpha(1+\beta)}} \log \frac{1}{\eps}}$ total parameters arranged in $L = O \prn{ \log \frac{1}{\eps}}$ layers.	
According to \cref{thm:vcd_nn}, we know 
\begin{align*}
\vcd( \cH_{\dnn}) = O \prn{\eps^{- \frac{d}{\alpha(1+\beta)}} \cdot \log^2 \prn{\eps^{-1}}} = \wt O \prn{ \eps^{- \frac{d}{\alpha(1+\beta)}}}.
\end{align*}
We now show that there exists a classifier $\wb h \in \cH_{\dnn}$ with small excess error.
Let $\wb h = h_{\wb f}$ be the classifier such that  $\nrm{ \wb f - \eta}_{\infty} \leq \kappa$. We can see that 
 \begin{align*}
	\exc(\wb h)
	& = \E \brk*{ \ind( \wb h(x) \neq y) - \ind( h^{\star}(x) \neq y)}\\
	& = \E \brk*{ \abs{ 2 \eta(x) - 1} \cdot \ind( \wb h(x) \neq h^{\star}(x))}\\
	& \leq 2 \kappa \cdot \P_{x \sim \cD_{\cX}} \prn*{ x \in \cX: \abs{\eta(x) - {1}/{2}} \leq \kappa}\\
	& = O \prn{ \kappa^{1 + \beta}}\\
	& = O(\eps),
\end{align*}
where the third line follows from the fact that $\wb h$ and  $h^{\star}$ disagrees only within region $\crl{x \in \cX: \abs{\eta(x) - 1 /2} \leq \kappa}$ and the incurred error is at most  $2 \kappa$ on each disagreed data point.
The fourth line follows from the Tsybakov noise condition and the last line follows from the selection of $\kappa$.
\end{proof}

Before proving \cref{thm:passive_noise}, we first recall the excess error guarantee for empirical risk minimization under Tsybakov noise condition.

\begin{theorem}[\citet{boucheron2005theory}]
	\label{thm:erm_tsy}
	Suppose $\cD_{\cX \cY}$ satisfies Tsybakov noise condition with parameter  $\beta \geq 0$.
	Consider a datatset $D_n = \crl{(x_i, y_i)}_{i=1}^{n}$ of $n$ points i.i.d. sampled from  $\cD_{\cX \cY}$.
	Let $\wh h \in \cH$ be the empirical risk minimizer on $D_n$.
	For any constant  $\rho > 0$, we have
	 \begin{align*}
		 & \err(\wh h)  - \min_{h \in \cH} \err(h) \\
		 &\leq \rho \cdot \prn{ \min_{h \in \cH}\err(h) - \err (h^{\star})}
		 + O \prn*{ \frac{\prn*{1 + \rho}^2}{\rho} \cdot \prn*{ \frac{\vcd(\cH) \cdot \log n}{n}}^{\frac{1+\beta}{2+\beta}} + \frac{\log\delta^{-1}}{n}},
	\end{align*}
	with probability at least $1-\delta$.
\end{theorem}

\thmPassiveNoise*

\begin{proof}
\cref{prop:vc_approx} certifies $\min_{h \in \cH_\dnn} \err(h) - \err(h^{\star}) = O(\eps)$
and $\vcd(\cH_\dnn) = O \prn*{\eps^{-\frac{d}{\alpha(1+\beta)}} \cdot \log^2(\eps^{-1})}$.
Take $\rho = 1$ in \cref{thm:erm_tsy}, leads to 
 \begin{align*}
& \err(\wh h) - \err(h^{\star}) \leq   O \prn*{ \eps + \prn*{ \eps^{-\frac{d}{\alpha(1+\beta)}} \cdot \log^2(\eps^{-1}) \cdot \frac{\log n}{n}}^{\frac{1+\beta}{2+\beta}} + \frac{\log\delta^{-1}}{n}},
\end{align*}
Taking $n = O \prn{\eps^{-\frac{d+2\alpha + \alpha \beta}{\alpha(1+\beta)}}\cdot \log (\eps^{-1}) + \eps^{-1} \cdot \log(\delta^{-1})} = \wt O \prn{\eps^{-\frac{d+2\alpha + \alpha \beta}{\alpha(1+\beta)}}}$ thus ensures that $\err(\wh h) - \err(h^{\star}) = O(\eps)$.
\end{proof}

\section{Generic version of \cref{alg:NCAL} and its guarantees}
\label{app:RCAL_gen}
We present \cref{alg:RCAL} below, a generic version of \cref{alg:NCAL} that doesn't require the approximating classifiers to be neural networks.
The guarantees of \cref{alg:RCAL} are provided in \cref{thm:RCAL_gen}, which is proved in \cref{app:RCAL_gen_proof} based on 
supporting lemmas provided in \cref{app:RCAL_gen_lemma}.

\begin{algorithm}[H]
	\caption{\textRCAL with Approximation}
	\label{alg:RCAL} 
	\renewcommand{\algorithmicrequire}{\textbf{Input:}}
	\renewcommand{\algorithmicensure}{\textbf{Output:}}
	\newcommand{\algorithmicbreak}{\textbf{break}}
    \newcommand{\BREAK}{\STATE \algorithmicbreak}
	\begin{algorithmic}[1]
		\REQUIRE Accuracy level $\epsilon \in (0, 1)$, confidence level $\delta \in (0, 1)$.
		\STATE Let $\cH$ be a set of approximating classifiers such that $\inf_{h \in \cH }\err(h) - \err(h^{\star}) = O \prn{\eps}$.
		\STATE Define $T \ldef \eps^{- \frac{2+\beta}{1+\beta}} \cdot \vcd(\cH) $, $M \ldef \ceil{\log_2 T}$, $\tau_m \ldef 2^m$ for $m\geq1$ and $\tau_0 \ldef 0$. 
		\STATE Define $\rho_m \ldef O \prn*{ \prn*{\frac{\vcd(\cH) \cdot \log (\tau_{m-1}) \cdot  \log (M /\delta) }{\tau_{m-1}}}^{\frac{1+\beta}{2+\beta}} }$ for $m \geq 2$ and  $\rho_1 \ldef 1$.
		\STATE Define $\wh R_m(h) \ldef \sum_{t = 1}^{\tau_{m-1}} Q_t \ind\prn*{h(x_t) \neq y_t}$ with the convention that $\sum_{t=1}^{0} \ldots = 0$.
		\STATE Initialize $\cH_0 \ldef \cH$.
		\FOR{epoch $m = 1, 2, \dots, M$}
		\STATE Update active set 
		    $\cH_m \ldef \crl*{ h \in \cH_{m-1}:  \wh R_m(h) \leq \inf_{h \in \cH_{m-1}} \wh R_m(h) + \tau_{m-1} \cdot \rho_m}$
		\IF{epoch $m=M$}
		\STATE \textbf{Return} any classifier $\wh h \in \cH_M$.
		\ENDIF
		\FOR{time $t = \tau_{m-1} + 1 ,\ldots , \tau_{m} $} 
		\STATE Observe $x_t \sim \cD_{\cX}$. Set $Q_t \ldef \ind(x_t \in \DIS(\cH_m))$.
		\IF{$Q_t = 1$}
		\STATE Query the label $y_t$ of $x_t$.
		\ENDIF
		\ENDFOR
		\ENDFOR

	\end{algorithmic}
\end{algorithm}
We provide guarantees for \cref{alg:RCAL}, and then specialize them to the settings with neural network approximation, i.e., in \cref{thm:active_noise} and \cref{thm:active_noise_RBV}.
Our proofs build on the analysis of \textRCAL \citep{hanneke2014theory}, with additional arguments to handle function approximation. We note that the original analysis assumes $h^\star \in \cH$, i.e., the Bayes optimal classifier is contained in the hypothesis class.

 \begin{restatable}{theorem}{thmRCALGen}
	\label{thm:RCAL_gen}
	Fix $\eps, \delta > 0$.
	Suppose $\inf_{h \in \cH} \err(h) - \err(h^{\star}) = O(\eps)$.
	With probability at least $1-\delta$, 
	\cref{alg:RCAL} returns a classifier $\wh h \in \cH$ with excess error $\wt O(\eps)$ 
	after querying 
	 \begin{align*}
		 \wt O \prn*{ \theta_\cH(\eps^{\frac{\beta}{1+\beta}}) \cdot \eps^{-\frac{2}{1+\beta}}\cdot \vcd(\cH)  }
	\end{align*}
	labels.
\end{restatable}

\subsection{Supporting lemmas}
\label{app:RCAL_gen_lemma}

\begin{lemma}[\citet{tsybakov2004optimal, hanneke2014theory}]
	\label{lm:bernstein}
  Let $h^\star$ denote the Bayes optimal classifier.
	Suppose $\cD_{\cX \cY}$ satisfies the Tsybakov noise condition with parameter $\beta \geq 0$, then there exists an universal constant $c^{\prime} > 0$ such that we have 
\begin{align*}
\P_{x \sim \cD_\cX} \prn{h(x) \neq h^{\star}(x)} \leq c^{\prime} \prn{ \err(h) - \err(h^{\star})}^{\frac{\beta}{1+\beta}}	
\end{align*}
for any measurable $h: \cX \rightarrow \cY$.
\end{lemma}

We next present a lemma in the passive learning setting, which will later be incorporated into the active learning setting. 
We first define some notations. 
Suppose $D_n = \crl{(x_i, y_i)}_{i=1}^{n}$ are $n$ i.i.d. data points drawn from $\cD_{\cX \cY}$.
For any measurable $h: \cX \rightarrow \cY$, we denote  $\wb R_n(h) \ldef \sum_{i=1}^{n} \ind(h(x_i) \neq y_i)$ as the empirical error of $h$ over dataset $D_n$. We clearly have $\E\brk{\wb R_n(h)} = n \cdot \err(h)$ by i.i.d. assumption.

\begin{lemma}
	\label{lm:tsy_concentration}
Fix $\eps, \wb \delta > 0$.
Suppose $\cD_{\cX \cY}$ satisfies Tsybakov noise condition with parameter $\beta \geq 0$ and $\err(\check h) - \err(h^{\star}) = O(\eps)$, where $\check h = \argmin_{h \in \cH} \err(h)$ and $h^{\star}$ is the Bayes classifier.
Let $D_n = \crl{(x_i, y_i)}_{i=1}^{n}$ be a set of $n$ i.i.d. data points drawn from $\cD_{\cX \cY}$. 
If $\beta > 0$, suppose $n$ satisfies 
\begin{align*}
	n \leq \eps^{- \frac{2+\beta}{1+\beta}} \cdot \vcd(\cH)^{\frac{2+2\beta}{\beta}} \cdot \log (\wb \delta^{-1}) \cdot \prn{\log n}^{\frac{2+2\beta}{\beta}}.
\end{align*}
With probability at least $1-\wb \delta$, we have the following inequalities hold:
 \begin{align}
	 n \cdot \prn{ \err(h) - \err(h^{\star})} & \leq 2 \cdot \prn{\wb R_n(h) - \wb R_n(\check h)} + n \cdot \rho(n,\wb\delta) , \quad \forall h \in \cH, \label{eq:tsy_concentration_1} \\
	 \wb R_n(\check h)  - \min_{h \in\cH} \wb R_n(h) & \leq  n \cdot \rho(n, \wb \delta), \label{eq:tsy_concentration_2}
\end{align}
  where $\rho(n,\wb \delta) \ldef C \cdot  \prn*{ \prn*{\frac{\vcd(\cH) \cdot \log n \cdot  \log \wb \delta^{-1} }{n}}^{\frac{1+\beta}{2+\beta}} + \eps }$ with a universal constant $C >0$.\footnote{The logarithmic factors in this bound might be further optimized; however, we do not focus on optimizing logarithmic factors in this work.}
\end{lemma}
\begin{proof}
	Denote $\wb \cH \ldef \cH \cup \crl{h^{\star}}$. We know that $\vcd(\wb \cH) \leq \vcd(\cH) + 1 = O(\vcd(\cH))$.
  Since $\cD_{\cX \cY}$ satisfies Tsybakov noise condition and $h^\star \in \wb \cH$, the condition in \cref{lm:bernstein} is satisfied by all $h \in \wb \cH$.
Invoking Lemma 3.1 in \citet{hanneke2014theory}, with probability at least $1-\frac{\wb \delta}{2}$, $\forall h \in \wb \cH$, we have 
 \begin{align}
n \cdot \prn{ \err(h) - \err(h^{\star})} & \leq \max \crl*{2 \cdot \prn{\wb R_n(h) - \wb R_n(h^{\star})} , n \cdot \wb \rho(n, \wb \delta)} , \label{eq:tsy_concentration_3}\\
\wb R_n(h)  - \min_{h \in \wb \cH} \wb R_n(h) & \leq  \max \crl*{ 2n \cdot \prn{\err(h) - \err(h^{\star})}  , n \cdot \wb \rho(n, \wb \delta)} \label{eq:tsy_concentration_4},
\end{align}
where $\wb \rho(n, \wb \delta) = O \prn*{ \prn*{\frac{\vcd(\wb \cH) \cdot \log n + \log \wb \delta^{-1} }{n}}^{\frac{1+\beta}{2+\beta}}  }= O \prn*{ \prn*{\frac{\vcd(\cH) \cdot \log n \cdot \log \wb \delta^{-1} }{n}}^{\frac{1+\beta}{2+\beta}}  }$.

\cref{eq:tsy_concentration_2} follows by taking $h = \check h$ in \cref{eq:tsy_concentration_4} and noticing that 
 \begin{align*}
\wb R_h(\check h) - \min_{h \in \cH} \wb R_n(h) 
& \leq \wb R_n (\check h) - \min_{h \in \wb \cH} \wb R_n(h)\\
& \leq \max \crl*{ 2n \cdot O(\eps) , n \cdot \wb \rho(n,\wb \delta)},
\end{align*}
where we use the assumption that $\err(\check h) - \err(h^{\star}) = O(\eps)$.

To derive \cref{eq:tsy_concentration_1}, we first notice that applying \cref{eq:tsy_concentration_3} for any $h \in \cH$, we have 
\begin{align*}
	n \cdot \prn{\err(h) - \err(h^{\star})}
	& \leq 2 \cdot \prn{ \wb R_n(h) -\wb R_n(\check h) + \wb R_n(\check h) - \wb R_n(h^{\star}) } + n \cdot \wb \rho(n, \wb \delta).
\end{align*}
We next only need to upper bound $\wb R_n(\check h) - \wb R_n(h^{\star})$, and show that it is order-wise smaller than $n \cdot \rho(n ,\wb \delta)$.
We consider random variable $g_i \ldef \ind(\check h(x_i) \neq y_i) - \ind(h^{\star}(x_i) \neq y_i)$.
We have 
\begin{align*}
	\V(g_i) 
	& \leq \E \brk{g_i^2}\\
	& = \E \brk{\ind(\check h(x_i) \neq h^{\star}(x_i)} \\
	& = O \prn*{ \eps^{\frac{\beta}{1+\beta}}  },
\end{align*}
where the last line follows from \cref{lm:bernstein} and the assumption that $\err(\check h) - \err(h^{\star}) = O(\eps)$.
Denote $g = \frac{1}{n} \sum_{i=1}^{n} g_i = \frac{1}{n} \prn{\wb R_n(\check h) - \wb R_n(h^{\star})}$, and notice that $\E \brk{g} = \err(\check h) - \err(h^{\star})$. 
  Applying Bernstein inequality (e.g., Lemma B.9 in \citet{shalev2014understanding}) on $g - \E\sq{g}$, with probability at least  $1- \frac{\wb \delta}{2}$, we have 
\begin{align*}
	g - \E \brk{g} \leq O \prn*{\prn*{\frac{\eps^{\frac{\beta}{1+\beta}} \log \wb \delta^{-1}} {n}}^{\frac{1}{2}} + \frac{\log \wb \delta^{-1}}{n} },
\end{align*}
which further leads to 
\begin{align*}
	\wb R_n(\check h) - \wb R_n (h^{\star}) \leq n \cdot O\prn*{ \eps + \prn*{\frac{\eps^{\frac{\beta}{1+\beta}} \log \wb \delta^{-1}} {n}}^{\frac{1}{2}} + \frac{\log \wb \delta^{-1}}{n}}.
\end{align*}

The RHS is order-wise smaller than $\rho_n$ when  $\beta = 0$. We consider the case when  $\beta > 0$ next.
Since $\log (\wb \delta^{-1}) /n $ is clearly a lower-order term compared to $\rho_n$, we only need to show that 
$\prn*{\frac{\eps^{\frac{\beta}{1+\beta}} \log \wb \delta^{-1}} {n}}^{\frac{1}{2}} $ is order-wise smaller than  $\rho_n$.
We can easily check that 
\begin{align*}
	\prn*{\frac{\eps^{\frac{\beta}{1+\beta}} \log \wb \delta^{-1}} {n}}^{\frac{1}{2}} \leq \prn*{\frac{\vcd(\cH) \cdot \log n \cdot  \log \wb \delta^{-1} }{n}}^{\frac{1+\beta}{2+\beta}}
\end{align*}
whenever $n$ satisfies the following condition
\begin{align*}
	n \leq \eps^{- \frac{2+\beta}{1+\beta}} \cdot \vcd(\cH)^{\frac{2+2\beta}{\beta}} \cdot \log (\wb \delta^{-1}) \cdot \prn{\log n}^{\frac{2+2\beta}{\beta}}.
\end{align*}
\end{proof}

We denote $\check h \ldef \argmin_{h \in \cH} \err(h)$, which satisfies $\err(\check h) - \err(h^{\star}) = O(\eps)$ (as assumed in \cref{thm:RCAL_gen}).
For any $h \in \cH$,
we also use the shorthand $\wb R_m(h) = \wb R_{\tau_{m-1}} (h) \ldef \sum_{t=1}^{\tau_{m-1}} \ind(h(x_t) \neq y_t)$. Note that $\wb R_m$ is only used in analysis since some  $y_t$ are not observable.

\begin{lemma}
	\label{lm:RCAL_set}
	With probability at least $1-\frac{\delta}{2}$, the following holds true for all epochs $m \in [M]$:
	\begin{enumerate}
		\item $\check h \in \cH_m$.
		\item $\err(h) - \err(h^{\star}) \leq 3 \rho_m, \forall h \in \cH_m$. 
	\end{enumerate}

\end{lemma}
\begin{proof}
	For each $m = 2, 3, \ldots, M$, we invoke \cref{lm:tsy_concentration} with $n = \tau_{m-1}$ and  $\wb \delta = \delta /2M$, which guarantees that 
 \begin{align}
	 \tau_{m-1} \cdot \prn{ \err(h) - \err(h^{\star})} & \leq 2 \cdot \prn{\wb R_m(h) - \wb R_m(\check h)} + \tau_{m-1} \cdot \rho_m, \quad \forall h \in \cH, \label{eq:RCAL_set_1} \\
	 \wb R_m(\check h)  - \min_{h \in\cH} \wb R_m(h) & \leq  \tau_{m-1}\cdot \rho_m. \label{eq:RCAL_set_2}
\end{align}
  Note that the choice of $T$ used in \cref{alg:RCAL} ensures that (1) the requirement needed for $n $ in \cref{lm:tsy_concentration} when  $\beta > 0$ is satisfied, and (2) the second term $\eps$ in $\rho(\tau_{m-1}, \delta /2M)$ (see \cref{lm:tsy_concentration} for definition of $\rho(\tau_{m-1}, \delta /2M)$) is a lower-order term compared to the first term.
We use $\cE$ to denote the good event where \cref{eq:RCAL_set_1} and  \cref{eq:RCAL_set_2} hold true across $m = 2, 3, \ldots, M$. This good event happens with probability at least $1-\frac{\delta}{2}$. We analyze under $\cE$ in the following.

We prove \cref{lm:RCAL_set} through induction. The statements clearly hold true for  $m = 1$. Suppose  the statements hold true up to epoch $m$, we next prove the correctness for epoch  $m+1$. 

We know that $\check h \in \cH_m$ by assumption. Based on the querying criteria of  \cref{alg:RCAL}, we know that 
\begin{align}
	\label{eq:RCAL_equiv}
	\wh R_{m+1} (\check h) - \wh R_{m + 1} (h) 
	&  = \wb R_{m+1} (\check h) - \wb R_{m+1} (h), \quad \forall h \in \cH_m
\end{align}
  From \cref{eq:RCAL_set_2} (at epoch $m+1$), we also have 
\begin{align*}
\wb R_{m+1} (\check h) - \min_{h \in \cH_m} R_{m+1} (h)
& \leq 
\wb R_{m+1} (\check h) - \min_{h \in \cH} R_{m+1} (h)\\
& \leq \tau_m \cdot \rho_{m+1}.
\end{align*}
Combining the above two inequalities leads to
\begin{align*}
	\wh R_{m+1} (\check h) - \wh R_{m + 1} (h)  \leq \tau_m \cdot \rho_{m+1},
\end{align*}
implying that $\check h \in \cH_{m+1}$ (due to the construction of $\cH_{m+1}$ in \cref{alg:RCAL}).

Based on \cref{eq:RCAL_equiv}, the construction of $\cH_{m+1}$, and the fact that  $\check h \in \cH_{m}$,
we know that, for any $h \in \cH_{m+1} \subseteq \cH_m$, 
\begin{align*}
	\wb R_{m+1} (h) - \wb R_{m+1} (\check h)
	& = \wh R_{m+1}(h) - \wh R_{m+1} (\check h)\\
	& \leq \wh R_{m+1} (h) - \min_{h \in \cH_m} \wh R_{m+1} (h)\\
	& \leq \tau_m \cdot \rho_{m+1}.
\end{align*}
Plugging the above inequality into \cref{eq:RCAL_set_1} (at epoch $m+1$) leads to  $\err(h) - \err(h^{\star}) \leq 3 \rho_{m+1}$ for any $h \in \cH_{m+1}$.  
We thus prove the desired statements at epoch  $m+1$.
\end{proof}

\subsection{Proof of \cref{thm:RCAL_gen}}
\label{app:RCAL_gen_proof}

\thmRCALGen*
\begin{proof}
	Based on \cref{lm:RCAL_set}, we know that, with probability at least $1- \frac{\delta}{2}$, we have 
	\begin{align*}
		\err(\wh h) - \err(h^{\star})
		& \leq 3 \rho_M \\
		& = O \prn*{ \prn*{\frac{\vcd(\cH) \cdot \log (\tau_{M-1}) \cdot  \log (M /\delta) }{\tau_{M-1}}}^{\frac{1+\beta}{2+\beta}}  }\\
		& = \wt O (\eps),
	\end{align*}
	where we use the definition of $T$ and  $\tau_M$.

	We next analyze the label complexity of \cref{alg:RCAL}. 
	Since \cref{alg:RCAL} stops and the beginning at epoch $M$, we only need to calculated the label complexity in the first  $M-1$ epochs.
	We have 
	\begin{align*}
		\sum_{t=1}^{\tau_{M-1}} Q_t
    & = \sum_{m=1}^{M-1} \sum_{t = \tau_{m-1}+ 1}^{\tau_m}\ind \prn{x_t \in \DIS(\cH_m)}\\
		& \leq \sum_{m=1}^{M-1} \sum_{t = \tau_{m-1}+ 1}^{\tau_m} \ind \prn*{x_t \in \DIS(\cB_{\cH}(h^{\star}, c^{\prime} \prn{3 \rho_m}^{\frac{\beta}{1+\beta}}))},
	\end{align*}
	where on the second line we use the facts
	(1) $\err(h) - \err(h^{\star}) \leq 3 \rho_m, \forall h \in \cH_m$ from \cref{lm:RCAL_set},
  and (2) $\P_{x \sim \cD_\cX}\prn{ h(x) \neq h^{\star}(x)} \leq c^{\prime} \prn{\err(h) - \err(h^{\star})}^{\frac{\beta}{1+\beta}}$ from \cref{lm:bernstein} (with the same constant  $c^{\prime}$).
	Suppose $\err(\check h ) - \err(h^{\star}) = c^{\prime \prime} \eps$ with a universal constant $c^{\prime \prime}$ by assumption. Applying \cref{lm:bernstein} on $\check h$ leads to the fact that 
	 $h^{\star} \in \cB_{\cH} ( \check h, c^{\prime}\prn{c^{\prime \prime} \eps}^{\frac{\beta}{1+\beta}})$.
   Since $\P_{x \sim \cD_\cX}( h(x) \neq \check h(x)) \leq \P_{x \sim \cD_\cX} ( h(x) \neq h^{\star}(x)) + \P_{x \sim \cD_\cX} (h^{\star}(x) \neq \check h(x))$, 
	 we further have
	 \begin{align*}
	 	\sum_{t=1}^{\tau_{M-1}}Q_t
		& \leq \sum_{m=1}^{M-1}\sum_{t = \tau_{m-1}+ 1}^{\tau_m} \ind \prn*{x_t \in \DIS(\cB_{\cH}(\check h, \wb c\cdot {\rho_m}^{\frac{\beta}{1+\beta}}))},
	 \end{align*}
	 with a universal constant $\wb c > 0$.
	Noticing that the RHS is a sum of independent Bernoulli random variables,
applying Chernoff bound leads to the following guarantees on an event $\cE^{\prime}$ that happens with probability at least $1-\frac{\delta}{2}$:
\begin{align*}
	 	\sum_{t=1}^{\tau_{M-1}}Q_t
		& \leq 2e \sum_{m=1}^{M-1}\sum_{t = \tau_{m-1}+ 1}^{\tau_m} \P \prn*{x \in \DIS(\cB_{\cH}(\check h, \wb c \cdot {\rho_m}^{\frac{\beta}{1+\beta}}))} + 2 \log(2 /\delta) \\
  & = 2e \sum_{m=1}^{M-1} \prn{\tau_m - \tau_{m-1}} \cdot \P \prn*{x \in \DIS(\cB_{\cH}(\check h, \wb c \cdot {\rho_m}^{\frac{\beta}{1+\beta}}))} + 2 \log(2 /\delta) \\
		& \leq 2e \sum_{m=2}^{M-1} { \tau_{m-1} } \cdot \theta_{\cH, \check h}\prn*{\wb c \cdot \rho_m^{\frac{\beta}{1+\beta}}} \cdot \wb c \cdot \rho_m^{\frac{\beta}{1+\beta}} + 2 \log (2 /\delta) + 4e\\
		& \leq 2e M \cdot \theta_{\cH, \check h}\prn*{\wb c \cdot \rho_M^{\frac{\beta}{1+\beta}}} \cdot  \prn*{\wb c\cdot \tau_{M-1}\cdot \rho_M^{\frac{\beta}{1+\beta}} } + 2 \log( 2 /\delta) + 4e,
\end{align*}
  where the third line follows from the definition of disagreement coefficient, and the last line follows from the facts that $\crl{\rho_m}$ is a non-increasing sequence yet $\crl{\tau_{m-1} \cdot \rho_m}$ is an increasing sequence.
Basic algebra and basic properties of the disagreement coefficient (i.e., Theorem 7.1 and Corollary 7.2 in \citet{hanneke2014theory}) shows that 
\begin{align*}
	\sum_{t=1}^{\tau_{M-1}} Q_t \leq \wt O \prn*{ \theta_\cH(\eps^{\frac{\beta}{1+\beta}}) \cdot \eps^{-\frac{2}{1+\beta}}\cdot \vcd(\cH)  } ,
\end{align*}
under event $\cE \cap \cE^{\prime}$, which happens with probability at least $1-\delta$.
\end{proof}

\section{Omitted details for \cref{sec:noise_active}}
We prove \cref{thm:active_noise} in \cref{app:active_noise_proof} and discuss the disagreement coefficient in \cref{app:dis_coeff}.

\subsection{Proof of \cref{thm:active_noise}}
\label{app:active_noise_proof}

\thmActiveNoise*
\begin{proof}
	Construct $\cH_\dnn$ based on \cref{prop:vc_approx} such that $\min_{h \in \cH_\dnn} \err(h) - \err(h^{\star}) = O(\eps)$ and $\vcd(\cH_\dnn) = \wt O \prn{ \eps^{-\frac{d}{\alpha(1+\beta)}} }$. Taking such $\cH_\dnn$ into \cref{thm:RCAL_gen} leads to the desired result.
\end{proof}

\subsection{Discussion on disagreement coefficient in \cref{thm:active_noise}}
\label{app:dis_coeff}
We discuss cases when the (classifier-based) disagreement coefficient with respect to a set of neural networks is well-bounded.
As mentioned before, even for simple classifiers such as linear functions, the disagreement coefficient has been analyzed under additional assumptions \citep{friedman2009active, hanneke2014theory}. 
In this section, we analyze the disagreement coefficient for a set of neural networks under additional assumptions on $\cD_{\cX \cY}$ and $\cH_{\dnn}$ (assumptions on $\cH_{\dnn}$ can be implemented via proper preprocessing steps).
We leave a more comprehensive investigation of the disagreement coefficient for future work.

The first case is when $\cD_\cX$ is supported on countably many data points. 
The following result show strict improvement over passive learning.
\begin{definition}[Disagreement core]
	\label{def:dis_core}
	For any hypothesis class $\cH$ and classifier $h$, the disagreement core of  $h$ with respect to $\cH$  under $\cD_{\cX \cY}$ is defined as 
	\begin{align}
		\partial_{\cH} h \ldef \lim_{r \rightarrow 0} \DIS \prn{ \cB_{\cH} (h, r) }.
	\end{align}
\end{definition}

\begin{proposition}[Lemma 7.12 and Theorem 7.14 in \citet{hanneke2014theory}]
For any hypothesis class $\cH$ and classifier $h$, we have 
$\theta_h(\eps) = o(1 /\eps)$ if and only if  $\cD_{\cX} \prn{ \partial_{\cH} h} = 0$.
In particular, this implies that 
$\theta_{\cH} (\eps) = o ( 1/ \eps)$ whenever $\cD_\cX$ is supported on countably many data points.	
\end{proposition}

We now discuss conditions under which we can upper bound the disagreement coefficient by $O(1)$, which ensures results in \cref{thm:active_noise} matching the minimax lower bound for active learning, up to logarithmic factors.
We introduce the following \emph{decomposable} condition.
\begin{definition}
	\label{def:decomposable}
	A marginal distribution $\cD_\cX$ is $\eps$-decomposable if its (known) support $\supp(\cD_\cX)$ can be decomposed into connected subsets, i.e.,  $\supp(\cD_\cX) = \cup_{i \in \cI} \cX_i$, such that 
\begin{align*}
	\cD_\cX \prn{\cup_{i \in \cI^{\prime}} \cX_i} = O( \eps),
\end{align*}
where $\cI^{\prime} \ldef \crl{i \in \cI: \cD_\cX(\cX_i) \leq \eps}$.
\end{definition}

\begin{remark}
	Note that \cref{def:decomposable} permits a decomposition such that $\abs{\wb \cI} = \Omega(\frac{1}{\eps})$ where $\wb \cI = \cI \setminus \cI^{\prime}$. 
	\cref{def:decomposable} requires no knowledge of the index set $\cI$ or any $\cX_i$; it also places
	no restrictions on the conditional probability on each $\cX_i$.
\end{remark}

We first give results for a general hypothesis class $\cH$ as follows, and then discuss how to bound the disagreement coefficient for a set of neural networks. 

\begin{proposition}
	\label{prop:dis_coeff_Oh1}
Suppose $\cD_\cX$ is decomposable (into $\cup_{i \in \cI} \cX_i$) and
the hypothesis class $\cH$ consists of classifiers whose predication on each $\cX_i$ is the same, i.e., $\abs{\crl{h(x): x \in \cX_i}} = 1$ for any $h \in \cH$ and $i \in \cI$.
We then have $\theta_\cH(\eps) = O(1)$ for  $\eps$ sufficiently small. 
\end{proposition}

\begin{proof}
Fix any $h \in \cH$. 
we know that for any $h^{\prime} \in \cB_\cH (h, \eps)$, we must have $\DIS(\crl{h, h^{\prime}}) \subseteq \cup_{i \in \cI^{\prime}} \cX_i$ since 
  $\cD_\cX \prn{x \in \cX: h(x) \neq h^{\prime}(x)} \leq \eps$, and $\abs{\crl{h(x): x \in \cX_i}} = 1$ for any  $h \in \cH$ and any $\cX_i$. 
This further implies that $\P \prn{ \DIS(\cB_{\cH} (h, \eps)} = O(\eps)$, and thus  $\theta_{\cH}(\eps) = O(1)$.
\end{proof}

We next discuss conditions under which we can satisfy the prerequisites of \cref{prop:dis_coeff_Oh1}. 
Suppose $\cD_{\cX \cY} \in \cP(\alpha, \beta)$. 
We assume that $\cD_\cX$ is  $(\eps^{\frac{\beta}{1+\beta}})$-decomposable, and, for the desired accuracy level $\eps$, we have 
\begin{align}
	\label{eq:supp_far_away_one_half}
	\abs{\eta(x) - {1}/{2}} \geq 2 \eps^{\frac{1}{1+\beta}}, \quad \forall x \in \supp(\cD_\cX).
\end{align}

With the above conditions satisfied, we can 
filter out neural networks that are clearly not ``close'' to $\eta$. Specifically, with $\kappa = \eps^{\frac{1}{1+\beta}}$ and $\cF_\dnn$ be the set of neural networks constructed from \cref{prop:vc_approx}, we consider 
\begin{align}
	\wt \cF_\dnn \ldef \crl{ f \in \cF_\dnn: \abs{f(x) - 1 /2} \geq \eps^{\frac{1}{1+\beta}}, \forall x \in \supp(\cD_\cX)},
\end{align}
which is guaranteed to contain $\wb f \in \cF_\dnn$ such that $\nrm{\wb f - \eta}_\infty \leq \eps^{\frac{1}{1+\beta}}$.
Now focus on the subset 
\begin{align}
	\wt \cH_{\dnn} \ldef \crl{ h_f: f \in \wt\cF_\dnn}.
\end{align}
We clearly have $h_{\wb f} \in \wt \cH_{\dnn}$ (which ensures an $O(\eps)$-optimal classifier) and  $\vcd(\wt \cH_\dnn) \leq \vcd(\cH_\dnn)$ (since $\wt \cH_\dnn \subseteq \cH_\dnn$).
We upper bound the disagreement coefficient $\theta_{\wt \cH_\dnn}(\eps^{\frac{\beta}{1+\beta}})$ next.

\begin{proposition}
	\label{prop:dis_coeff_Oh1_nn}
	Suppose $\cD_{\cX \cY} \in \cP(\alpha, \beta)$ such that $\cD_\cX$ is  $(\eps ^{\frac{\beta}{1+\beta}})$-decomposable and \cref{eq:supp_far_away_one_half} is satisfied (with the desired accuracy level $\eps$).
	We then have $\theta_{\wt \cH_\dnn}(\eps^{\frac{\beta}{1+\beta}}) = O(1)$.
\end{proposition}

\begin{proof}
The proof is similar to the proof of \cref{prop:dis_coeff_Oh1}.	
Fix any $h = h_f \in \wt \cH_\dnn$. 
We first argue that, for any $i \in \cI$, under \cref{eq:supp_far_away_one_half}, $\abs{\crl{h_f(x): x \in \cX_i}} = 1$, i.e.,  for $x \in \cX_i$, $h_f(x)$ equals either $1$ or  $0$, but not both:
This can be seen from the fact that any $f \in \wt \cF_\dnn$ is continuous and satisfies $\abs{f(x) - 1 /2} \geq \eps^{\frac{1}{1+\beta}}$ for any  $x \in \cX_i$.

Fix any $h \in \wt \cH_\dnn$. 
We know that for any $h^{\prime} \in \cH_{\wt \cH_\dnn} (h, \eps^{\frac{\beta}{1+\beta}})$, we must have $\DIS(\crl{h, h^{\prime}}) \subseteq {\cup_{i \in \cI^{\prime}} \cX_i} $ due to similar reasons argued in the proof of \cref{prop:dis_coeff_Oh1}. 
This further implies that $\P \prn{ \DIS(\cB_{\wt \cH_\dnn} (h, \eps^{\frac{\beta}{1+\beta}})} = O(\eps^{\frac{\beta}{1+\beta}})$, and thus  $\theta_{\wt \cH_\dnn}(\eps^{\frac{\beta}{1+\beta}}) = O(1)$.
\end{proof}

We next argue that \cref{eq:supp_far_away_one_half} is only needed in an approximate sense. 
We define the approximate decomposable condition in the following.

\begin{definition}
	\label{def:decomposable_approx}
	A marginal distribution $\cD_\cX$ is $(\eps,\delta)$-decomposable if there exists a known subset $\wb \cX \subseteq \supp(\cD_\cX)$ such that 
\begin{align}
	\cD_\cX(\wb \cX) \geq 1- \delta,
\end{align}
and it can be decomposed into connected subsets, i.e.,  $\wb \cX = \cup_{i \in \cI} \cX_i$, such that 
\begin{align*}
	\cD_\cX \prn{\cup_{i \in \cI^{\prime}} \cX_i} = O( \eps),
\end{align*}
where $\cI^{\prime} \ldef \crl{i \in \cI: \cD_\cX(\cX_i) \leq \eps}$.
\end{definition}

Suppose $\cD_{\cX \cY} \in \cP(\alpha, \beta)$. 
We assume that $\cD_\cX$ is  $(\eps^{\frac{\beta}{1+\beta}}, \eps^{\frac{\beta}{1+\beta}})$-decomposable (wrt $\wb \cX \subseteq \cD_\cX$), and, for the desired accuracy level $\eps$, we have 
\begin{align}
	\label{eq:supp_far_away_one_half_approx}
	\abs{\eta(x) - {1}/{2}} \geq 2 \eps^{\frac{1}{1+\beta}}, \quad \forall x \in \wb \cX.
\end{align}

With the above conditions satisfied, we can 
filter out neural networks that are clearly not ``close'' to $\eta$. Specifically, with $\kappa = \eps^{\frac{1}{1+\beta}}$ and $\cF_\dnn$ be the set of neural networks constructed from \cref{prop:vc_approx}, we consider 
\begin{align}
	\wb \cF_\dnn \ldef \crl{ f \in \cF_\dnn: \abs{f(x) - 1 /2} \geq \eps^{\frac{1}{1+\beta}}, \forall x \in \wb \cX},
\end{align}
which is guaranteed to contain $\wb f \in \cF_\dnn$ such that $\nrm{\wb f - \eta}_\infty \leq \eps^{\frac{1}{1+\beta}}$.
Now focus on the subset 
\begin{align}
	\wb \cH_{\dnn} \ldef \crl{ h_f: f \in \wb \cF_\dnn}.
\end{align}
We clearly have $h_{\wb f} \in \wb \cH_{\dnn}$ (which ensures an $O(\eps)$-optimal classifier) and  $\vcd(\wb \cH_\dnn) \leq \vcd(\cH_\dnn)$ (since $\wb \cH_\dnn \subseteq \cH_\dnn$).
We upper bound the disagreement coefficient $\theta_{\wb \cH_\dnn}(\eps^{\frac{\beta}{1+\beta}})$ next.

\begin{proposition}
	\label{prop:dis_coeff_Oh1_nn_approx}
	Suppose $\cD_{\cX \cY} \in \cP(\alpha, \beta)$ such that $\cD_\cX$ is  $(\eps ^{\frac{1}{1+\beta}}, \eps)$-decomposable (wrt known $\wb \cX \subseteq \supp(\cD_\cX)$) and \cref{eq:supp_far_away_one_half_approx} is satisfied (with the desired accuracy level $\eps$).
	We then have $\theta_{\wb \cH_\dnn}(\eps^{\frac{\beta}{1+\beta}}) = O(1)$.
\end{proposition}

\begin{proof}
The proof is the same as the proof of \cref{prop:dis_coeff_Oh1_nn_approx} except 
for any $h^{\prime} \in \cH_{\wb \cH_\dnn} (h, \eps^{\frac{\beta}{1+\beta}})$, we must have $\DIS(\crl{h, h^{\prime}}) \subseteq \prn{\cup_{i \in \cI^{\prime}} \cX_i} \cup \prn{\supp(\cD_\cX) \setminus \wb \cX}$. 
Based on the assumption that $\cD_\cX$ is  $(\eps ^{\frac{1}{1+\beta}}, \eps)$-decomposable, this also leads to $\theta_{\wb \cH_\dnn}(\eps^{\frac{\beta}{1+\beta}}) = O(1)$.
\end{proof}

\section{Generic version of \cref{alg:NCALP} and its guarantees}
\label{app:abs_gen}

This section is organized as follows.
We first introduce some complexity measures in \cref{app:value_func_complexity}. 
We then provide the generic algorithm (\cref{alg:abs}) and state its theoretical guarantees (\cref{thm:abs_gen}) in \cref{app:abs_gen_class}.

\subsection{Complexity measures}
\label{app:value_func_complexity}

We first introduce \emph{pseudo dimension} \citep{pollard1984convergence, haussler1989decision, haussler1995sphere}, a complexity measure used to analyze real-valued functions.

\begin{definition}[Pseudo dimension]
\label{def:pseudo_d}
Consider a set of real-valued function $\cF: \cX \rightarrow \R$. The pseudo dimension $\pseud(\cF)$ of $\cF$ is defined as the VC dimension of the set of threshold functions 
$\crl{(x,\zeta) \mapsto \ind(f(x) > \zeta) : f \in \cF}$.
\end{definition}

As discussed in \citet{bartlett2019nearly}, similar results as in \cref{thm:vcd_nn} holds true for $\pseud(\cF)$ as well.

\begin{theorem}[\citet{bartlett2019nearly}]
	\label{thm:pdim_nn}
	Let $\cF_{\dnn}$ be a set of neural network regression functions of the same architecture and with $W$ parameters arranged in  $L$ layers.
	We then have 
	\begin{align*}
	\Omega(WL \log \prn*{{W}/{L}}) \leq \pseud(\cF_\dnn)  \leq O(WL \log \prn*{ W}). 
	\end{align*}
\end{theorem}

We now introduce \emph{value function disagreement coefficient}, which is proposed by \citet{foster2020instance} in contextual bandits and then adapted to active learning by \citet{zhu2022efficient} with additional supreme over the marginal distribution $\cD_\cX$ to deal with distributional shifts caused by selective sampling.

\begin{definition}[Value function disagreement coefficient]
    \label{def:dis_coeff_value}
    For any $f^{\star} \in \cF$ and $\gamma_0 ,\epsilon_0 > 0$, the value function disagreement coefficient $	 \theta^{\val}_{f^{\star}}(\cF, \gamma_0, \eps_0)$ is defined as 
    \begin{align*}
 \sup_{\cD_\cX}\sup_{\gamma> \gamma_0, \eps> \eps_0} 
	 \crl*{ \frac{\gamma^2}{\epsilon^2} \cdot 
	\P_{\cD_\cX} \prn*{ \exists f \in \cF: \abs{f(x) - f^{\star}(x)} > \gamma,
	\nrm*{ f - f^{\star}}_{\cD_\cX} \leq \eps} } \vee 1,
    \end{align*}
    where $\nrm{f}^2_{\cD_\cX} \ldef \E_{x \sim \cD_\cX} \brk{f^2(x)}$. We also define $\theta^{\val}_\cF (\gamma_0) \ldef \sup_{f^{\star} \in \cF, \eps_0 > 0} \theta^{\val}_{f^{\star}}(\cF, \gamma_0, \eps_0)$.
\end{definition}

\subsection{The generic algorithm and its guarantees}
\label{app:abs_gen_class}
We present \cref{alg:abs}, a generic version of \cref{alg:NCALP} that doesn't require the approximating classifiers to be neural networks.

\begin{algorithm}[H]
	\caption{\textNCALP (Generic Version)}
	\label{alg:abs} 
	\renewcommand{\algorithmicrequire}{\textbf{Input:}}
	\renewcommand{\algorithmicensure}{\textbf{Output:}}
	\newcommand{\algorithmicbreak}{\textbf{break}}
    \newcommand{\BREAK}{\STATE \algorithmicbreak}
	\begin{algorithmic}[1]
		\REQUIRE Accuracy level $\epsilon \in (0, 1)$, confidence level $\delta \in (0, 1)$, abstention parameter $\gamma \in (0, 1/2)$.
		\STATE Let $\cF: \cX \rightarrow [0, 1]$ be a set of regression functions such that there exists a regression function $\wb f \in \cF$ with $\nrm{\wb f - \eta}_\infty \leq \kappa \leq \gamma / 4$.
		\STATE Define $T \ldef \frac{\theta^{\val}_\cF(\gamma / 4) \cdot \pseud(\cF) }{\eps \, \gamma }$, $M \ldef \ceil{\log_2 T}$, and $C_\delta \ldef O\prn{\pseud(\cF) \cdot \log(T /\delta)}$.
		\STATE Define $\tau_m \ldef 2^m$ for $m\geq1$, $\tau_0 \ldef 0$, and $\beta_m \ldef 3(M-m+1) C_\delta$. 
		\STATE Define  $\wh R_m(f) \ldef \sum_{t=1}^{\tau_{m-1}} Q_t \prn{\wh f(x_t) - y_t}^2 $ with the convention that $\sum_{t=1}^{0} \ldots = 0$.
		\FOR{epoch $m = 1, 2, \dots, M$}
		\STATE Get $\widehat f_m \ldef \argmin_{f \in \cF} \sum_{t=1}^{\tau_{m-1}} Q_t \paren{f(x_t) - y_t}^2 $.
		\STATE 
		(Implicitely) Construct active set 
		   $ \cF_m \ldef \crl*{ f \in \cF:  \wh R_m(f) \leq  \wh R_m(\wh f_m) + \beta_m }$.
		\STATE Construct classifier $\wh h_m: \cX \rightarrow \crl{0, 1,\bot}$ as 
		\begin{align*}
			\wh h_m (x) \ldef 
			\begin{cases}
				\bot, & \text{ if } \brk { \lcb(x;\cF_m) - \frac{\gamma}{4}, \ucb(x;\cF_m) + \frac{\gamma}{4}} \subseteq 
				\brk*{ \frac{1}{2} - \gamma, \frac{1}{2} + \gamma}; \\
        \ind(\wh f_m(x) \geq \frac{1}{2} ) , & \text{o.w.}
			\end{cases}
		\end{align*}
		and query function 
			$g_m(x)\ldef \ind \prn*{ \frac{1}{2} \in \prn*{\lcb(x;\cF_m) - \frac{\gamma}{4}, \ucb(x;\cF_m) + \frac{\gamma}{4}} } \cdot
		\ind \prn{\wh h_m(x) \neq \bot}$.
		\IF{epoch $m=M$}
		\STATE \textbf{Return} classifier $\wh h_M$.
		\ENDIF
		\FOR{time $t = \tau_{m-1} + 1 ,\ldots , \tau_{m} $} 
		\STATE Observe $x_t \sim \cD_{\cX}$. Set $Q_t \ldef g_m(x_t)$.
		\IF{$Q_t = 1$}
		\STATE Query the label $y_t$ of $x_t$.
		\ENDIF
		\ENDFOR
		\ENDFOR

	\end{algorithmic}
\end{algorithm}

We next state the theoretical guarantees for \cref{alg:abs}.

\begin{restatable}{theorem}{thmAbsGen}
	\label{thm:abs_gen}
	Suppose $\theta^{\val}_\cF(\gamma / 4) \leq \wb \theta$ and the approximation level $\kappa \in (0, \gamma /4]$ satisfies 
 \begin{align}
	 \label{eq:kappa_requirement}
	\prn*{\frac{432 \wb \theta \cdot M^2}{\gamma^2}} \cdot \kappa^2  
	\leq \frac{1}{10}.
\end{align}
	With probability at least  $1-\delta$, \cref{alg:abs} returns a classifier $\wh h: \cX \rightarrow \crl{0, 1, \bot}$ with Chow's excess error 
\begin{align*}
	\exc_\gamma(\wh h) =
	 O \prn*{  \eps \cdot \log \prn*{\frac{\wb \theta \cdot \pseud(\cF)}{\eps \, \gamma \, \delta}}}, 
\end{align*}
after querying at most 
\begin{align*}
O\prn*{  \frac{ M^2 \cdot \pseud(\cF) \cdot \log(T /\delta) \cdot \wb \theta }{\gamma^2}} 
\end{align*}
labels.
\end{restatable}

\cref{thm:abs_gen} is proved in \cref{app:abs_gen_proof}, based on supporting lemmas and theorems established in \cref{app:concentration_abs} and \cref{app:supporting_abs}.
The general result (\cref{thm:abs_gen}) will be used to prove results in specific settings (e.g., \cref{thm:abs} and \cref{thm:abs_RBV}).

\subsubsection{Concentration results}
\label{app:concentration_abs}

\begin{lemma}[Freedman's inequality \citep{freedman1975tail, agarwal2014taming}]
    \label{lm:freedman}
    Let $(X_t)_{t \leq T}$ be a real-valued martingale difference sequence adapted to a filtration $\mfF_t$, and let $\E_t \sq{\cdot} \ldef \E \sq{\cdot \mid \mfF_{t-1}}$. If $\abs{X_t} \leq B$ almost surely, then for any $\eta \in (0,1/B)$ it holds with probability at least $1 - \delta$,
    \begin{align*}
        \sum_{t=1}^{T} X_t \leq \eta \sum_{t=1}^{T} \E_{t} \sq{X_t^2} + \frac{\log \delta^{-1}}{\eta}.
    \end{align*}
\end{lemma}

\begin{lemma}[\citep{foster2020instance}]
   \label{lm:martingale_two_sides} 
   Let $(X_t)_{t \leq T}$ be a sequence of random variables adapted to a filtration $\mfF_t$. If $0 \leq {X_t} \leq B$ almost surely, then with probability at least $1-\delta$,
   \begin{align*}
       \sum_{t =1 }^T X_t \leq \frac{3}{2} \sum_{t=1}^T \E_{t}\sq{X_t} + 4B \log(2 \delta^{-1}),
   \end{align*}
   and 
   \begin{align*}
       \sum_{t =1 }^T \E_{t} \sq{X_t} \leq 2 \sum_{t=1}^T X_t + 8B \log(2 \delta^{-1}).
   \end{align*}
\end{lemma}
\begin{proof}
  These two inequalities are obtained by applying \cref{lm:freedman} to $\prn{X_t - \E_t \sq{X_t}}_{t \leq T}$ and $\prn{\E_t \sq{X_t} - X_t}_{t \leq T}$, with $\eta = 1/2B$ and $\delta / 2$.	
  Note that $\E_t \sq{\prn{X_t - \E_t \sq{X_t}}^2} \leq \E_t \sq{X_t^2} \leq B \E_t\sq{X_t}$ if $0 \leq X_t \leq B$.
\end{proof}

We now define/recall some notations.
Denote $n_m \ldef \tau_{m} - \tau_{m-1}$.
Fix any epoch $m \in [M]$ and any time step $t$ within epoch $m$.
We have $f^{\star} = \eta$.
For any $f \in \cF$, we denote $M_t(f) \ldef Q_t \prn{ \prn{f(x_t) - y_t}^2 - \prn{f^\star(x_t) - y_t}^2}$, 
and $\wh R_m(f) \ldef \sum_{t=1}^{\tau_{m-1}} Q_t \prn{f(x_t) - y_t}^2$.
Recall that we have $Q_t = g_m(x_t)$.
We define filtration $\mfF_t \ldef \sigma \prn{ \prn{x_1, y_1}, \ldots , \prn{x_{t}, y_{t}}}$,\footnote{$y_t$ is not observed (and thus not included in the filtration) when $Q_t = 0$. Note that $Q_t$ is measurable with respect to $\sigma( (\mfF_{t-1}, x_t) ) $.} 
and denote $\E_t \sq{\cdot } \ldef \E \sq{\cdot \mid \mfF_{t-1}}$.
We next present concentration results with respect to a general set of regression function $\cF$ with finite pseudo dimension.

\begin{lemma}[\citet{krishnamurthy2019active}]
    \label{lm:expected_sq_loss_pseudo}
    Consider an infinite set of regression function $\cF$.
   Fix any $\delta \in (0,1)$.  For any $\tau, \tau^\prime \in [T]$ such that $\tau < \tau^\prime$, with probability at least $1 - \frac{\delta}{2} $, we have 
   \begin{align*}
   	\sum_{t = \tau}^{\tau^\prime} M_t(f) \leq \sum_{t=\tau}^{\tau^\prime} \frac{3}{2} \E_t \brk{M_t(f)} + 
	C_\delta,
   \end{align*}
   and
   \begin{align*}
       \sum_{t = \tau}^{\tau^\prime} \E_t \sq{ M_t(f)}  \leq 2 \sum_{t = \tau}^{\tau^\prime} M_t(f) + C_\delta,
   \end{align*}
   where $C_\delta  = C \cdot \prn*{ \pseud(\cF) \cdot \log T +   \log \prn*{ \frac{\pseud(\cF)  \cdot  T}{\delta}} } $ with a universal constant $C >0$.

\end{lemma}

\subsubsection{Supporting lemmas for \cref{thm:abs_gen}}
\label{app:supporting_abs}
Fix any classifier $\wh h: \cX \rightarrow \crl{0, 1,\bot}$. For any $x\in\cX$, we use the notion
\begin{align}
& 	\exc_\gamma ( \wh h;x) \ldef \nonumber\\
    &  \P_{y\mid x} \prn[\big]{y \neq \widehat h(x)} \cdot \ind \prn[\big]{ \widehat h(x) \neq \bot} + \prn[\big]{{1}/{2} - \gamma} \cdot \1 \prn[\big]{\widehat h(x) = \bot} - \P_{y\mid x} \prn[\big]{ y \neq h^\star(x) }\nonumber\\
    & = \ind \prn[\big]{ \widehat h(x) \neq \bot} \cdot \prn[\big]{\P_{y\mid x} \prn[\big]{y \neq \widehat h(x)} -  \P_{y\mid x} \prn[\big]{ y \neq h^\star(x) }} \nonumber \\
    & \quad + \ind \prn[\big]{ \widehat h(x) = \bot} \cdot \prn[\big]{ \prn[\big]{{1}/{2} - \gamma}  -  \P_{y\mid x} \prn[\big]{ y \neq h^\star(x) }} \label{eq:excess_x}
\end{align}
to represent the excess error of $\wh h$ at point $x\in \cX$. Excess error of classifier $\wh h$ can be then written as $\exc_\gamma (\wh h) \ldef \err_\gamma(\wh h) - \err(h^{\star}) = \E_{x \sim \cD_\cX} \brk{ \exc_\gamma (\wh h;x)}$.

We let $\cE$ denote the good event considered in \cref{lm:expected_sq_loss_pseudo}, we analyze under this event through out the rest of this section.
Most lemmas presented in this section are inspired by results provided \citet{zhu2022efficient}. Our main innovation is an inductive analysis of lemmas that eventually relaxes the requirements for approximation error for \cref{thm:abs_gen}.

\paragraph{General lemmas} We introduce some general lemmas for \cref{thm:abs_gen}. 

\begin{lemma}
\label{lm:query_implies_width}
For any $m \in [M]$, we have $g_m(x)= 1 \implies w(x;\cF_m) > \frac{\gamma}{2}$.
\end{lemma}
\begin{proof}
We only need to show that $\ucb(x;\cF_m) - \lcb(x;\cF_m) \leq \frac{\gamma}{2} \implies g_m(x) = 0$. Suppose otherwise $g_m(x) = 1$, which implies that both 
\begin{align}
\label{eq:query_condition}
&\frac{1}{2} \in \prn*{\lcb(x;\cF_m)-\frac{\gamma}{4}, \ucb(x;\cF_m) + \frac{\gamma}{4}}  \quad \text{ and } \nonumber \\
&{\brk*{\lcb(x;\cF_m) - \frac{\gamma}{4}, \ucb(x;\cF_m) + \frac{\gamma}{4}} \nsubseteq \brk*{  \frac{1}{2}-\gamma, \frac{1}{2} +\gamma } } .
\end{align}
If $\frac{1}{2} \in \prn*{\lcb(x;\cF_m)- \frac{\gamma}{4}, \ucb(x;\cF_m)+ \frac{\gamma}{4}}$ and $\ucb(x;\cF_m) - \lcb(x;\cF_m) \leq \frac{\gamma}{2}$, we must have $\lcb(x;\cF_m) \geq  \frac{1}{2}- \frac{3}{4}\gamma$ and $\ucb(x;\cF_m) \leq \frac{1}{2} + \frac{3}{4}\gamma$, which contradicts with \cref{eq:query_condition}.
\end{proof}

\begin{lemma}
\label{lm:regret_no_query_mis}
Fix any $m \in [M]$.
Suppose $\wb f \in \cF_m$,
we have $\exc_{\gamma}(\wh h_m ;x) \leq 0 $ if $g_m(x) = 0$.
\end{lemma}
\begin{proof}
	Recall that
\begin{align*}
	\exc_{\gamma}( \wh h;x) & =  \nonumber
      \ind \prn[\big]{ \widehat h(x) \neq \bot} \cdot \prn[\big]{\P_{y \mid x} \prn[\big]{y \neq \widehat h(x)} -  \P_{y \mid x} \prn[\big]{ y \neq  h^{\star}(x) }} \nonumber \\
    & \quad + \ind \prn[\big]{ \widehat h(x) = \bot} \cdot \prn[\big]{ \prn[\big]{{1}/{2} - \gamma}  -  \P_{y\mid x} \prn[\big]{ y \neq  h^{\star}(x) }} .
\end{align*}
We now analyze the event $\curly*{g_m(x)= 0}$ in two cases. 

\textbf{Case 1: ${\widehat h_m(x) = \bot} $.} 

Since $\wb f(x) \in [\lcb(x;\cF_m), \ucb(x;\cF_m)]$ and $\kappa \leq \frac{\gamma}{4}$ by assumption, 
we know that $\eta(x) = f^{\star}(x) \in \sq{ \frac{1}{2} - \gamma, \frac{1}{2} + \gamma}$ and thus $\P_{y} \prn[\big]{ y\neq h^\star(x) } \geq \frac{1}{2} - \gamma$. 
As a result, we have $\exc_\gamma(\wh h_m;x) \leq 0$.

\textbf{Case 2: ${\widehat h_m(x) \neq \bot}$ but ${\frac{1}{2} \notin \prn{\lcb(x;\cF_m) - \frac{\gamma}{4}, \ucb(x;\cF_m) + \frac{\gamma}{4}}} $.} 

Since $\wb f(x) \in [\lcb(x;\cF_m), \ucb(x;\cF_m)]$ and $\kappa \leq \frac{\gamma}{4}$ by assumption, 
we clearly have $\widehat h_m (x) = h^\star(x)$ when ${\frac{1}{2} \notin \prn{\lcb(x;\cF_m) - \frac{\gamma}{4}, \ucb(x;\cF_m) + \frac{\gamma}{4}}} $. We thus have 
$\exc_\gamma(\wh h_m;x) \leq 0$.
\end{proof}

\paragraph{Inductive lemmas}
We prove a set of statements for \cref{thm:abs_gen} in an inductive way. 
Fix any epoch $m \in [M]$,
we consider 
\begin{align}
	\label{eq:inductive1}
\begin{dcases}
	 \wh R_m(\wb f) - \wh R_m(f^{\star}) \leq \E_{t} \brk*{ Q_t\prn*{\wb f(x_t) - f^{\star}(x_t)}^2} + C_\delta  \leq \frac{3}{2} C_\delta \\
 \wb f\in \cF_m\\
 \sum_{t=1}^{\tau_{m-1}} \E_t \brk{M_t(f)} \leq 4 \beta_m , \forall f \in \cF_m \\
\sum_{t=1}^{\tau_{m-1}} \E \brk{ Q_t(x_t) \prn{ f(x_t) - \wb f(x_t) }^2} \leq 9 \beta_m,  \forall f \in \cF_m\\
\cF_{m} \subseteq \cF_{m-1}
\end{dcases},
\end{align}
\begin{align}
	\label{eq:inductive2}
    \E_{x \sim \cD_\cX} \sq{\ind (g_m(x)= 1)} \leq \frac{144 \beta_m}{{\tau_{m-1}} \, \gamma^2} \cdot \theta^{\val}_{\wb f}\prn*{\cF, \gamma/4, \sqrt{\beta_m/\tau_{m-1}}} \leq \frac{144 \beta_m}{{\tau_{m-1}} \, \gamma^2} \cdot \wb \theta,
\end{align}
and 
\begin{align}
	\label{eq:inductive3}
    	\E_{x \sim \cD_\cX} \sq{\ind(g_m(x) = 1)\cdot w(x;\cF_m)} \leq  \frac{72 \beta_m}{\tau_{m-1} \gamma} \cdot \theta^{\val}_{\wb f}\prn*{\cF, \gamma/4, \sqrt{\beta_m/\tau_{m-1}}}\leq \frac{72 \beta_m}{\tau_{m-1} \gamma} \cdot \wb \theta.
\end{align}
\begin{lemma}
\label{lm:set_f_mis}
Fix any $\wb m = [M]$.
When $\wb m=1, 2$ or when \cref{eq:inductive2} holds true for epochs $m = 2,3,\dots,\wb m-1$, then 
\cref{eq:inductive1} holds true for epoch $m = \wb m$.
\end{lemma}
\begin{proof}
The statements in \cref{eq:inductive1} clearly hold true for $m = \wb m =1$ since, by definition,  $\cF_0 = \cF$ and  $\sum_{t=1}^{0} \ldots = 0 $.
We thus only need to consider the case when $\wb m\geq 2$.
We next prove each of the five statements in \cref{eq:inductive1} for epoch $m = \wb m$.
\begin{enumerate}
	\item 
	In the case when $\wb m = 2$, from \cref{lm:expected_sq_loss_pseudo}, we know that 
	\begin{align*}
		\wh R_{\wb m}(\wb f) - \wh R_{\wb m}(f^{\star}) & \leq \sum_{t=1}^{\tau_{\wb m-1}} \frac{3}{2} \cdot
		\E_{t} \brk*{ Q_t\prn*{\wb f(x_t) - f^{\star}(x_t)}^2} + C_\delta \\
	& \leq 3 + C_\delta \leq \frac{3}{2} C_\delta,
	\end{align*}
	where the second line follows from the fact that $\tau_1 = 2$ (without loss of generality, we assume $C_\delta \geq 6$ here).

	We now focus on the case when  $\wb m \geq 3$.
	We have 
	\begin{align*}
		\wh R_{\wb m}(\wb f) - \wh R_{\wb m}(f^{\star}) & \leq \sum_{t=1}^{\tau_{\wb m-1}} \frac{3}{2} \cdot
		\E_{t} \brk*{ Q_t\prn*{\wb f(x_t) - f^{\star}(x_t)}^2} + C_\delta \\
	& \leq \frac{3}{2} \sum_{\check m = 1}^{\wb m-1} n_{\check m} \E_{x \sim \cD_\cX} \brk{ \ind(g_{\check m}(x) = 1)} \cdot \kappa^2 + C_\delta\\
	& \leq \frac{3}{2} \prn*{2+ \sum_{\check m = 2}^{\wb m-1} n_{\check m} \frac{144 \beta_{\check m} \cdot \wb \theta}{\tau_{\check m-1} \gamma^2}} \cdot  \kappa^2  + C_\delta \\
	& \leq \prn*{3+ \frac{144 \wb \theta}{\gamma^2} \cdot \prn*{\sum_{\check m = 2}^{\wb m-1}\beta_{\check m}} } \cdot  \kappa^2  + C_\delta \\
	& \leq \prn*{3+ \frac{432 \wb \theta \cdot M^2}{\gamma^2} \cdot  C_\delta } \cdot  \kappa^2  + C_\delta \\
	& \leq  \frac{3}{2} C_\delta,
	\end{align*}
	where the first line follows from \cref{lm:expected_sq_loss_pseudo}; the second line follows from the fact that $\nrm{\wb f - f^{\star}}_{\infty} \leq \kappa$; the third line follows from \cref{eq:inductive2}; the forth line follows from $n_{\check m} = \tau_{\check m - 1}$; the fifth line follows from the definition of $\beta_{\check m}$; and
	the last line follows from the choice of  $\kappa$ in \cref{eq:kappa_requirement}
\item 
Since $\E_t \brk{ M_t(f)} = \E_{t}\brk{Q_t \prn{f(x_t) - f^{\star}(x_t)}^2}$,
by \cref{lm:expected_sq_loss_pseudo}, we have 
$\wh R_{\wb m} (f^\star) \leq \wh R_{\wb m}(f) + C_\delta /2 $ for any $f \in \cF$.
Combining this with statement 1 leads to 
\begin{align*}
	\wh R_{\wb m}(\wb f) 
	& \leq \wh R_{\wb m}(f)  + 2 C_\delta\\
	& \leq \wh R_{\wb m}(f) + \beta_{\wb m}
\end{align*}
for any $f \in \cF$, where the second line follows from the definition of $\beta_{\wb m}$.
We thus have $\wb f \in \cF_{\wb m}$ based on the elimination rule.
\item  Fix any $f \in \cF_{\wb m}$. We have 
	\begin{align*}
		\sum_{t=1}^{\tau_{\wb m-1}} \E_t [M_t(f)] & \leq 2 \sum_{t=1}^{\tau_{\wb m-1}} M_t(f) + C_\delta \\
			& = 2 \wh R_{\wb m }(f) - 2\wh R_{\wb m}(f^{\star}) + C_\delta \\
			& \leq 2 \wh R_{\wb m}(f) - 2\wh R_{\wb m}(\wb f) + 4 C_\delta \\
			& \leq 2 \wh R_{\wb m }(f) - 2\wh R_{\wb m}(\wh f_{\wb m})+ 4 C_\delta \\
			& \leq 2 \beta_{\wb m} + 4 C_\delta \\
			& \leq 4 \beta_{\wb m} , 
	\end{align*}
	where 
the first line follows from \cref{lm:expected_sq_loss_pseudo};
	the third line follows from statement 1; the fourth line follows from the fact that $\wh f_{\wb m}$ is the minimizer of $\wh R_{\wb m} (\cdot) $; and the fifth line follows from the fact that $f \in \cF_{\wb m}$.
\item  Fix any $f \in \cF_{\wb m}$. We have 
	\begin{align*}
		\sum_{t=1}^{\tau_{\wb m-1}} \E_t \brk{ Q_t(x_t) \prn{ f(x_t) - \wb f(x_t) }^2}
		& = \sum_{t=1}^{\tau_{\wb m-1}} \E_t \brk{ Q_t(x_t) \prn{ (f(x_t) - f^{\star}(x_t)) + 
		( f^{\star}(x_t) - \wb f(x_t)) }^2} \\
		& \leq 2 \sum_{t=1}^{\tau_{\wb m-1}} \E_t \brk{ Q_t(x_t) \prn{ f(x_t) -  f^{\star}(x_t) }^2} + 2C_\delta\\
		&  = 2 \sum_{t=1}^{\tau_{\wb m-1}} \E_t [M_t(f)] + 2C_\delta  \\
		& \leq 8 \beta_{\wb m} + 2C_\delta\\
		& \leq 9 \beta_{\wb m},
	\end{align*}
	where the second line follows from $\prn{a+b}^2 \leq 2(a^2 + b^2)$ and (the proof of) statement 1 on the second line; and the fourth line follows from statement 3.
	\item Fix any $f \in \cF_{\wb m}$. We have 
	\begin{align*}
		\wh R_{\wb m-1} (f) - \wh R_{\wb m-1} (\wh f_{\wb m-1}) & \leq   
		\wh R_{\wb m-1} (f) - \wh R_{\wb m-1} (f^{\star}) + \frac{C_\delta}{2}\\
		& = \wh R_{\wb m}(f) - \wh R_{\wb m}(f^{\star}) 
		- \sum_{t=\tau_{\wb m-2}+1}^{\tau_{\wb m-1}} M_t(f) + \frac{C_\delta}{2}\\
		& \leq \wh R_{\wb m}(f) - \wh R_{\wb m} (\wb f) + \frac{3}{2} C_\delta
		- \sum_{t=\tau_{\wb m-2}+1}^{\tau_{\wb m-1}} \E_t [M_t(f)] /2 + {C_\delta}\\
		& \leq \wh R_{\wb m}(f) - \wh R_{\wb m} (\wh f_{\wb m}) + \frac{5}{2}C_\delta\\
		& \leq \beta_{\wb m} + 3 C_\delta\\
		& \leq \beta_{\wb m-1},
	\end{align*}	
	where the first line follows from \cref{lm:expected_sq_loss_pseudo}; 
	the third line follows from statement 1 and \cref{lm:expected_sq_loss_pseudo}; the fourth line follows from the fact that $\wh f_{\wb m}$ is the minimizer with respect to $\wh R_{\wb m}$ and \cref{lm:expected_sq_loss_pseudo}; the last line follows from the construction of $\beta_{\wb m}$.
\end{enumerate}
\end{proof}

We introduce more notations.
    Denote $\prn{\cX, \Sigma, \cD_\cX}$ as the (marginal) probability space,
    and denote $\wb \cX_m \ldef \crl{x \in \cX: g_m(x) = 1} \in \Sigma$ be the region where query \emph{is} requested within epoch $m$.
    Under the prerequisites of \cref{lm:conf_width_dis_coeff_mis} and \cref{lm:per_round_regret_dis_coeff_mis} (i.e., \cref{eq:inductive1} holds true for epochs $m = 1,2,\ldots,\wb m$), we have $\cF_{m} \subseteq \cF_{m - 1}$ for  $m = 1, 2, \ldots, \wb m$, which leads to
    $\wb \cX_{m} \subseteq \wb \cX_{m-1}$ for $m = 1, 2, \ldots, \wb m$.
    We now define a sub probability measure $\wb \mu_{m} \ldef ({\cD_\cX})_{\mid \wb \cX_{m}}$ such that $\wb \mu_{m}(\omega) = \cD_{\cX}\prn{ \omega \cap \wb \cX_{m}}$ for any $\omega \in \Sigma$. 
    Fix any epoch $m \leq \wb m$ and
    consider any measurable function $F$ (that is $\cD_\cX$ integrable), we have 
    \begin{align}
    	\E_{x \sim \cD_\cX} \brk*{ \ind(g_{\wb m}(x) = 1) \cdot F(x)}
	& = \int_{x \in \wb \cX_{\wb m}} F(x) \, d \cD_\cX(x) \nonumber \\ 
	& \leq \int_{x \in \wb \cX_{m}} F(x) \, d \cD_\cX(x)\nonumber \\ 
	& = \int_{x \in \cX} F(x) \, d \wb \mu_{m} (x) \nonumber \\ 
	& \rdef \E_{x \sim \wb \mu_{m}} \brk*{ F(x)}, \label{eq:change_of_measure} 
    \end{align}
    where, by a slightly abuse of notations, we use $\E_{x \sim \mu} \sq{\cdot}$ to denote the integration with any sub probability measure $\mu$. 
    In particular, \cref{eq:change_of_measure} holds with equality when $m = \wb m$.

\begin{lemma}
    \label{lm:conf_width_dis_coeff_mis}
 Fix any epoch $\wb m \geq 2$. Suppose \cref{eq:inductive1} holds true for epochs $m = 1,2,\ldots,\wb m$, we then have \cref{eq:inductive2} holds true for epoch $m = \wb m$.
\end{lemma}

\begin{proof}
We prove \cref{eq:inductive2} for epoch $m = \wb m$.
We know that $\ind(g_{\wb m}(x) = 1) = \ind (g_{\wb m}(x)= 1) \cdot \ind(w(x;\cF_{\wb m}) > \gamma/2 )$ from \cref{lm:query_implies_width}. Thus, for any $\check m \leq \wb m$, we have 
    \begin{align}
	    \E_{x \sim \cD_\cX} \sq{\ind(g_{\wb m}(x)= 1)} 
	    &  = \E_{x \sim \cD_\cX} \sq{\ind(g_{\wb m}(x)= 1) \cdot \ind(w(x;\cF_{\wb m})> \gamma/2 )}\nonumber \\ 
	    & \leq \E_{x \sim \wb \mu_{\check m} } \sq{\ind(w(x;\cF_{\wb m})> \gamma/2  )}\nonumber \\
    & \leq \E_{x \sim \wb \mu_{\check m}} \prn[\Big]{ \ind \prn[\big]{\sup_{f \in \cF_{\wb m}} \abs*{f(x) - \wb f(x)} > \gamma/4}} , \label{eq:conf_width_dis_coeff_1}
    \end{align}
where the second line uses \cref{eq:change_of_measure} and the last line follows from the facts that
    $\wb f \in \cF_{\wb m}$ (by \cref{eq:inductive1}) and $w(x;\cF_{\wb m}) > \gamma/2  \implies \exists f \in \cF_{\wb m}, \abs{f(x) - \wb f (x)} > {\gamma}/ {4}$. 

    For any time step $t$, let  $m(t)$ denote the epoch where  $t$ belongs to.
From \cref{eq:inductive1}, we know that, $\forall f \in \cF_{\wb m}$,  
    \begin{align}
9 \beta_{\wb m} &
\geq \sum_{t=1}^{\tau_{\wb m -1}} \E_{t} \sq[\Big]{ Q_t \prn[\big]{f(x_t) - \wb f(x_t)}^2} \nonumber \\
       &  = \sum_{t=1}^{\tau_{\wb m -1}} \E_{x \sim \cD_\cX} \sq[\Big]{\ind(g_{m(t)}(x)=1) \cdot \prn[\big]{f(x) - \wb f(x)}^2} \nonumber \\
	  & = \sum_{\check m=1}^{\wb m-1} n_{\check m} \cdot \E_{x \sim \wb \mu_{\check m}} \brk*{  \prn*{f(x) - \wb f(x)}^2}\nonumber \\
	& = \tau_{\wb m-1} \E_{x \sim \wb \nu_{\wb m}} \brk*{  \prn*{f(x) - \wb f(x)}^2}, \label{eq:conf_width_dis_coeff_2}
    \end{align}
    where we use $Q_t = g_{m(t)}(x_t) = \ind(g_{m(t)}(x) = 1)$ and \cref{eq:change_of_measure} on the second line, and define a new sub probability measure 
    $$\wb \nu_{\wb m} \ldef \frac{1}{\tau_{\wb m-1}} \sum_{\check m =1}^{\wb m-1} n_{\check m} \cdot \wb \mu_{\check m}$$ on the third line.

    Plugging \cref{eq:conf_width_dis_coeff_2} into \cref{eq:conf_width_dis_coeff_1} leads to the bound 
    \begin{align*}
	& \E_{x \sim \cD_\cX} \sq{\ind(g_{\wb m}(x)= 1)} \\
	& \leq \E_{x \sim \wb \nu_{\wb m}} \sq[\bigg]{\ind \prn[\Big]{\exists f \in \cF, \abs[\big]{f(x) - \wb f(x)} > \gamma/4, \E_{x \sim \wb \nu_{\wb m}} \sq[\Big]{\prn[\big]{f(x) - \wb f(x)}^2} \leq \frac{9 \beta_{\wb m}}{\tau_{\wb m-1}}}},
    \end{align*}
    where we use the definition of $\wb \nu_{\wb m}$ again (note that \cref{eq:conf_width_dis_coeff_1} works with any $\check m \leq \wb m$).  
    Based on the \cref{def:dis_coeff_value},\footnote{Note that analyzing with a sub probability measure $\wb \nu$ does not cause any problem. See \citet{zhu2022efficient} for a detailed discussion.} we then have
    \begin{align*}
	& \E_{x \sim \cD_\cX} \sq{\ind(g_{\wb m}(x)= 1)}  \\
	& \leq  \frac{144 \beta_{\wb m}}{{\tau_{\wb m-1}} \, \gamma^2} \cdot \theta^{\val}_{\wb f}\prn*{\cF, \gamma/4, \sqrt{9 \beta_{\wb m}/2\tau_{\wb m-1}}}\\
	& \leq  \frac{144 \beta_{\wb m}}{{\tau_{\wb m-1}} \, \gamma^2} \cdot \theta^{\val}_{\wb f}\prn*{\cF, \gamma/4, \sqrt{ \beta_{\wb m}/\tau_{\wb m-1}}}\\
	& \leq \frac{144 \beta_{\wb m}}{{\tau_{\wb m-1}} \, \gamma^2} \cdot \wb \theta.
    \end{align*}
\end{proof}

\begin{lemma}
    \label{lm:per_round_regret_dis_coeff_mis}
 Fix any epoch $\wb m \geq 2$. Suppose \cref{eq:inductive1} holds true for epochs $m = 1,2,\ldots,\wb m$, we then have \cref{eq:inductive3} holds true for epoch $m = \wb m$.
\end{lemma}

\begin{proof}
	We prove \cref{eq:inductive3} for epoch $m = \wb m$.
    Similar to the proof of \cref{lm:conf_width_dis_coeff_mis}, we have 
    \begin{align*}
	 \E_{x \sim \cD_\cX} \sq{\ind(g_{\wb m}(x)= 1)\cdot w(x;\cF_{\wb m})} 
	& = \E_{x \sim \cD_\cX} \sq{\ind(g_{\wb m}(x)=1) \cdot \ind(w(x;\cF_{\wb m})> \gamma/2 )\cdot w(x;\cF_{\wb m})} \\
	& \leq \E_{x \sim \wb \mu_{\check m}} \sq{\ind(w(x;\cF_{\wb m})> \gamma/2 )\cdot w(x;\cF_{\wb m})}
    \end{align*}
    for any $\check m \leq \wb m$. 
With $\wb \nu_{\wb m} \ldef \frac{1}{\tau_{\wb m-1}} \sum_{\check m =1}^{\wb m-1} n_{\check m} \cdot \wb \mu_{\check m}$, we then have 
\begin{align*}
	 & \E_{x \sim \cD_\cX} \sq{\ind(g_{\wb m}(x)= 1)\cdot w(x;\cF_{\wb m})} \\
	& \leq \E_{x \sim \wb \nu_{\wb m}} \sq{\ind(w(x;\cF_{\wb m})> \gamma/2 )\cdot w(x;\cF_{\wb m})}\\
        & \leq \E_{x \sim \wb \nu_{\wb m}} \sq*{\ind \prn*{\sup_{f \in \cF_{\wb m}} \abs[\big]{f(x) - \wb f(x)} > \gamma/4}\cdot \prn*{\sup_{f , f^\prime \in \cF_{\wb m}} \abs*{f(x) - f^\prime(x)}}} \\
        & \leq 2\E_{x \sim \wb \nu_{\wb m}} \sq*{\ind \prn*{\sup_{f \in \cF_{\wb m}} \abs[\big]{f(x) - \wb f(x)} > \gamma/4}\cdot \prn*{\sup_{f \in \cF_{\wb m}} \abs{f(x) - \wb f(x)}}} \\
        & \leq 2 \int_{\gamma/4}^1 \E_{x \sim \wb \nu_{\wb m}} \sq*{\ind \prn*{\sup_{f \in \cF_{\wb m}} \abs[\big]{f(x) - \wb f(x)} \geq \omega}} \, d \, \omega \\
& \leq  2 \int_{\gamma/4}^1  \frac{1}{\omega^2} \, d \, \omega \cdot \prn*{ \frac{9 \beta_{\wb m}}{\tau_{\wb m-1}} \cdot \theta^{\val}_{\wb f}\prn*{\cF, \gamma/4, \sqrt{9 \beta_{\wb m}/2\tau_{\wb m-1}}}}\\
& \leq { \frac{72 \beta_{\wb m}}{\tau_{\wb m-1} \, \gamma} \cdot \theta^{\val}_{\wb f}\prn*{\cF, \gamma/4, \sqrt{\beta_{\wb m}/\tau_{\wb m-1}}}}\\
& \leq  \frac{72 \beta_{\wb m}}{\tau_{\wb m-1} \, \gamma} \cdot \wb \theta,
\end{align*}
where we follow similar steps as in the proof of \cref{lm:conf_width_dis_coeff_mis} and use some basic arithmetic facts.
\end{proof}

\begin{lemma}
	\label{lm:induction}
	\cref{eq:inductive1}, \cref{eq:inductive2} and \cref{eq:inductive3} hold true for all $m \in [M]$.
\end{lemma}
\begin{proof}
	We first notice that, by \cref{lm:set_f_mis}, \cref{eq:inductive1} holds true for epochs $\wb m=1,2$ unconditionally.
	We also know that, by \cref{lm:conf_width_dis_coeff_mis} and \cref{lm:per_round_regret_dis_coeff_mis}, once \cref{eq:inductive1} holds true for epochs $m = 1,2, \ldots, \wb m$, \cref{eq:inductive2} and \cref{eq:inductive3} hold true for epochs $m = \wb m$ as well; 
	at the same time, by \cref{lm:set_f_mis}, once \cref{eq:inductive2} holds true for epochs $m = 2, 3, \ldots, \wb m$, \cref{eq:inductive1} will hold true for epoch $m = \wb m+1$.

	We thus can start the induction procedure from  $\wb m = 2$, and make sure that 
	\cref{eq:inductive1}, \cref{eq:inductive2} and \cref{eq:inductive3} hold true for all $m \in [M]$.
\end{proof}

\subsection{Proof of \cref{thm:abs_gen}}
\label{app:abs_gen_proof}

\thmAbsGen*

\begin{proof}
	We analyze under the good event $\cE$ defined in \cref{lm:expected_sq_loss_pseudo}, which holds with probability at least $1-\frac{\delta}{2}$. Note that all supporting lemmas stated in \cref{app:supporting_abs} hold true under this event.

Fix any $m \in [M]$.
We analyze the Chow's excess error of $\wh h_m$, which is measurable with respect to $\mfF_{\tau_{m-1}}$. 
For any $x \in \cX$, if $g_m(x) = 0$, 
\cref{lm:regret_no_query_mis} implies that $\exc_\gamma (\wh h_m ;x) \leq 0 $. 
If $g_m(x)= 1$, we know that $\wh h_m(x) \neq \bot$ and $\frac{1}{2} \in (\lcb(x;\cF_m) - \frac{\gamma}{4},\ucb(x;\cF_m) + \frac{\gamma}{4})$. 
Since $\wb f \in \cF_m$ by \cref{lm:induction} (with \cref{eq:inductive1}) and $\sup_{x \in \cX} \abs{ \wb f(x) - f^{\star}(x)} \leq \kappa \leq \gamma /4$ by construction. 
The error incurred in this case is upper bounded by 
\begin{align*}
	\exc(\wh h_m; x) 
	& \leq 2 \abs{ f^{\star}(x)- 1 /2}\\
	& \leq 2\kappa + 2 \abs{ \wb f(x)- 1 /2}\\
	& \leq 2\kappa + 2 w(x;\cF_m) + \frac{\gamma}{2} \\
	& \leq 4 w(x;\cF_m),
\end{align*}
where we use \cref{lm:query_implies_width} in the last line.

Combining these two cases together, we have 
\begin{align*}
	\exc( \wh h_m) \leq 4 \, \E_{x \sim \cD_\cX} \brk{ \ind(g_m(x) = 1) \cdot w(x;\cF_m)}.	
\end{align*}
Take $m=M$ and apply \cref{lm:induction} (with \cref{eq:inductive3}) leads to the following guarantee.
\begin{align*}
	\exc( \wh h_M)
	& \leq   { \frac{ 576 \beta_M}{\tau_{M-1} \gamma} \cdot \theta^{\val}_{\wb  f}\prn*{\cF, \gamma/4, \sqrt{\beta_M/ \tau_{M-1}}}}\\
	& \leq  O \prn*{  \frac{ \pseud(\cF) \log ( T / \delta)}{T \, \gamma} \cdot \wb \theta} \\
	& = O \prn*{  \eps \cdot \log \prn*{\frac{\wb \theta \cdot \pseud(\cF)}{\eps \, \gamma \, \delta}}}, 
\end{align*}
where we 
use the fact that $T = \frac{\wb \theta \cdot \pseud(\cF)}{\eps \, \gamma}$.

We now analyze the label complexity (note that the sampling process of \cref{alg:abs} stops at time $t = \tau_{M-1}$).
Note that $\E \brk{\ind(Q_t = 1) \mid \mfF_{t-1}} = \E_{x\sim\cD_\cX} \brk{ \ind(g_m(x) = 1) }$ for any epoch $m \geq 2$ and time step $t$ within epoch $m$. 
  Combining \cref{lm:martingale_two_sides} and \cref{lm:induction} (with \cref{eq:inductive2}) leads to
    \begin{align*}
        \sum_{t=1}^{\tau_{M-1}} \ind(Q_t = 1) & \leq \frac{3}{2} \sum_{t=1}^{\tau_{M-1}} \E \sq{\ind(Q_t = 1) \mid \mfF_{t-1}} + 4 \log (2 /\delta)\\
        & \leq 3 + \frac{3}{2}\sum_{m=2}^{M-1}\frac{(\tau_m - \tau_{m-1}) \cdot 144 \beta_m}{{\tau_{m-1}} \, \gamma^2} \cdot \wb \theta + 4 \log (2 /\delta) \\
	& \leq 3 + 4 \log (2 /\delta) + O\prn*{  \frac{ M^2 \cdot \pseud(\cF) \cdot \log(T /\delta) \cdot \wb \theta }{\gamma^2}}
	\\ 
	& =O\prn*{  \frac{ M^2 \cdot \pseud(\cF) \cdot \log(T /\delta) \cdot \wb \theta }{\gamma^2}} 
    \end{align*}
    with probability at least $1-\delta$ (due to another application of \cref{lm:martingale_two_sides} with confidence level $\delta /2$), 
  where we use the fact that $\beta_m \ldef 3 (M - m + 1) C_\delta$ and $C_\delta \ldef O\prn{\pseud(\cF) \cdot \log(T /\delta)}$.
\end{proof}

\section{Proof of \cref{thm:abs}}
\label{app:abs}

We provide prerequisites in \cref{app:abs_prereq} and the preprocessing procedures in \cref{app:filtering}. We give the proof of \cref{thm:abs} in \cref{app:abs_proof}.

\subsection{Prerequisites}
\label{app:abs_prereq}

\subsubsection{Upper bounds on pseudo dimension}

We present a result regarding the approximation and an upper bound on the pseudo dimension (i.e., \cref{def:pseudo_d}). 
\begin{proposition}
	\label{prop:pd_approx}
	Suppose $\cD_{\cX \cY} \in \cP(\alpha, \beta)$. 
	One can construct a set of neural network regression functions  $\cF_\dnn$ such that the following two properties hold simultaneously:
	 \begin{align*}
		\exists f \in \cF_\dnn \text{ s.t. } \nrm{ f - f^{\star}}_\infty \leq \kappa, \quad \text{ and } \quad \pseud(\cF_\dnn) \leq c \cdot {\kappa^{-\frac{d}{\alpha}} \log^2 (\kappa^{-1})},
	\end{align*}
	where $c > 0$ is a universal constant.
\end{proposition}

\begin{proof}
	The result follows by combining \cref{thm:approx_sobolev} and \cref{thm:pdim_nn}.
\end{proof}

\subsubsection{Upper bounds on value function disagreement coefficient}
We derive upper bounds on the value function disagreement coefficient (i.e., \cref{def:dis_coeff_value}). We first introduce the (value function) eluder dimension, a complexity measure that is closely related to the value function disagreement coefficient \citet{russo2013eluder, foster2020instance}.

\begin{definition}[Value function eluder dimension]
\label{def:eluder}
For any $f^{\star} \in \cF$ and $\gamma_0 > 0$, let $\check{\mathfrak{e}}_{f^{\star}}(\cF, \gamma)$ be the length of the longest sequence of data points $x^{1}, \dots, x^{m}$ such that for all $i$, there exists $f^{i} \in \cF$ such that 
\begin{align*}
    \abs{ f^{i}(x^{i}) - f^{\star}(x^{i}) } > \gamma, \quad \text{ and } \quad \sum_{j < i} \paren{ f^{i}(x^{j})  - f^{\star}(x^{j})}^2 \leq \gamma^2.
\end{align*}
The value function eluder dimension is defined as $\mathfrak{e}_{f^{\star}}(\cF, \gamma_0) \coloneqq \sup_{\gamma > \gamma_0} \check{\mathfrak{e}}_{f^{\star}}(\cF, \gamma)$. 
\end{definition}

The next result shows that the value function disagreement coefficient can be upper bounded by eluder dimension.
\begin{proposition}[\citet{foster2020instance}]
    \label{prop:eluder_star_dis}
    Suppose $\cF$ is a uniform Glivenko-Cantelli class.
    For any $f^{\star}: \cX \rightarrow [0,1]$ and $\gamma, \eps > 0$,
    we have $\theta^{\val}_{f^\star}(\cF, \gamma, \epsilon) \leq 4 \, {\mfe_{f^\star}(\cF, \gamma)}$.
\end{proposition}

We remark here that the requirement that $\cF$ is a uniform Glivenko-Cantelli class is rather weak: It is satisfied as long as  $\cF$ has finite pseudo dimension \citep{anthony2002uniform}.

In the following, we only need to derive upper bounds on the value function eluder dimension, which upper bounds on the value function disagreement coefficient.\footnote{We focus on Euclidean geometry on $\cX$ (i.e., using $\nrm{\cdot}_2$ norm) in deriving the upper bound. Slightly tighter bounds might be possible with other norms.}
We first define two definitions: (i) the standard definition of covering number (e.g., see \citet{wainwright2019high}), and (ii) a newly-proposed definition of approximate Lipschitzness.

\begin{definition}
	\label{def:covering}
An $\iota$-covering of a set  $\cX$ with respect to a metric  $\rho$ is a set  $\crl{x_1, \ldots,x_N} \subseteq \cX$ such that for each $x \in \cX$, there exists some $i \in [N]$ such that $\rho(x, x_i) \leq \iota$. The $\iota$-covering number  $\cN(\iota; \cX, \rho)$ is the cardinality of the smallest  $\iota$-cover.
\end{definition}

\begin{definition}
	\label{def:lip_approx}
	We call a function $f: \cX \rightarrow \R$ $(L,\kappa)$-approximate Lipschitz if 
	\begin{align*}
		\abs{ f(x) - f(x^{\prime}) } \leq L \cdot \nrm{x - x^{\prime}}_2 + \kappa
	\end{align*}
	for any $x, x^{\prime} \in \cX$.
\end{definition}

We next provide upper bounds on value function eluder dimension and value function disagreement coefficient.

\begin{theorem}
\label{thm:eluder_lip_approx}
Suppose $\cF$ is a set of $(L,\kappa /4)$-approximate Lipschitz functions. 
For any $\kappa^{\prime} \geq \kappa$,
we have $\sup_{f \in \cF} \mfe_f(\cF, \kappa^{\prime}) \leq 17 \cdot \cN(\frac{\kappa^{\prime}}{8L}; \cX, \nrm{\cdot}_2)$.
\end{theorem}
\begin{proof}
Fix any $f \in \cF$ and $\wb \kappa \geq \kappa^{\prime}$. We first give upper bounds on $\check \mfe_f(\cF, \wb \kappa)$.

We construct $\cG \ldef \cF - f$, which is a set of $(2L,\kappa /2)$-Lipschitz functions. 
Fix any eluder sequence $x^1, \dots, x^m$ at scale $\wb \kappa$ and any $\check x \in \cX$. We claim that $\abs{ \crl{x_j}_{j\leq m} \cap \cS  } \leq 17$ where $\cS \ldef \crl{ x \in \cX: \nrm{x-\check x}_2 \leq \frac{\wb \kappa}{8L}}$.
Suppose $\crl{x_j}_{j\leq m} \cap \cS = x_{j_1}, \dots, x_{j_k}$ ($j_i$ is ordered based on the ordering of $\crl{x_j}_{j\leq m}$). 
Since $x^{j_k}$ is added into the eluder sequence, there must exists a $g^{j_k} \in \cG$ such that  
\begin{align}
\label{eq:eluder_lip}
    \abs{g^{j_k}(x^{j_k})} > \wb \kappa, \quad \text{ and } \quad 
    \sum_{j<j_k} \prn{g^{j_k}(x^j)}^2 \leq \wb \kappa^2.
\end{align}
Since $g^{j_k}$ is $(2L,\kappa /2)$-Lipschitz, $\wb \kappa \geq \kappa^{\prime}\geq \kappa$ and $x^{j_k} \in \cS$, we must have $g^{j_k}(x) \geq \frac{\wb \kappa}{4}$ for any $x \in \cS$. As a result, we must have $\abs{ \crl{x_j}_{j<j_k} \cap \cS^i  } \leq 16$
as otherwise the second constraint in \cref{eq:eluder_lip} will be violated.
We cover the space $\cX$ with $\cN(\frac{\wb \kappa}{8L}; \cX, \nrm{\cdot}_2)$ balls of radius $\frac{\wb \kappa}{8L}$. Since the eluder sequence contains at most $17$ data points within each ball, 
we know that $\check \mfe_f(\cF, \wb \kappa) \leq 17 \cdot\cN(\frac{\wb \kappa}{8L}; \cX, \nrm{\cdot}_2) $.

The desired result follows by noticing that $17 \cdot\cN(\frac{\wb \kappa}{8L}; \cX, \nrm{\cdot}_2)$ is non-increasing in $\wb \kappa$.
\end{proof}

\begin{corollary}
	\label{cor:eluder_lip_approx}
	Suppose $\cX \subseteq \B^{d}_r \ldef \crl{x \in \R^{d}: \nrm{x}_2 \leq r}$
and $\cF$ is a set of $(L,\kappa /4)$-approximate Lipschitz functions. 
For any $\kappa^{\prime} \geq \kappa$,
there exists a universal constant $c > 0$, such that
$\theta^{\val}_\cF (\kappa^{\prime}) \ldef \sup_{f \in \cF, \iota > 0} \theta_f^{\val}(\cF, \kappa^{\prime}, \iota) \leq c \cdot \prn{\frac{Lr}{\kappa^{\prime}}}^{d}$.
\end{corollary}
\begin{proof}
It is well-known that $\cN(\iota; \B^{d}_r, \nrm{\cdot}_2) \leq \prn*{1 + 2r / \iota}^{d}$ \citep{wainwright2019high}. The desired result thus follows from combining \cref{thm:eluder_lip_approx} with \cref{prop:eluder_star_dis}.	
\end{proof}

\subsection{The preprocessing step: Clipping and filtering}
\label{app:filtering}

Let $\eta: \cX \rightarrow [0,1]$ denote the true conditional probability and 
$\cF_\dnn$ denote a set of neural network regression functions (e.g., constructed based on \cref{thm:approx_sobolev}).
We assume that (i) $\eta$ is  $L$-Lipschitz, and (ii) there exists a $f \in \cF$ such that $\nrm{f - \eta}_\infty \leq \kappa$ for some approximation factor $\kappa > 0$.
We present the preprocessing step below in \cref{alg:preprocess}.

\begin{algorithm}[H]
	\caption{The Preprocessing Step: Clipping and Filtering}
	\label{alg:preprocess} 
	\renewcommand{\algorithmicrequire}{\textbf{Input:}}
	\renewcommand{\algorithmicensure}{\textbf{Output:}}
	\newcommand{\algorithmicbreak}{\textbf{break}}
    \newcommand{\BREAK}{\STATE \algorithmicbreak}
	\begin{algorithmic}[1]
		\REQUIRE A set of regression functions $\cF$, Lipschitz parameter $L > 0$, approximation factor  $\kappa > 0$.
		\STATE \textbf{Clipping.}
		Set $\check \cF \ldef \crl{\check f: f \in \cF} $, where, for any $f \in \cF$, we denote
\begin{align*}
	\check f(x) \ldef 
	\begin{cases}
		1, & \text{ if } f(x) \geq 1;\\
		0, & \text{ if } f(x) \leq 0;\\
		f(x), & \text{ o.w. }
	\end{cases}
\end{align*}
		\STATE \textbf{Filtering.} Set $\wt \cF \ldef \crl{\check f \in \check \cF: \check f \text{ is $(L,2\kappa)$-approximate Lipschitz}}$
		\STATE \textbf{Return} $\wt \cF$.
	\end{algorithmic}
\end{algorithm}

\begin{proposition}
	\label{prop:preprocess}
Suppose $\eta$ is  $L$-Lipschitz and  $\cF_\dnn$ is a set of neural networks (of the same architecture) with $W$ parameters arranged in  $L$ layers such that there exists a  $f \in \cF_\dnn$ with  $\nrm{f - \eta}_\infty \leq \kappa$. Let  $\wt \cF_\dnn$ be the set of functions obtained by applying \cref{alg:preprocess} on  $\cF_\dnn$, we then have 
(i)  $\pseud(\wt \cF_\dnn) = O \prn{WL \log (W)}$, and (ii) there exists a $\wt f \in \wt \cF_\dnn$ such that $\nrm{\wt f - \eta}_\infty \leq \kappa$.
\end{proposition}
\begin{proof}
Suppose $f$ is a neural network function, we first notice that the ``clipping'' step can be implemented by adding one additional layer with $O(1)$ additional parameters for each neural network function. More specifically, fix any $f: \cX \rightarrow \R$, we can set $\check f(x) \ldef \relu(f(x)) - \relu(f(x) - 1)$. 
Set $\check \cF_\dnn \ldef \crl{\check f: f \in \cF_\dnn} $, we then have $\pseud(\check \cF_\dnn) = O (WL \log (W)) $ based on \cref{thm:pdim_nn}. Let $\wt \cF_\dnn$ be the filtered version of  $\check \cF_\dnn$.
Since $\wt \cF_\dnn \subseteq \check \cF_\dnn$, we have $\pseud(\wt \cF_\dnn) = O \prn{WL \log (W)}$.

Since $\eta: \cX \rightarrow [0,1]$, we have 
	$\nrm{\check f - \eta }_\infty \leq \nrm{f - \eta }_\infty$, which implies that there must exists a $\check f \in \check \cF_\dnn$ such  $\nrm{\check f - \eta}_\infty \leq \kappa$. To prove the second statement, it suffices to show that the $\check f \in \check \cF$ that achieves  $\kappa$ approximation error is not removed in the ``filtering'' step, i.e., $\check f$ is  $(L, 2\kappa)$-approximate Lipschitz.
For any $x, x^\prime \in \cX$, we have 
 \begin{align*}
	 \abs{ \check f(x) - \check f(x^{\prime})}
	 & = \abs{ \check f(x) - \eta(x) + \eta(x) - \eta(x^{\prime}) + \eta(x^{\prime}) - \check f(x^{\prime})}\\
	 & \leq L \nrm{x - x^{\prime}}_2 + 2 \kappa,
\end{align*}
where we use the $L$-Lipschitzness of $\eta$ and the fact that $\nrm{ \check f - \eta}_{\infty} \leq \kappa$.
\end{proof}
\begin{proposition}
	\label{prop:clip_filter}
	Suppose $\eta$ is  $L$-Lipschitz and  $\cX \subseteq \B^{d}_r$.
	Fix any $\kappa \in (0, \gamma / 32]$. There exists a set of neural network regression functions $\cF_\dnn$ such that the followings hold simultaneously.
	\begin{enumerate}
		\item 	$\pseud(\cF_\dnn) \leq c \cdot {\kappa^{-\frac{d}{\alpha}} \log^2(\kappa^{-1})}$ with a universal constant $c > 0$.
		\item There exists a $\wb f \in \cF_\dnn$ such that $\nrm{\wb f - \eta}_\infty \leq \kappa$.
		\item $ \theta^{\val}_{\cF_\dnn}(\gamma / 4) \ldef \sup_{f \in \cF_\dnn, \iota > 0} \theta^{\val}_f (\cF_\dnn, \gamma /4, \iota) \leq c^{\prime} \cdot \prn{\frac{Lr}{\gamma}}^{d}$ with a universal constant $c^{\prime} > 0$.
	\end{enumerate}
\end{proposition}
\begin{proof}
Let $\cF_\dnn$ be obtained by (i) invoking \cref{thm:approx_sobolev} with approximation level $\kappa $, and (ii) invoking \cref{alg:preprocess} on the set of functions obtained in step (i).
The first two statements follow from \cref{prop:preprocess}, and the third statement follows from \cref{cor:eluder_lip_approx} (note that to achieve guarantees for disagreement coefficient at level $\gamma / 4$, we need to have  $\kappa \leq \gamma / 32$ when invoking \cref{thm:approx_sobolev}). 
\end{proof}

\subsection{Proof of \cref{thm:abs}}
\label{app:abs_proof}

\thmAbs*

\begin{proof}
Let line 1 of \cref{alg:NCALP} be the set of neural networks $\cF_\dnn$ generated from \cref{prop:clip_filter} with approximation level $\kappa \in (0, \gamma / 32]$ (and constants $c, c^{\prime}$ specified therein).
To apply results derived in \cref{thm:abs_gen}, we need to satisfying \cref{eq:kappa_requirement}, i.e., specifying an approximation level $\kappa \in (0, \gamma / 32]$ such that the following holds true
 \begin{align*}
	\frac{1}{\kappa^2} \geq \frac{4320 \cdot c^{\prime} \cdot \prn{\frac{Lr}{\gamma}}^{d} \cdot \prn*{\ceil*{\log_2 \prn*{\frac{c^{\prime} \cdot \prn{\frac{Lr}{\gamma}}^{d} \cdot c \cdot \prn{\kappa^{-\frac{d}{\alpha}} \log^2 (\kappa^{-1})}}{\eps \, \gamma}}}}^{2} }{\gamma^2}
\end{align*}

For the setting we considered, i.e., $\cX = [0,1]^{d}$ and $\eta \in \cW_{1}^{\alpha, \infty}(\cX)$, we have $r = \sqrt{d} = O(1)$ and $L \leq \sqrt{d} = O(1)$ (e.g., see Theorem 4.1 in \citet{heinonen2005lectures}).\footnote{Recall that we ignore constants that can be potentially $\alpha$-dependent and $d$-dependent.}
We thus only need to select a $\kappa \in (0, \gamma /32]$ such that 
\begin{align*}
	\frac{1}{\kappa} \geq \wb c \cdot \prn*{\frac{1}{\gamma}}^{\frac{d}{2}+1} \cdot \prn*{\log \frac{1}{\eps \, \gamma} + \log \frac{1}{\kappa}},
\end{align*}
with a universal constant $\wb c > 0$ (that is possibly  $d$-dependent and $\alpha$-dependent).
Since $x \geq 2a \log a \implies x \geq a \log x$ for any  $a > 0$, we can select a  $\kappa > 0$ such that 
 \begin{align*}
	\frac{1}{\kappa} = {\check c \cdot \prn*{\frac{1}{\gamma}}^{\frac{d}{2} + 1} \cdot \log \frac{1}{\eps \, \gamma}}
\end{align*}
with a universal constant $\check c > 0$.
With such choice of $\kappa$, from \cref{prop:clip_filter}, we have  
\begin{align*}
\pseud(\cF_\dnn) = O\prn*{\prn*{\frac{1}{\gamma}}^{\frac{d^2+d}{2\alpha}} \cdot \polylog \prn*{\frac{1}{\eps \, \gamma}} }.
\end{align*}
Plugging this bound on $\pseud(\cF_\dnn)$ and the upper bound on $\theta^{\val}_{\cF_\dnn}(\gamma / 4)$ from \cref{prop:clip_filter} into the guarantee of \cref{thm:abs_gen} leads to $\exc_\gamma(\wh h) = O \prn{\eps \cdot \log \prn{\frac{1}{\eps \, \gamma \, \delta}}}$ after querying 
\begin{align*}
	O \prn*{\prn*{\frac{1}{\gamma}}^{d + 2 + \frac{d^2+d}{2\alpha}} \cdot \polylog \prn*{\frac{1}{\eps \, \gamma \, \delta}} }
\end{align*}
labels.
\end{proof}

\section{Other omitted details for \cref{sec:abstention}}
We discuss the proper abstention property of classifier learned in \cref{alg:NCALP} and its exponential speedups under standard excess error and Massart noise in \cref{app:proper}. We discuss the computational efficiency of \cref{alg:NCALP} in \cref{app:computational}. We provide the proof of \cref{thm:single_relu_lb} in \cref{app:single_relu_lb}.

\subsection{Proper abstention and exponential speedups under Massart noise}
\label{app:proper}

We first recall the definition of \emph{proper abstention} introduced in \citet{zhu2022efficient}.
\begin{definition}[Proper abstention]
\label{def:proper_abstention}
A classifier $\widehat h : \cX \rightarrow \cY \cup \curly*{\bot}$ enjoys proper abstention if and only if it abstains in regions where abstention is indeed the optimal choice, i.e., 
$\crl[\big]{x \in \cX: \widehat h(x) = \bot} \subseteq \crl*{x \in \cX: \eta(x) \in \brk*{\frac{1}{2} - \gamma , \frac{1}{2} + \gamma  } } \rdef \cX_\gamma$.
\end{definition}

We next show that the classifier $\wh h$ returned by \cref{alg:abs} enjoys the proper abstention property. 
We also convert the abstaining classifier $\widehat h: \cX \rightarrow \cY \cup \curly*{\bot}$ into a standard classifier $\check h: \cX \rightarrow \cY$ and quantify its standard excess error.
The conversion is through randomizing the prediction of $\wh h$ over its abstention region, i.e., if $\wh h(x) = \bot$, then its randomized version $\check h(x)$ predicts $0$ and $1$ with equal probability \citep{puchkin2021exponential}.

\begin{proposition}
\label{prop:proper_abstention}
The classifier $\wh h$ returned by \cref{alg:abs} enjoys proper abstention. With randomization over the abstention region, we have the following upper bound on its standard excess error
\begin{align}
	\label{eq:prop_abstention}
    \err(\check h) - \err(h^\star)  
     = \err_{\gamma}(\widehat h) - \err(h^\star) + \gamma \cdot \P_{x \sim \cD_{\cX}} (x \in \cX_{\gamma}).
\end{align}
\end{proposition}

\begin{proof}
The proper abstention property of $\wh h$ returned by \cref{alg:abs} is achieved via conservation: $\wh h$ will avoid abstention unless it is absolutely sure that abstention is the optimal choice (also see the proof of \cref{lm:regret_no_query_mis}.

Let $\check h: \cX \rightarrow \cY$ be the randomized version of $\wb h: \cX \rightarrow \crl{0, 1, \bot}$ (over the abstention region  $\crl{x \in \cX: \wh h(x) = \bot} \subseteq \cX_\gamma$).
We can see that, compared to the Chow's abstention error $1 /2 - \gamma$, the additional error incurred over the abstention region is exactly $\gamma \cdot \P_{x \sim \cD_{\cX}} (x \in \cX_{\gamma})$. 
We thus have
\begin{align*}
    \err(\widehat h) - \err(h^\star)  
     \leq \err_{\gamma}(\widehat h) - \err(h^\star) + \gamma \cdot \P_{x \sim \cD_{\cX}} (x \in \cX_{\gamma}).
\end{align*}
\end{proof}

To characterize the standard excess error of classifier with proper abstention, we only need to upper bound the term $ \P_{x \sim \cD_{\cX}} (x \in \cX_{\gamma})$, which does \emph{not} depends on the (random) classifier $\wh h$. Instead, it only depends on the marginal distribution. 

We next introduce the Massart \citep{massart2006risk}, which can be viewed as the extreme version of the Tsybakov noise by sending $\beta \rightarrow \infty$. 
\begin{definition}[Massart noise]
	\label{def:massart}
  A distribution $\cD_{\cX \cY}$ satisfies the Massart noise condition with parameter $\tau_0> 0$ if
  $\P_{x \sim \cD_\cX} \paren*{\abs*{\eta(x) - 1 / 2} \leq \tau_0} = 0$.
\end{definition}
\begin{restatable}{proposition}{propProperAbstention}
\label{prop:proper_abstention_2}
Suppose Massart noise holds.
By setting the abstention parameter $\gamma = \tau_0$ in \cref{alg:abs} (and randomization over the abstention region), with probability at least $1-\delta$, we obtain a classifier with standard excess error $\wt O(\eps)$ after querying $\poly(\frac{1}{\tau_0}) \cdot \polylog(\frac{1}{\eps \, \delta})$ labels.
\end{restatable}
\begin{proof}
	This is a direct consequence of \cref{thm:abs} and \cref{prop:proper_abstention}.
\end{proof}

\subsection{Computational efficiency}
\label{app:computational}
    We discuss the efficient implementation of \cref{alg:abs} and its computational complexity in the section. 
The computational efficiency of \cref{alg:abs} mainly follows from the analysis in \citet{zhu2022efficient}. We provide the discussion here for completeness.

\paragraph{Regression orcale}
We introduce the regression oracle over the set of initialized neural networks $\cF_\dnn$ (line 1 at \cref{alg:NCALP}).
Given any set $\cS$ of weighted examples $(w, x, y) \in \R_+ \times \cX \times \cY$ as input, the regression oracle outputs 
\begin{align*}
    \widehat f_\dnn \ldef \argmin_{f \in \cF_\dnn} \sum_{(w, x, y) \in \cS} w \paren*{f(x) - y}^2.
\end{align*}
While the exact computational complexity of such oracle with a set of neural networks remains elusive, in practice, running stochastic gradient descent often leads to great approximations.
We quantify the computational complexity in terms of the number of calls to the regression oracle. Any future analysis on such oracle can be incorporated into our guarantees.

We first state some known results in computing the confidence intervals with respect to a general set of regression functions $\cF$.

\begin{proposition}[\citet{krishnamurthy2017active, foster2018practical, foster2020instance}] \label{prop:CI_oracle}
Consider the setting studied in \cref{alg:abs}. 
Fix any epoch $m \in [M]$ and denote $\cB_m \ldef \crl{ (x_t,Q_t, y_t)}_{t=1}^{\tau_{m-1}}$.
Fix any $\iota > 0$.
For any data point $x \in \cX$, there exists algorithms $\AlgLcb$ and $\AlgUcb$ that certify
\begin{align*}
    & \lcb(x;\cF_m) - \iota \leq \AlgLcb(x;\cB_m,\beta_m,\iota) \leq \lcb(x;\cF_m) \quad \text{and}\\
    &\ucb(x;\cF_m) \leq \AlgUcb(x;\cB_m,\beta_m,\iota) \leq \ucb(x;\cF_m) + \iota.
\end{align*}
The algorithms take 
 $O(\frac{1}{\iota^2} \log \frac{1}{\iota})$ calls of the regression oracle for general $\cF$
 and take $O(\log \frac{1}{\iota})$ calls of the regression oracle if $\cF$ is convex and closed under pointwise convergence.
\end{proposition}
\begin{proof}
See Algorithm 2 in \citet{krishnamurthy2017active} for the general case; and Algorithm 3 in \citet{foster2018practical} for the case when $\cF$ is convex and closed under pointwise convergence.
\end{proof}

We now state the computational guarantee of \cref{alg:abs}, given the regression oracle introduced above.
\begin{restatable}{theorem}{thmAbsEfficient}
	\label{thm:abs_efficient}
	\cref{alg:abs} can be efficiently implemented via the regression oracle and enjoys the same theoretical guarantees stated in \cref{thm:abs}.
	The number of oracle calls needed is $\poly(\frac{1}{\gamma}) \cdot \frac{1}{\eps}$; the per-example inference time of the learned $\wh h_{M}$ is $\wt O ( \frac{1}{\gamma^2} \cdot \polylog \prn{\frac{1}{\eps \, \gamma}})$ for general $\cF$, and $\wt O ( \polylog \prn{\frac{1}{\eps \, \gamma}}) $ when $\cF$ is convex.
\end{restatable}

\begin{proof}
Fix any epoch $m \in [M]$.
Denote $\wb \iota \ldef \frac{\gamma}{8M}$ and $\iota_m \ldef \frac{(M-m) \gamma}{8M}$.
With any observed $x \in \cX$, we construct the approximated confidence intervals $\wh \lcb(x;\cF_m)$ and 
$\wh \ucb(x; \cF_m)$ as follows.
\begin{align*}
&	\wh \lcb(x;\cF_m) \ldef \AlgLcb(x;\cB_m,\beta_m,\wb \iota) - \iota_m \quad \text{and}	
    & \wh \ucb(x;\cF_m) \ldef\AlgUcb(x;\cB_m,\beta_m,\wb \iota)+ \iota_m. 
\end{align*}
For efficient implementation of \cref{alg:abs}, we replace $\lcb(x;\cF_m)$ and $\ucb(x;\cF_m)$ with $\wh \lcb(x;\cF_m)$ and $\wh \ucb(x;\cF_m)$ in the construction of $\wh h_m$ and $g_m$.

Based on \cref{prop:CI_oracle}, we know that 
\begin{align*}
    & \lcb(x;\cF_m) - \iota_m - \wb \iota \leq 	\wh \lcb(x;\cF_m) \leq \lcb(x;\cF_m) - \iota_m \quad \text{and}\\
    &\ucb(x;\cF_m) + \iota_m \leq \wh \ucb(x;\cF_m) \leq \ucb(x;\cF_m) + \iota_m + \wb \iota .
\end{align*}
Since $\iota_m + \wb \iota  \leq \frac{\gamma}{8}$ for any $m \in [M]$, the guarantee stated in \cref{lm:query_implies_width} can be modified as $g_m(x)= 1 \implies w(x;\cF_m)\geq \frac{\gamma}{4}$. 
The guarantee stated in \cref{lm:regret_no_query_mis} also holds true since we have $\wh \lcb(x;\cF_m) \leq \lcb(x;\cF_m)$ and $\wh \ucb(x;\cF_m) \geq \ucb(x;\cF_m)$ by construction. 
Suppose $\cF_{m} \subseteq \cF_{m-1}$ (as in \cref{lm:set_f_mis}), we have 
\begin{align*}
	& \wh \lcb(x;\cF_m) \geq \lcb(x;\cF_m) -  \iota_m - \wb \iota \geq \lcb(x;\cF_{m-1}) - \iota_{m-1} \geq \wh\lcb (x;\cF_{m-1}) \quad \text{and} \\
	& \wh \ucb(x;\cF_m) \leq \ucb(x;\cF_m) +  \iota_m + \wb \iota \leq \ucb(x;\cF_{m-1}) + \iota_{m-1} \leq \wh\ucb (x;\cF_{m-1}),
\end{align*}
which ensures that $\ind(g_m(x) = 1) \leq \ind(g_{m-1}(x)=1)$. 
Thus, the inductive lemmas appearing in \cref{app:supporting_abs} can be proved similarly with changes only in constant terms (also change the constant terms in the definition of $\wb \theta$ and in \cref{eq:kappa_requirement}, since $\frac{\gamma}{2}$ is replaced by $\frac{\gamma}{4}$ in \cref{lm:query_implies_width}).
As a result, the guarantees stated in \cref{thm:abs_gen} (and \cref{thm:abs}) hold true with changes only in constant terms.

We now discuss the computational complexity of the efficient implementation. 
At the beginning of each epoch $m$. We use one oracle call to compute $\widehat f_m \ldef \argmin_{f \in \cF} \sum_{t =1}^{ \tau_{m-1}} Q_t \paren{f(x_t) - y_t}^2 $. 
The main computational cost comes from computing $\wh \lcb$ and $\wh \ucb$ at each time step.
We take $\iota = \wb \iota \ldef \frac{\gamma}{8M}$ into \cref{prop:CI_oracle}, which leads to 
$O \prn{ \frac{(\log T)^2}{\gamma^2}\cdot \log \prn{ \frac{\log T}{\gamma}}}$ calls of the regression oracle for general $\cF$ and 
$O \prn{ \log \prn{ \frac{\log T}{\gamma}}}$ calls of the regression oracle for any convex $\cF$ that is closed under pointwise convergence. This also serves as the per-example inference time for $\wh h_{M}$. 
The total computational cost of \cref{alg:abs} is then derived by multiplying the per-round cost by $T$ and plugging $T = \frac{\theta \, \pseud(\cF)}{\eps \, \gamma} = \wt O \prn{\poly(\frac{1}{\gamma}) \cdot \frac{1}{\eps}}$ into the bound.
\end{proof}

\subsection{Proof of \cref{thm:single_relu_lb}}
\label{app:single_relu_lb}

For ease of construction, we suppose the instance space is $\cX = \B^{d}_1 \ldef \crl{x \in \R^{d}: \nrm{x}_2 \leq 1}$. 
Part of our construction is inspired by \citet{li2021eluder}.

\thmSingleReLULB*

\begin{proof}
	Fix any $\gamma \in (0,1 /8)$.
We first claim that we can find a discrete subset $\wb \cX \subseteq \cX$ with cardinality $\abs{\wb \cX} \geq (1/8\gamma)^{d/2}$ such that $\nrm{x_i}_2 =1$ and $\ang{x_1,x_2} \leq 1 - 4\gamma$ for any $x_i \in \wb \cX$. 
To prove this, we first notice that $\nrm{x_1 - x_2}_2 \geq \tau \iff \ang{x_1, x_2} \leq 1 - \tau^2/2$. Since the $\tau$-packing number on the unit sphere is at least $(1/\tau)^d$, setting $\tau = \sqrt{8 \gamma}$ leads to the desired claim.

We set $\cD_\cX \ldef \unif(\wb \cX)$ and $\cF_\dnn \ldef \crl{ \relu(\ang*{w, \cdot} - (1-4 \gamma))+ (1 /2 - 2\gamma): w \in \wb \cX}$. We have $\cF_\dnn \subseteq \cW^{1, \infty}_1(\cX)$ since $\nrm{w}_2 \leq$ for any $w \in \wb \cX$. We randomly select a $w^\star \in \cX$ and set $f^\star(\cdot) = \eta(\cdot) = \relu(\ang{w^\star,\cdot} - (1-4\gamma)) + (1 /2 - 2 \gamma)$. 
We assume that the labeling feedback is the conditional expectation, i.e., $\eta(x)$ is provided if $x$ is queried.
We see that $f^\star(x) = 1 /2 - 2 \gamma$ for any $x \in \cX$ but $x \neq w^\star$, and $f^\star(w^\star) = 1 /2 + 2 \gamma$. 
We can see that mistakenly select the wrong $\wh f \neq f^\star$ leads to $\frac{\gamma}{4} \cdot \frac{2}{\abs{\wb \cX}} = \frac{\gamma}{2 \abs{\wb \cX}}$ excess error. Note that the excess error holds true in both standard excess error and Chow's excess error (with parameter $\gamma$) since $\cD_\cX \prn{x \in \cX: \eta(x) \in [1 /2 - \gamma, 1 /2 + \gamma]} = 0$ by construction.

We suppose the desired access error $\eps$ is sufficiently small (e.g.,  $\eps \leq \frac{\gamma}{8 \abs{\wb \cX}}$).
We now show that, with label complexity at most $K \ldef \floor{{\abs{\wb \cX}}/{2}} = \Omega(\gamma^{- d /2})$, any active learning algorithm will, in expectation, pick a classifier that has $\Omega(\eps)$ excess error.
Since the worst case error of any randomized algorithm is lower bounded by the expected error of the best deterministic algorithm against a input distribution \citep{yao1977probabilistic}, we only need to analyze a deterministic learner.
We set the input distribution as the uniform distribution over instances with parameter $w^\star \in \wb \cX$.
For any deterministic algorithm, we use $s \ldef (x_{i_1}, \dots, x_{i_K})$ to denote the data points queried under the constraint that at most $K$ labels can be queried. We denote $\wh f \in \cF$ as the learned classifier conditioned on $s$.
Since $w^\star \sim \unif(\wb \cX)$, we know that, with probability at least $\frac{1}{2}$, $w^\star \notin s$. Conditioned on that event, we know that, with probability at least $\frac{1}{2}$, the learner will output $\wh f \neq f^\star$ since more than half of the data points remains unqueried. The deterministic algorithm thus outputs the wrong $\wh f \neq f^\star$ with probability at least $\frac{1}{2} \cdot \frac{1}{2} = \frac{1}{4}$, which has $\frac{\gamma}{2 \abs{\wb \cX}}$ excess error as previously discussed.
When $\eps \leq \frac{\gamma}{8 \abs{\wb \cX}}$, this leads to $\Omega(\eps)$ excess error in expectation.
\end{proof}

\section{Omitted details for \cref{sec:extension}}
\label{app:extension}
We provide mathematical backgrounds for the Radon $\BV^{2}$ space in \cref{app:radon}, derive approximation results and passive learning results in \cref{app:radon_passive}, and derive active learning results in \cref{app:radon_active}.

\subsection{The Radon $\BV^2$ space}
\label{app:radon}
We provide explicit definition of the $\nrm{f}_{\RBV^2(\cX)}$ and associated mathematical backgrounds in this section. 
Also see \citet{ongie2020function, parhi2021banach, parhi2022kinds, parhi2022near, unser2022ridges} for more discussions.

We first introduce the \emph{Radon transform} of a function $f: \R^{d} \rightarrow \R$ as 
\begin{align*}
\mathscr{R} \crl{f} (\gamma ,t) \ldef \int_{\crl{x: \gamma^{\trn}x = t}} f(x)\, \mathsf{d} s(x), \quad (\gamma,t) \in \S^{d-1} \times \R,
\end{align*}
where $s$ denotes the surface measure on the hyperplane $\crl{x: \gamma^{\trn}x = t}$.
The Radon domain is parameterized by a \emph{direction} $\gamma \in \S^{d-1}$ and an \emph{offset} $t \in \R$.
We also introduce the \emph{ramp filter} as 
\begin{align*}
	\Lambda^{d-1} \ldef \prn{- \partial^{2}_t}^{\frac{d-1}{2}},
\end{align*}
where $\partial_t$ denotes the partial derivative with respect to the offset variable,  $t$, of the Radon domain, and the fractional powers are defined in terms of Riesz potentials.

With the above preparations, we can define the  $\RTV^2$-seminorm as 
\begin{align*}
\RTV^2(f) \ldef c_d \nrm{\partial^2_t \Lambda^{d-1} \mathscr{R}f}_{\cM(\S^{d-1} \times \R)},	
\end{align*}
where $c_d = 1 / (2 (2 \pi)^{d-1})$ is a dimension-dependent constant, and $\nrm{\cdot}_{\cM(\S^{d-1} \times \R)}$ denotes the \emph{total variation norm} (in terms of measures) over the bounded domain  $\S^{d-1} \times \R$.
The $\RBV^2$ norm of $f$ over  $\R^{d}$ is defined as 
\begin{align*}
\nrm{f}_{\RBV^2(\R^{d})} \ldef \RTV^2(f) + \abs{f(0)} + \sum_{k=1}^{d} \abs{f(e_k) - f(0)},	
\end{align*}
where $\crl{e_k}_{k=1}^{d}$ denotes the canonical basis of $\R ^{d}$.
The $\RBV^2(\R^{d})$ space is then defined as 
\begin{align*}
	\RBV^2(\R^{d}) \ldef \crl{f \in L^{\infty,1}(\R^{d}): \RBV^2(f) < \infty},
\end{align*}
where $L^{\infty,1}(\R^{d})$ is the Banach space of functions mapping $\R^{d} \rightarrow \R$ of at most linear growth.
To define the $\RBV^2$ norm of $f$ over a bounded domain  $\cX \subseteq \R ^{d}$, we use the standard approach of considering restrictions of functions in $\RBV^2(\R^{d})$, i.e.,
\begin{align*}
	\nrm{f}_{\RBV^2(\cX)} \ldef \inf_{g \in \RBV^2(\R^{d})} \nrm{g}_{\RBV^2(\R^{d})} \quad \text{ s.t. }
	\quad g \vert_{\cX} = f.
\end{align*}

In the rest of \cref{app:extension}, we use $\cP(\beta)$ to denote the set of distributions that satisfy (1) Tsybakov noise condition with parameter  $\beta \geq 0$; and (2)  $\eta \in \RBV^2_1(\cX)$.
\subsection{Approximation and passive learning results}
\label{app:radon_passive}

\begin{proposition}
	\label{prop:vc_approx_RBV}
Suppose $\cD_{\cX \cY} \in \cP(\beta)$. 	
One can construct a set of neural network classifier $\cH_{\dnn}$ such that the following two properties hold simultaneously: 
\begin{align*}
	\min_{h \in \cH_{\dnn} }\err(h) - \err(h^{\star}) = O \prn{\eps} \quad \text{ and }
\quad 	\vcd( \cH_{\dnn}) = \wt O \prn{ \eps^{- \frac{2d}{(1+\beta)(d+3)}}}.
\end{align*}
\end{proposition}
\begin{proof}
We take $\kappa = \eps^{\frac{1}{1+\beta}}$ in \cref{thm:approx_RBV2} to construct a set of neural network classifiers $\cH_\dnn$ with  $W = O \prn{ \eps^{- \frac{2d}{(1+\beta)(d+3)}} }$ total parameters arranged in $L = O \prn{ 1}$ layers.	
According to \cref{thm:vcd_nn}, we know 
\begin{align*}
\vcd( \cH_{\dnn}) = O \prn{\eps^{- \frac{2d}{(1+\beta)(d+3)}} \cdot \log \prn{\eps^{-1}}} = \wt O \prn{ \eps^{- \frac{2d}{(1+\beta)(d+3)}}}.
\end{align*}
We now show that there exists a classifier $\wb h \in \cH_{\dnn}$ with small excess error.
Let $\wb h = h_{\wb f}$ be the classifier such that  $\nrm{ \wb f - \eta}_{\infty} \leq \kappa$. We can see that 
 \begin{align*}
	\exc(\wb h)
	& = \E \brk*{ \ind( \wb h(x) \neq y) - \ind( h^{\star}(x) \neq y)}\\
	& = \E \brk*{ \abs{ 2 \eta(x) - 1} \cdot \ind( \wb h(x) \neq h^{\star}(x))}\\
	& \leq 2 \kappa \cdot \P_{x \sim \cD_{\cX}} \prn*{ x \in \cX: \abs{\eta(x) - {1}/{2}} \leq \kappa}\\
	& = O \prn{ \kappa^{1 + \beta}}\\
	& = O(\eps),
\end{align*}
where the third line follows from the fact that $\wb h$ and  $h^{\star}$ disagrees only within region $\crl{x \in \cX: \abs{\eta(x) - 1 /2} \leq \kappa}$ and the incurred error is at most  $2 \kappa$ on each disagreed data point.
The fourth line follows from the Tsybakov noise condition and the last line follows from the selection of $\kappa$.
\end{proof}

\begin{theorem}
	\label{thm:passive_noise_RBV}
	Suppose $\cD_{\cX \cY} \in \cP(\beta)$.
	Fix any $\eps, \delta > 0$.
	Let $\cH_{\dnn}$ be the set of neural network classifiers constructed in \cref{prop:vc_approx_RBV}.
	With $n =\wt O \prn{\eps^{-\frac{4d+6+ \beta(d+3)}{(1+\beta)(d+3)}}} $ i.i.d. sampled data points, with probability at least $1-\delta$,
	the empirical risk minimizer $\wh h \in \cH_{\dnn}$ achieves excess error $ O(\eps)$.
\end{theorem}

\begin{proof}
\cref{prop:vc_approx_RBV} certifies $\min_{h \in \cH_\dnn} \err(h) - \err(h^{\star}) = O(\eps)$
and $\vcd(\cH_\dnn) = O \prn*{\eps^{-\frac{2d}{(1+\beta)(d+3)}} \cdot \log(\eps^{-1})}$.
Take $\rho = 1$ in \cref{thm:erm_tsy}, leads to 
 \begin{align*}
& \err(\wh h) - \err(h^{\star}) \leq   O \prn*{ \eps + \prn*{ \eps^{-\frac{2d}{(1+\beta)(d+3)}} \cdot \log(\eps^{-1}) \cdot \frac{\log n}{n}}^{\frac{1+\beta}{2+\beta}} + \frac{\log\delta^{-1}}{n}},
\end{align*}
Taking $n = O \prn{\eps^{-\frac{4d+6+ \beta(d+3)}{(1+\beta)(d+3)}}\cdot \log (\eps^{-1}) + \eps^{-1} \cdot \log(\delta^{-1})} = \wt O \prn{\eps^{-\frac{4d+6+ \beta(d+3)}{(1+\beta)(d+3)}}}$ thus ensures that $\err(\wh h) - \err(h^{\star}) = O(\eps)$.
\end{proof}

\subsection{Active learning results}
\label{app:radon_active}

\thmActiveNoiseRBV*
\begin{proof}
	Construct $\cH_\dnn$ based on \cref{prop:vc_approx_RBV} such that $\min_{h \in \cH_\dnn} \err(h) - \err(h^{\star}) = O(\eps)$ and $\vcd(\cH_\dnn) = \wt O \prn{ \eps^{-\frac{2d}{(1+\beta)(d+3)}} }$. Taking such $\cH_\dnn$ as the initialization of \cref{alg:RCAL} (line 1) and applying \cref{thm:RCAL_gen} leads to the desired result.
\end{proof}

To derive deep active learning guarantee with abstention in the Radon $\BV^2$ space, 
we first present two supporting results below.
\begin{proposition}
	\label{prop:pd_approx_RBV}
	Suppose $\cD_{\cX \cY} \in \cP(\beta)$. 
	One can construct a set of neural network regression functions  $\cF_\dnn$ such that the following two properties hold simultaneously:
	 \begin{align*}
		\exists f \in \cF_\dnn \text{ s.t. } \nrm{ f - f^{\star}}_\infty \leq \kappa, \quad \text{ and } \quad \pseud(\cF_\dnn) \leq c \cdot {\kappa^{-\frac{2d}{d+3}} \log^2 (\kappa^{-1})},
	\end{align*}
	where $c > 0$ is a universal constant.
\end{proposition}
\begin{proof}
	The result follows by combining \cref{thm:approx_RBV2} and \cref{thm:pdim_nn}.
\end{proof}

\begin{proposition}
	\label{prop:clip_filter_radon}
	Suppose $\eta$ is  $L$-Lipschitz and  $\cX \subseteq \B^{d}_r$.
	Fix any $\kappa \in (0, \gamma / 32]$. There exists a set of neural network regression functions $\cF_\dnn$ such that the followings hold simultaneously.
	\begin{enumerate}
		\item 	$\pseud(\cF_\dnn) \leq c \cdot {\kappa^{-\frac{2d}{d+3}} \log^2(\kappa^{-1})}$ with a universal constant $c > 0$.
		\item There exists a $\wb f \in \cF_\dnn$ such that $\nrm{\wb f - \eta}_\infty \leq \kappa$.
		\item $ \theta^{\val}_{\cF_\dnn}(\gamma / 4) \ldef \sup_{f \in \cF_\dnn, \iota > 0} \theta^{\val}_f (\cF_\dnn, \gamma /4, \iota) \leq c^{\prime} \cdot \prn{\frac{Lr}{\gamma}}^{d}$ with a universal constant $c^{\prime} > 0$.
	\end{enumerate}
\end{proposition}
\begin{proof}
	The implementation and proof are similar to those in \cref{prop:clip_filter}, except we use \cref{prop:pd_approx_RBV} instead of \cref{prop:pd_approx}.
\end{proof}

We now state and prove deep active learning guarantees in the Radon $\BV^{2}$ space.
\begin{restatable}{theorem}{thmAbsRBV}
	\label{thm:abs_RBV}
	Suppose $\eta \in \RBV^2_1(\cX)$.
	Fix any $\eps, \delta, \gamma >0$.
	There exists an algorithm such that,
	with probability at least $1-\delta$, it learns a classifier  $\wh h$ with Chow's excess error $\wt O(\eps)$ after querying $\poly(\frac{1}{\gamma}) \cdot \polylog(\frac{1}{\eps \, \delta})$ labels.
\end{restatable}
\begin{proof}
The result is obtained by applying \cref{alg:abs} with line 1 be the set of neural networks $\cF_\dnn$ generated from \cref{prop:clip_filter_radon} with approximation level $\kappa \in (0, \gamma / 32]$ (and constants $c, c^{\prime}$ specified therein).
The rest of the proof proceeds in a similar way as the proof \cref{thm:abs}.
Since we have $r=1$ and  $L\leq 1$ \citep{parhi2022near},
we only need to choose a $\kappa > 0$ such that 
 \begin{align*}
	\frac{1}{\kappa} = {\check c \cdot \prn*{\frac{1}{\gamma}}^{\frac{d}{2} + 1} \cdot \log \frac{1}{\eps \, \gamma}}
\end{align*}
with a universal constant $\check c > 0$.
With such choice of $\kappa$, we have  
\begin{align*}
\pseud(\cF_\dnn) = O\prn*{\prn*{\frac{1}{\gamma}}^{\frac{d^2+2d}{d+3}} \polylog \prn*{\frac{1}{\eps \, \gamma}}}.
\end{align*}
Plugging this bound on $\pseud(\cF_\dnn)$ and the upper bound on $\theta^{\val}_{\cF_\dnn}(\gamma / 4)$ from \cref{prop:clip_filter_radon} into the guarantee of \cref{thm:abs_gen} leads to $\exc_\gamma(\wh h) = O \prn{\eps \cdot \log \prn{\frac{1}{\eps \, \gamma \, \delta}}}$ after querying 
\begin{align*}
	O \prn*{\prn*{\frac{1}{\gamma}}^{d + 2 + \frac{d^2 + 2d}{d+3}} \cdot \polylog \prn*{\frac{1}{\eps \, \gamma \, \delta}} }
\end{align*}
labels.
\end{proof}

\end{document}